\documentclass[letterpaper,11pt]{article}
\usepackage{parskip}
\usepackage[english]{babel}
\usepackage[nottoc]{tocbibind}

\usepackage[square]{natbib}
\setcitestyle{numbers,square}
\usepackage{diagbox}
\usepackage{makecell}
\usepackage{booktabs}
\usepackage{tikz}
\usepackage{amsmath}
\usepackage{amsfonts}
\usepackage{amssymb}
\usepackage{amsthm}
\usepackage{url,ifthen}
\usepackage{enumitem}
\usepackage{srcltx}
\usepackage{multirow}
\usepackage{boxedminipage}
\usepackage[margin=1.1in]{geometry}
\usepackage{nicefrac}
\usepackage{xspace}
\usepackage{graphicx}
\usepackage{color}
\usepackage{colortbl}
\usepackage{setspace}
\usepackage{soul}

\usepackage{hyperref}
\usepackage{cleveref}
\usepackage{mathptmx}

\definecolor{DarkGreen}{rgb}{0.1,0.5,0.1}
\definecolor{DarkRed}{rgb}{0.5,0.1,0.1}
\definecolor{DarkBlue}{rgb}{0.1,0.1,0.5}
\definecolor{Gray}{rgb}{0.2,0.2,0.2}

\usepackage{listings}
\lstdefinestyle{mystyle}{
    commentstyle=\color{DarkBlue},
    keywordstyle=\color{DarkRed},
    numberstyle=\tiny\color{Gray},
    stringstyle=\color{DarkGreen},
    basicstyle=\footnotesize,
    breakatwhitespace=false,         
    breaklines=true,                 
    captionpos=b,                    
    keepspaces=true,                 
    numbers=left,                    
    numbersep=5pt,                  
    showspaces=false,                
    showstringspaces=false,
    showtabs=false,                  
    tabsize=2
}
\usepackage[small]{caption}
\usepackage{hyperref}
\hypersetup{
    unicode=false,          
    pdftoolbar=true,        
    pdfmenubar=true,        
    pdffitwindow=false,      
    pdfnewwindow=true,      
    colorlinks=true,       
    linkcolor=DarkGreen,          
    citecolor=DarkGreen,        
    filecolor=DarkRed,      
    urlcolor=DarkBlue,          
    pdftitle={},
    pdfauthor={},
}

\newcommand{\softblackwell}{\textsc{Soft-Blackwell FTRL}\xspace}
\def\draft{1}

\def\submit{0}

\ifnum\draft=1 
    \def\ShowAuthNotes{1}
\else
    \def\ShowAuthNotes{0}
\fi

\ifnum\submit=1
\newcommand{\forsubmit}[1]{#1}
\newcommand{\forreals}[1]{}
\else
\newcommand{\forreals}[1]{#1}
\newcommand{\forsubmit}[1]{}
\fi

\ifnum\ShowAuthNotes=1
\newcommand{\authnote}[2]{{ \footnotesize \bf{\color{DarkRed}[#1's Note:
{\color{DarkBlue}#2}]}}}
\else
\newcommand{\authnote}[2]{}
\fi

\newtheorem{theorem}{Theorem}[section]
\newtheorem{remark}[theorem]{Remark}
\newtheorem{lemma}[theorem]{Lemma}
\newtheorem{corollary}[theorem]{Corollary}
\newtheorem{proposition}[theorem]{Proposition}

\theoremstyle{definition}
\newtheorem{definition}[theorem]{Definition}
\newtheorem{assumption}[theorem]{Assumption}

\newtheoremstyle{example_contd}
{\topsep} {\topsep}%
{}
{}
{\bfseries}
{.}
{1em}
{\thmname{#1} \thmnumber{ #2}\thmnote{#3} (continued)}

\theoremstyle{example_contd}

\newcommand{\chapterref}[1]{\hyperref[ch:#1]{Chapter~\ref{ch:#1}}}

\newcommand{\claimref}[1]{\hyperref[claim:#1]{Claim~\ref{claim:#1}}}

\newcommand{\corollaryref}[1]{\hyperref[cor:#1]{Corollary~\ref{cor:#1}}}

\newcommand{\definitionref}[1]{\hyperref[def:#1]{Definition~\ref{def:#1}}}

\newcommand{\equationref}[1]{\hyperref[eq:#1]{Equation~\ref{eq:#1}}}

\newcommand{\factref}[1]{\hyperref[fact:#1]{Fact~\ref{fact:#1}}}

\newcommand{\figureref}[1]{\hyperref[fig:#1]{Figure~\ref{fig:#1}}}

\newcommand{\tableref}[1]{\hyperref[tab:#1]{Table~\ref{tab:#1}}}

\newcommand{\itemref}[1]{\hyperref[item:#1]{Item~(\ref{item:#1})}}

\newcommand{\lemmaref}[1]{\hyperref[lem:#1]{Lemma~\ref{lem:#1}}}

\newcommand{\propref}[1]{\hyperref[prop:#1]{Proposition~\ref{prop:#1}}}

\newcommand{\propositionref}[1]{\hyperref[prop:#1]{Proposition~\ref{prop:#1}}}

\newcommand{\remarkref}[1]{\hyperref[rem:#1]{Remark~\ref{rem:#1}}}

\newcommand{\sectionref}[1]{\hyperref[sec:#1]{Section~\ref{sec:#1}}}

\newcommand{\theoremref}[1]{\hyperref[thm:#1]{Theorem~\ref{thm:#1}}}

\usepackage[utf8]{inputenc}
\usepackage[T1]{fontenc}
\usepackage{kpfonts}
\usepackage{microtype}

\newcommand{\cD}{{\cal D}}

\newcommand{\cQ}{{\cal Q}}

\newcommand{\cX}{{\cal X}}

\newcommand{\cZ}{{\cal Z}}

\renewcommand{\le}{\leqslant}

\renewcommand{\ge}{\geqslant}

\newcommand{\norm}[1]{\lVert#1\rVert_2}

\newcommand{\R}{\mathbb{R}}

\usepackage{bm}

\newcommand{\iprod}[1]{\langle #1\rangle}

\newcommand{\ignore}[1]{}

\renewcommand{\epsilon}{\varepsilon}

\newcommand{\remove}[1]{}

\usepackage{color-edits}
\addauthor{hui}{cyan}

\newcommand{\bE}{\mathbb{E}}

\newcommand{\kc}[1]{\textcolor{blue}{[kaicheng:~#1]}}

\usepackage{bm}
\usepackage{amsfonts}
\usepackage[utf8]{inputenc} 
\usepackage[T1]{fontenc}    
\usepackage{hyperref}       
\usepackage{url}            
\usepackage{booktabs}       
\usepackage{amsfonts}       
\usepackage{nicefrac}       
\usepackage{microtype}      
\usepackage{xcolor}         

\usepackage{algorithm}
\usepackage{algpseudocode}
\usepackage{diagbox}
\usepackage{makecell}
\usepackage{booktabs}
\usepackage{tikz}
\usepackage{amsmath}
\allowdisplaybreaks
\usepackage{amssymb}
\usepackage{amsthm}
\usepackage{url,ifthen}
\usepackage{enumitem}
\usepackage{srcltx}
\usepackage{multirow}
\usepackage{boxedminipage}
\usepackage{nicefrac}
\usepackage{xspace}
\usepackage{graphicx}
\usepackage{color}
\usepackage{colortbl}
\usepackage{setspace}
\usepackage{soul}
\usepackage{graphicx}
\usepackage{wrapfig}
\usepackage{subcaption}
\usepackage{enumitem}
\usepackage{cleveref}
\definecolor{DarkGreen}{rgb}{0.1,0.5,0.1}
\definecolor{DarkRed}{rgb}{0.5,0.1,0.1}
\definecolor{DarkBlue}{rgb}{0.1,0.1,0.5}
\definecolor{Gray}{rgb}{0.2,0.2,0.2}

\newcommand{\Sphere}{\mathbb S}

\newcommand{\proj}{\operatorname{P}}
\newcommand{\poff}{g}          
\newcommand{\cstr}{c}         
\newcommand{\cum}{S}       
\newcommand{\blk}{G}       

\newcommand{\WinSet}{\mathcal W_H}
\newcommand{\Potential}{\varphi}

\title{Constrained Gradient Equilibrium: \\ Approachability under Locally Constrained Adversaries}

\usepackage{authblk}
\usepackage{comment}
\usepackage{xurl}

\author[1]{Kaicheng Zhang}

\affil[1]{University of North Carolina at Chapel Hill}

\author[2]{Lujing Zhang}
\affil[2]{Carnegie Mellon University}
\author[1]{Yifei Pang}
\author[3]{Yuxiao Wen}
\affil[3]{New York University}
\author[1]{Zhun Deng}
\author[3]{Yanjun~Han}
\date{}
\begin{document}

\maketitle

\begin{abstract}
Gradient equilibrium (GEQ) is a recent framework in online learning that requires the time-averaged gradients along the learner's trajectory to vanish. We study GEQ when the learner is constrained and ask when such an equilibrium remains achievable in an adversarial environment. We first show that constrained GEQ (CGEQ) can be impossible even under simple and slowly varying constraints, motivating imposing local temporal structure on the adversary. We then develop an approachability framework for CGEQ under locally constrained adversaries, and obtain necessary and sufficient characterizations in the spirit of classical Blackwell approachability. Beyond these theoretical characterizations, we provide an algorithmic principle that converts window-wise no-regret learners into CGEQ algorithms under a local feasibility condition with margin. Experiments on synthetic and real data demonstrate that the resulting algorithms perform strongly in both achieving GEQ and respecting the constraints.
\end{abstract}
\clearpage
\setcounter{tocdepth}{2}
\tableofcontents

\clearpage

\section{Introduction}\label{sec:introduction}

Gradient equilibrium (GEQ), introduced recently by \citet{angelopoulos2025gradient}, has emerged as a useful perspective on online learning, with connections to applications including calibration \cite{foster1998asymptotic} and conformal prediction \cite{gibbs2021adaptiveconformalinferencedistribution}. In contrast to regret minimization, which compares the learner against the best fixed decision in hindsight, GEQ asks whether the time-average of the loss (sub)gradients along the learner's trajectory vanishes. Specifically, letting $\poff_t\in\partial \ell_t$ denote a (sub)gradient of the loss at round $t$, GEQ requires
\begin{align}\label{eq:GEQ}
    \left\|
        \frac{1}{T}
        \sum_{t=1}^T \poff_t(x_t)
    \right\|_2
    \longrightarrow 0 \quad \text{as }T\to\infty, 
\end{align}
where $x_1,\dots,x_T$ denote the learner's trajectory. For example, when $\ell_t$ is the pinball loss, the corresponding GEQ problem is online quantile calibration that requires the fraction of overpredictions to reach a prescribed quantile level. As another example, consider the cross-entropy loss $\ell_t(x_t) = -z_t\log \sigma(x_t) - (1-z_t)\log(1-\sigma(x_t))$ with probabilistic predictions $\sigma(x_t)=1/(1+e^{-x_t})\in [0,1]$ (allowing $x_t=\pm\infty$) and binary outcomes $z_t\in \{0,1\}$. In this case, GEQ requires $T^{-1}\sum_{t=1}^T (\sigma(x_t)-z_t) \to 0$, an online prediction problem known as defensive forecasting \cite{vovk2005defensive,perdomo2025defense}.

Practical scenarios often concern the case where the learner's action $x_t$ is constrained, i.e. $x_t\in C_t\subseteq \mathbb{R}^d$. Existing guarantees for constrained GEQ (CGEQ) \cite{angelopoulos2025gradient,lee2026blackwell} often take the form
\begin{equation}\label{eq:normal-cone GEQ}
    \left\|
        \frac{1}{T}
        \sum_{t=1}^T
        \bigl(\poff_t(x_t)+n_t\bigr)
    \right\|_2
    \longrightarrow 0, \quad \text{as }T\to\infty, 
\end{equation}
where $x_t\in C_t$ and $n_t$ is a normal vector to $C_t$ at $x_t$. The correction $n_t$ is nonzero only if $x_t$ lies on the boundary of $C_t$, and can be viewed as accounting for the projection of an unconstrained action back onto the feasible set. In many applications, however, the relevant objective remains the original GEQ condition \eqref{eq:GEQ}, together with the feasibility requirement  $x_t\in C_t$, rather than the normal-cone formulation \eqref{eq:normal-cone GEQ}. For instance, in online quantile calibration, it is the original GEQ objective that guarantees that the average coverage error at the target quantile level converges to zero. This stronger constrained formulation is studied in recent work by  \citet{ding2026calibratedmultilevelquantileforecasting}, which gives sufficient conditions on the gradients $g_t$ under which \eqref{eq:GEQ} remains achievable despite the constraints. It remains unclear, however, whether these conditions are also necessary and, more fundamentally, when constraints genuinely change the achievability of GEQ relative to the unconstrained setting.

In this paper, we provide a systematic study of the CGEQ problem and make the following contributions:
\begin{enumerate}
    \item First, we identify a simple example showing that CGEQ can be impossible without suitable restrictions on the adversary. In particular, unlike standard global feasibility or stationarity conditions, achieving CGEQ may require the adversary to satisfy a \emph{local} temporal constraint.
    
    \item Second, we study two formulations of locally constrained adversaries and derive necessary and sufficient conditions for CGEQ in the spirit of classical Blackwell approachability. Our characterizations subsume existing sufficient conditions such as restorativity and inward flow.
    
    \item Third, we translate these characterizations into efficient algorithms for achieving CGEQ. As an important special case, we show that under local feasibility with a margin condition, a generic reduction converts any suitable no-regret learner into a CGEQ algorithm.
    
    \item Finally, we evaluate our CGEQ algorithms on both synthetic and real data. The experiments show that our methods are effective at simultaneously achieving GEQ and respecting the constraints, whereas existing baselines generally fail to satisfy both objectives.
\end{enumerate}

\subsection{Related work}
\paragraph{GEQ.} 
Although GEQ and no-regret learning are known to be incomparable in general \cite{angelopoulos2025gradient}, existing algorithms for GEQ are inspired by no-regret learning.  
For instance, \citet{angelopoulos2025gradient} consider online gradient descent with a constant learning rate, i.e. $x_{t+1} = x_t - \eta g_t(x_t)$, and introduce a \emph{restorative} condition on the adversary to control the divergence of the trajectory $x_t$ in the unconstrained setting.  
In the constrained setting, \citet{ding2026calibratedmultilevelquantileforecasting} consider lazy online gradient descent, i.e. $x_t = \mathrm{Proj}_{C_t}(\widetilde{x}_t)$ and $\widetilde{x}_{t+1} = \widetilde{x}_t - \eta g_t(x_t)$, and introduce an \emph{inward flow} condition, together with restorativity, to control the magnitude of the unconstrained iterate $\widetilde{x}_t$.  
As discussed in the introduction, however, these conditions are only sufficient and, more importantly, do not characterize when constraints fundamentally change the achievability of GEQ.

\paragraph{Blackwell approachability.} 
Reduction techniques connecting different frameworks in online learning have underpinned recent algorithmic advances in a range of problems, including calibration \cite{hebertjohnson2018multicalibration,foster2021forecast,okoroafor2024faster}, multi-objective learning \cite{lee2022online,haghtalab2023unifying}, and omniprediction/panprediction \cite{okoroafor2025nearoptimal,balakrishnan2025panprediction}.  
A celebrated example is Blackwell approachability \cite{blackwell1956analog}, together with its equivalences to no-regret learning \cite{abernethy2011blackwell} and calibration \cite{perchet2013approachability}. 
In particular, a recent work \cite{lee2026blackwell} establishes an equivalence between GEQ and Blackwell approachability by showing that GEQ can be formulated as approaching the singleton target set $\{{\bf 0}\}$ under the vector payoff $g_t$.  
In the constrained setting, however, this equivalence leads only to the normal-cone objective \eqref{eq:normal-cone GEQ}, rather than the stronger objective \eqref{eq:GEQ}.  

Our characterization of locally constrained adversaries in \Cref{sec:general-theory} is also formulated through Blackwell approachability, with the additional complication that local adversarial constraints induce temporal dependence in the game state.  
As a result, we develop approachability results for stochastic Stackelberg games in which the state evolution is governed by a Markov decision process (MDP), in a spirit similar to \cite{milman2006approachable,kalathil2017approachability}.

\paragraph{Regret minimization.} Regret minimization is one of the most classical frameworks in online learning and online convex optimization (OCO), and has inspired a wide range of algorithmic ideas \cite{cesa2006prediction,orabona2019modern}.  
In the unconstrained setting, a canonical no-regret algorithm for OCO is online gradient descent \cite{zinkevich2003online}.  
The constrained setting, referred to as constrained online convex optimization (COCO), is more challenging, but recent work \cite{sinha2024optimal} obtains near-optimal guarantees simultaneously for objective regret and soft constraint violation, even under time-varying constraints.

\citet{lee2026blackwell} also establish reductions between regret minimization and GEQ.  
We emphasize that their reductions differ from ours in \Cref{sec:local-feasibility} in two important ways.  
First, the reduction in \cite{lee2026blackwell} relies on the restorativity condition, whereas ours is based on a distinct local feasibility condition.  
Second, and more importantly, our no-regret learner is applied separately within each local window, whereas the online learner in \cite{lee2026blackwell} operates over the entire time horizon.

\subsection{Organization}
The remainder of the paper is organized as follows.  
In \Cref{sec:lowerbound}, we construct a simple example showing that CGEQ can be impossible without local constraints on the adversary, and introduce two notions of locally constrained adversaries.  
In \Cref{sec:general-theory}, we develop necessary and sufficient characterizations of admissible local constraint sets using Blackwell approachability for vector-payoff stochastic Stackelberg games.  
In \Cref{sec:local-feasibility}, we specialize these general results to local feasibility constraints and develop efficient algorithms for CGEQ.  
We evaluate our algorithms on both synthetic and real data in Section~\ref{sec:experiments}, and defer proof details and extensions to the appendices.

\section{Negative Results for CGEQ}\label{sec:lowerbound}
We begin with a negative example showing that CGEQ could be impossible even if there is a fixed action \emph{in hindsight} that achieves CGEQ. Consider the defensive forecasting problem in the introduction, with learner's action reformulated as $p_t = \sigma(x_t)\in [0,1]$. In other words, at the beginning of round $t$, the learner makes a probabilistic prediction $p_t \in [0,1]$, and nature (adversary) reveals a binary outcome $z_t \in \{0,1\}$. The GEQ objective \eqref{eq:GEQ} in this case is
\begin{align}\label{eq:GEQ-forecast}
\frac{1}{T}\sum_{t=1}^T (p_t - z_t) \to 0, \quad \text{as }T\to\infty. 
\end{align}
In the unconstrained case, achieving \eqref{eq:GEQ-forecast} is easy, e.g. by setting $p_1=0$ and $p_{t+1} = z_t$. To make it a constrained GEQ problem, we impose a simple constraint: 
\begin{align}\label{eq:constraint-forecast}
p_t \in \left[0, \frac{1}{2}\right], \quad \text{for all }t\ge 1. 
\end{align}
We show that CGEQ is now impossible even if the adversary is randomizing over two deterministic choices.

\begin{theorem}\label{thm:lowerbound}
Let $T\in \mathbb{N}$ be an even integer, and assume the following strategy of the adversary: with probability $\frac{1}{2}$, set $z^T = (0,0,\dots,0)$; with probability $\frac{1}{2}$, set $z^T = (0,\dots,0,1,\dots,1)$, with an equal number of $0$s and $1$s. 
Then, every admissible learner subject to the constraint \eqref{eq:constraint-forecast} must incur
\begin{align*}
\bE\left| \frac{1}{T}\sum_{t=1}^T (p_t - z_t) \right| \ge c > 0,
\end{align*}
for some universal constant $c>0$. 
\end{theorem}

In short, \Cref{thm:lowerbound} shows that (a uniform version of) the CGEQ objective \eqref{eq:GEQ-forecast} is non-achievable under the above adversary strategy.\footnote{This strategy can easily be made agnostic to $T$ such that the pointwise version of \eqref{eq:GEQ-forecast} still fails, but we choose to use this simple construction to illustrate our main idea.} In particular, the adversarial construction in \Cref{thm:lowerbound} shows that CGEQ can be impossible against a seemingly ``nice'' adversary:
\begin{enumerate}
    \item In hindsight, the best fixed action $p_t \equiv \frac{1}{T}\sum_{s=1}^T z_s$ achieves GEQ and is feasible, i.e. $p_t\in [0,1/2]$; 
    \item More strongly, for every $T_0\le T$, the partial sequence $z^{T_0}$ admits the feasible fixed action $p_t \equiv \frac{1}{T_0}\sum_{s=1}^{T_0} z_s$ which achieves GEQ on the first $T_0$ rounds; 
    \item The adversarial sequence $z^T$ also has very small total variation: $\sum_{t=1}^{T-1} |z_{t+1}- z_t| \le 1$. 
\end{enumerate}

Therefore, \Cref{thm:lowerbound} shows that CGEQ cannot be guaranteed merely from global or cumulative feasibility in hindsight, nor from a globally stable environment with few switches. The obstruction is instead \emph{local}: although every prefix is feasible in hindsight, CGEQ may become infeasible when attention is restricted to a local block. For example, the second half of $z^T$ may consist entirely of ones, $(1,\dots,1)$. On such a block, no fixed feasible prediction satisfying \eqref{eq:constraint-forecast} can achieve the GEQ condition \eqref{eq:GEQ-forecast}. This motivates imposing local, rather than merely global, restrictions on the adversary.

Interestingly, an appropriate local feasibility condition is sufficient to restore CGEQ in this example. Suppose every consecutive subsequence $z_{s+1}^{s+H}$ of length $H$ admits a feasible fixed action that achieves GEQ exactly; equivalently, $\frac{1}{H}\sum_{t=s+1}^{s+H} z_t \in [0,1/2]$ for every $s\ge 0$. In this case, consider the feasible learner defined by $p_1=\dots=p_H=0$ and $p_{t+H} = \frac{1}{H}\sum_{s=t}^{t+H-1} z_s$ for every $t\ge 1$. For this learner, a simple telescoping argument gives $|\frac{1}{T}\sum_{t=1}^T(p_t-z_t)|=O(\frac{H}{T})$, which is $o(1)$ as long as $H=o(T)$. Thus, in sharp contrast to the lower-bound construction above, CGEQ becomes achievable once feasibility is imposed uniformly over sufficiently local windows.

In this paper, we generalize the above observation and consider two types of locally constrained adversaries: \emph{blockwise} and \emph{sliding window}.

\begin{definition}\label{defn:local-adversary}
Let $H\ge 1$ be a fixed window length, and $\WinSet\subseteq \mathcal{Z}^H$ be a fixed constraint set of length-$H$ sequences. We define two types of locally constrained adversaries:
\begin{itemize}
    \item \textbf{Blockwise adversary}: An adversarial sequence $z^T$ satisfies an $H$-\emph{blockwise} condition if $T=KH$ is an integer multiple of $H$, and $z_{(k-1)H+1}^{kH} \in \WinSet$ for every $k\in [K]$.
    \item \textbf{Sliding window adversary}: An adversarial sequence $z^T$ satisfies an $H$-\emph{sliding window} condition if $z_t^{t+H-1} \in \WinSet$ for every $t\in [T-H+1]$.
\end{itemize}
\end{definition}

We view the sliding window condition as the more natural notion of local constraint, while the blockwise condition serves as a mathematically simpler starting point. We will address the following key question:
\emph{For each of these two locally constrained adversaries, which local constraint sets $\WinSet$ make CGEQ achievable?}

\section{Blackwell Approachability of CGEQ}\label{sec:general-theory}

To fully characterize the approachability of CGEQ, we introduce a two-player vector-payoff game framework with the two types of adversaries motivated in Section~\ref{sec:lowerbound}. Starting with the weaker blockwise constraint, we propose a necessary and sufficient Blackwell condition for when CGEQ is achievable. This characterization is extended to the sliding window adversary with a Bellman-type correction to account for the trade-off between the adversary's immediate power and the limitation in future flexibility.

To formalize our results, consider the following Stackelberg game. Let $C$ denote the learner's compact constrained action space and $\mathcal Z$ the adversary's outcome space. At each time $t=1,\dots,T$, the learner chooses an action $x_t\in C$ based on the history, and then the adversary chooses an outcome
\(
z_t\in\mathcal Z.
\) 
The interaction generates a vector payoff
\(
\poff(x_t,z_t)\in\mathbb R^d
\)
under a known payoff function $g$, with $\|g\|_\infty \le L$. In the language of Blackwell approachability \cite{blackwell1956analog}, the learner aims to make the average payoff approach the origin $\{0\}$:
\begin{equation}\label{eq:obj_avg_payoff_zero}
\left\|
\frac{1}{T}
\sum_{t=1}^T
\poff(x_t,z_t)
\right\|_2
\longrightarrow 0.
\end{equation}
Specializing to GEQ, we take $g(\cdot,z_t)=g_t(\cdot)\in\partial \ell_t$ as the (sub)gradient of the loss $\ell_t$, and the objective in \eqref{eq:obj_avg_payoff_zero} recovers the CGEQ objective in \eqref{eq:GEQ} for $C_t \equiv C$.

\subsection{Approachability under blockwise adversary}

This section concerns the conceptually simpler blockwise adversary and proposes a Blackwell condition as an exact characterization for attaining \eqref{eq:obj_avg_payoff_zero}.

\begin{definition}[Blockwise Blackwell condition]\label{def:split-blackwell}
Under the notations above, the blockwise Blackwell condition is
\begin{equation}\label{eq:split-blackwell}
\beta_H \coloneqq \sup_{u\in\mathbb{S}^{d-1}}\inf_{\pi\in\Pi_H}\sup_{s\in\mathcal{W}_H}\langle u, G_H(\pi,s)\rangle \le 0,
\end{equation}
where $\Pi_H$ is the set of causal learner policies over an $H$-round block, and
\[
G_H(\pi,s)
\coloneqq
\sum_{h=1}^H
\poff(x_h^\pi,z_h)
\]
is the cumulative vector payoff when the learner follows $\pi$ and the adversary chooses the sequence $s=(z_1,\ldots,z_H)$.
\end{definition}

The outer supremum over \(u\) collects the Blackwell halfspace conditions over all directions. The inner inf-sup reflects the leader–follower structure of the induced block game: the learner commits to a causal $H$-round policy $\pi\in\Pi_H$, while the adversary selects a worst-case admissible window $s\in\mathcal{W}_H$. For the special case $H=1$, the quantity $\beta_1$ recovers the classical Blackwell halfspace condition for approaching the origin \citep{blackwell1956analog}.

\begin{theorem}[Approachability under blockwise adversary]
\label{thm:split-blackwell}
There exists a learning algorithm that guarantees
\[
    \left\|
        \frac{1}{T}
        \sum_{t=1}^T \poff(x_t,z_t)
    \right\|_2
    \longrightarrow 0
\]
against every blockwise constrained adversary if and only if \eqref{eq:split-blackwell} holds.

In particular, if \eqref{eq:split-blackwell} holds and $\|g\|_\infty \le L$, for every horizon $T\ge 1$, the learning algorithm achieves
\[
\left\|
        \frac{1}{T}
        \sum_{t=1}^T \poff(x_t,z_t)
    \right\|_2 \le L\sqrt{\frac{H}{T}}+\frac{LH}{T}.
\]
\end{theorem}

Theorem~\ref{thm:split-blackwell} generalizes the equivalence result of \cite{lee2026blackwell} to the constrained setting and general $H\ge 1$. For $H=1$, we can compare the Blackwell condition \eqref{eq:split-blackwell} with the restorative \citep{lee2026blackwell} and inward flow conditions \citep{ding2026calibratedmultilevelquantileforecasting} in the literature.\footnote{\citet{ding2026calibratedmultilevelquantileforecasting} do not assume $C$ to be compact and thus require both Assumptions \ref{assump:restorative} and \ref{assump:inward-flow}.} 

\begin{assumption}[Restorativity]\label{assump:restorative}
Let $h\ge0$ and let $\phi:C\to[0,\infty)$.
The payoff $\poff$ is $(h,\phi)$-restorative if,
for every $z\in\cZ$ and every $x\in C$ with $\norm{x}\ge h$,
\[
\iprod{x,-\poff(x,z)}
\le
-\phi(x).
\]
In particular, \citet{lee2026blackwell} primarily use the special case $\phi\equiv 0$, in which case $\langle x, -g(x,z)\rangle \le 0$.
\end{assumption}

\begin{assumption}[Inward flow]\label{assump:inward-flow}
Let the constraint set $C$ be nonempty, compact, and convex. The inward flow condition states that, for every $x\in\partial C$ and every $z\in\mathcal Z$,
\[
-\poff(x,z)\in 
\left\{
    v\in\mathbb R^d:
    \langle v,n\rangle\le 0,
    \quad \forall n\in N_{C}(x)
\right\},
\]
where $N_{C}(x)$ denotes the normal cone of $C$ at $x$ as
\[
N_{C}(x)
\coloneqq
\left\{
    n\in\mathbb R^d:
    \langle n,y-x\rangle\le 0,
    \quad \forall y\in C
\right\}.
\]
\end{assumption}

It is shown in \cite{lee2026blackwell} that the restorative condition implies $\beta_1\le 0$ in the unconstrained case $C = \mathbb{R}^d$, by taking $x = -ru$ with $r\to\infty$. Our next result shows that in the constrained case, the inward flow condition also implies $\beta_1 \le 0$, and is therefore a special case where the Blackwell condition holds. In addition, neither condition is necessary in general.

\begin{proposition}
\label{prop:sufficient-but-unnecessary-conditions}
The inward flow condition in Assumption \ref{assump:inward-flow} implies $\beta_1\le 0$. However, neither Assumption \ref{assump:restorative} nor Assumption \ref{assump:inward-flow} is necessary: there exist instances where CGEQ is achievable but restorative and inward flow conditions fail. 
\end{proposition}




\subsection{Approachability under sliding window adversary}\label{subsec:sliding-window-theory}

We now turn to the more natural sliding window condition, where the key difference is the temporal dependence induced by the sliding window condition. To this end, let
\(
s=(z_{t-H},\ldots,z_{t-1})\in\mathcal W_H
\)
denote the length-$H$ history immediately preceding an $H$-round block starting at time $t$.\footnote{For the initial rounds, one may
prepend a fixed placeholder outcome to define the initial state.}
We define the set of admissible successors for the next $H$ rounds as
\begin{align*}
\mathcal W_H(s)
& \coloneqq
\{
s'=(s_1',\dots,s_H') \in\mathcal W_H: \\
& (s_{r+1},\ldots,s_H,s'_1,\ldots,s'_r)\in\mathcal W_H, \forall r\in[H-1]
\}.
\end{align*}
Throughout we assume that $\mathcal{W}_H(s)\neq \varnothing$ for every state under consideration; otherwise, the adversary has no eligible move and the game breaks.

This continuation constraint creates a trade-off between the adversary's current directional payoff and the set of future continuations that remain available. To capture this dependence, let $\varphi:\mathcal W_H\to\mathbb R$ be a bounded continuation potential function. The difference $\varphi(s')-\varphi(s)$ accounts for the change in future continuation value induced by the transition $s\to s'$. For any direction $u\in\mathbb S^{d-1}$, define the state-dependent operator
\begin{equation*}
\label{eq:sliding-operator}
(\mathcal T_{u,H}\varphi)(s)
\coloneqq
\inf_{\pi\in\Pi_H}
\sup_{s'\in\mathcal W_H(s)}
\left\{
    \left\langle u,G_H(\pi,s')\right\rangle
    +\varphi(s')
\right\}.
\end{equation*}

Using this operator, define the lower and upper directional
values
\begin{align*}
\underline\lambda_H
&\coloneqq
\sup_{u\in\mathbb S^{d-1}}
\sup_{\Potential\in\ell^\infty(\WinSet)}
\inf_{s\in\WinSet}
\left\{
    (\mathcal T_{u,H}\Potential)(s)-\Potential(s)
\right\},
\\
\overline\lambda_H
&\coloneqq
\sup_{u\in\mathbb S^{d-1}}
\inf_{\Potential\in\ell^\infty(\WinSet)}
\sup_{s\in\WinSet}
\left\{
    (\mathcal T_{u,H}\Potential)(s)-\Potential(s)
\right\}.
\end{align*}
By the minimax inequality, $\underline\lambda_H\le\overline\lambda_H$. These values turn out to determine approachability of CGEQ under sliding window adversaries, as shown in the following theorem. 

\begin{theorem}[Approachability under sliding window adversary]
\label{thm:slide-blackwell}
For a constraint set $\mathcal{W}_H$, $\underline\lambda_H \le 0$ is a necessary condition for the pointwise approachability
\begin{equation}\label{eq:weak_approachability}
\lim_{T\to\infty}\left\|\frac{1}{T}\sum_{t=1}^T \poff(x_t,z_t)\right\|_2 = 0
\end{equation}
against all $H$-sliding window adversaries. Similarly, $\overline\lambda_H\le 0$ is a necessary and sufficient condition for the uniform approachability 
\begin{equation}\label{eq:strong_approachability}
\left\|\frac{1}{T}\sum_{t=1}^T \poff(x_t,z_t)\right\|_2 \le r(T),\quad \forall T\ge 1, 
\end{equation}
with $r(T)=o(1)$ being a deterministic function, against all $H$-sliding window adversaries.

In addition, under Assumption~\ref{ass:finite-step-bridging} below, we have
\[
    \underline\lambda_H
    =
    \overline\lambda_H
    \eqqcolon
    \lambda_H, \quad \text{and }\eqref{eq:weak_approachability} \Longleftrightarrow \eqref{eq:strong_approachability}. 
\]
\end{theorem}

The missing assumption is a mild bridging assumption stating that every admissible state can reach every other admissible state in a uniformly bounded number of transitions. This condition rules out disconnected components of the sliding state space in the long run.
\begin{assumption}[Finite-step bridging]
\label{ass:finite-step-bridging}
There exists $C_B\in\mathbb N$ such that, for every
$s,s'\in\WinSet$, there exist an integer
$0\le m\le C_B$ and admissible states
\(
s=s_0,s_1,\ldots,s_m=s'
\)
satisfying
\(
s_{j+1}\in\mathcal W_H(s_j)
\)
for $j=0,\ldots,m-1$.
\end{assumption}

Similar to the second statement of \Cref{thm:split-blackwell}, we can obtain a quantitative version of \Cref{thm:slide-blackwell} by assuming a bounded range of $\varphi$ in the infimum of $\overline\lambda_H$. 

\begin{theorem}[Quantitative sliding-window guarantee]
\label{thm:slide-quantitative}
Assume that $\|g\|_\infty\le L$, and for every $u\in\mathbb S^{d-1}$, there exists a bounded function
$\varphi_u:\WinSet\to\mathbb R$ such that
\begin{equation*}
\label{eq:bounded-potential}
\sup_{u\in\mathbb S^{d-1}}
    \operatorname{range}(\varphi_u)
    \le B_H<\infty,
\end{equation*}
and
\(
    \mathcal T_{u,H}\varphi_u
    \le
    \varphi_u
\)
on $\WinSet.$ Here $\operatorname{range}(\Potential_u):=\sup_{s\in\WinSet}\Potential_u(s)-\inf_{s\in\WinSet}\Potential_u(s)$. Then for any $T\ge1$, there exists a learning algorithm such that
{\small
\[
    \left\|
        \sum_{t=1}^T \poff(x_t,z_t)
    \right\|_2
    =
    O\left(
        B_H^{1/3}L^{2/3}T^{2/3}
        +L\sqrt{HT}+HL
    \right).
\]
}
\end{theorem}

Compared with \Cref{thm:split-blackwell}, the necessary and sufficient characterization $\lambda_H\le 0$ in \Cref{thm:slide-blackwell} can be viewed as a state-corrected analogue of the classical Blackwell halfspace condition. More precisely, this condition means that, for every direction
$u\in\Sphere^{d-1}$ and every $\varepsilon>0$, there exists a
bounded potential
$\Potential_{u,\varepsilon}:\WinSet\to\R$
such that, for every current state $s\in\WinSet$, the learner can
choose a causal block policy
$\pi_{u,s,\varepsilon}\in\Pi_H$ satisfying
\[
\sup_{s'\in\WinSet(s)}
\left\{
\iprod{u,\blk_H(\pi_{u,s,\varepsilon},s')}
+\Potential_{u,\varepsilon}(s')
-\Potential_{u,\varepsilon}(s)
\right\}
\le \varepsilon.
\]
In particular, by choosing $\varphi_u\equiv 0$, the condition $\beta_H\le 0$ in \eqref{eq:split-blackwell} clearly implies $\lambda_H\le 0$. In the special case where $\mathcal{W}_H(s) = \mathcal{W}_H$ for all $s\in \mathcal{W}_H$, i.e. sliding window induces no additional temporal dependence, simple algebra gives $\beta_H = \underline\lambda_H = \overline\lambda_H$. The next result shows that they could be different in general. 

\begin{proposition}\label{prop:comparison}
There exist instances with $\lambda_H \le 0 < \beta_H$, meaning that approachability of the origin is possible under the sliding window adversary, but impossible under the blockwise adversary. 
\end{proposition}

\section{Algorithms under Local Feasibility}
\label{sec:local-feasibility}

The Blackwell conditions in Theorems~\ref{thm:split-blackwell} and \ref{thm:slide-blackwell}, although both necessary and sufficient, can be hard to verify in general. In this section, we stick to the target of CGEQ and restrict $\mathcal{W}_H$ to be the set of \emph{locally feasible} sequences, namely, for every consecutive interval
$I$ of length $H$, the adversary satisfies
\begin{equation}
\label{eq:weak_local_feasibility}
    \min_{x\in C}
    \sup_{u\in\mathbb S^{d-1}}
    \sum_{t\in I}
        \langle u,g_t(x)\rangle
    =
    \min_{x\in C}
    \left\|
        \sum_{t\in I} g_t(x)
    \right\|_2
    = 0.
\end{equation}
In other words, every admissible length-$H$ window admits an action $x_I\in C$ such that $\sum_{t\in I} g_t(x_I)=0$. This is the assumption we made on the simple example in \Cref{sec:lowerbound}, and ensures that in hindsight, there exists a feasible action that could attain exact local CGEQ on every length-$H$ (blockwise or sliding) window. 

The local feasibility condition \eqref{eq:weak_local_feasibility} may or may not imply the Blackwell conditions in Theorems~\ref{thm:split-blackwell} and \ref{thm:slide-blackwell}; see \Cref{sec:weak_feasibility_sufficient} for sufficiency results in dimension 1, but \Cref{sec:weak_feasibility_insufficient} for insufficiency results in dimension 2. More importantly, in \Cref{sec:strong-local-feasibility}, under a directional variant of \eqref{eq:weak_local_feasibility}, we provide an algorithmic reduction that converts any no-regret learning algorithm into CGEQ. 

\begin{remark} 
In Appendix \ref{app:soft-cgeq}, we also extend our feasibility condition \eqref{eq:weak_local_feasibility} to a competitive framework, and extend our algorithms and results to cases where the constraint $C_t$ can depend on $t$ and be \emph{soft}, i.e. allow positive but small cumulative violations. 
\end{remark}


\subsection{When local feasibility is sufficient}
\label{sec:weak_feasibility_sufficient}

In Section~\ref{sec:lowerbound}, we have seen that CGEQ is achievable in the online forecasting example with local feasibility conditions. This turns out to be a general phenomenon for one-dimensional
convex losses $\ell_t$, whose selected subgradients $\poff_t$
are nondecreasing. 


\begin{proposition}\label{prop:1d_example_sufficiency}
Let $C=[a,b]\subseteq\mathbb R$ and each loss $\ell_t$ be convex. Suppose the adversary satisfies local feasibility \eqref{eq:weak_local_feasibility}. Then the blockwise Blackwell condition \eqref{eq:split-blackwell} holds, and thereby CGEQ is achievable under both blockwise and sliding window adversaries. 
\end{proposition}

As a sanity check, the forecasting example specializes to
\(
    g_t(p)=p-z_t
\), which is increasing in $p$, and $C=[0,1/2]$.
Local feasibility precisely requires
\[
    \overline{z}_I\coloneqq \frac1H\sum_{t\in I}z_t\in[0,1/2]
\]
for every length-$H$ window $I$. Since $u\in\mathbb S^0=\{-1,+1\}$, the learner may select the constant action $0$ for $u=+1$ and $1/2$ for $u=-1$. Then the blockwise Blackwell condition holds:
\[
\beta_H \le H \max\left\{\overline{z}_I - \frac 12, -\overline{z}_I\right\} \le 0.
\]
Thus Theorem~\ref{thm:split-blackwell} indeed shows that CGEQ is achievable, confirming Proposition \ref{prop:1d_example_sufficiency}. 

The result of Proposition \ref{prop:1d_example_sufficiency} can also be extended to ``1D-like'' problems, such as $g_t(p)=(p-z_t)\Phi(p)$ with a feasible map $\Phi: \mathbb{R}\to\mathbb{R}^d$, using the same argument above. This function class is also in the defensive forecasting paper \cite{perdomo2025defense}. 


\subsection{When local feasibility is insufficient}
\label{sec:weak_feasibility_insufficient}

Unfortunately, the local feasibility condition is not always sufficient for the Blackwell conditions in \Cref{sec:general-theory}. A simple counterexample already arises in dimension 2. Let the learner's action be
\[
    p_t=(p_{t,1},p_{t,2})
    \in
    C
    \coloneqq
    \{p\in\mathbb R_+^2:\|p\|_2\le1\}.
\]
After observing $p_t$, the adversary chooses a coordinate $j_t\in\{1,2\}$ to evaluate and a binary outcome $z_t\in\{0,1\}$. Only the selected coordinate incurs loss
\(
    \ell_t(p)
    =
    \frac12(p_{j_t}-z_t)^2,
\)
so $g_t(p) = \nabla \ell_t(p) = (p_1-z_t,0)$ if $j_t = 1$, and $g_t(p) = (0,p_2-z_t)$ if $j_t = 2$. The next result shows that CGEQ is impossible under the sliding window adversary with a locally feasible $\mathcal{W}_H$. 

\begin{proposition}
\label{thm:invalid_weak_feasibility}
For every fixed $H<\infty$, there exists a deterministic, causal, and $H$-sliding window adversary satisfying
\eqref{eq:weak_local_feasibility} such that, against
every learner sequence $p_1,p_2,\dots \in C$,
\[
    \left\|
        \sum_{t=1}^T\nabla\ell_t(p_t)
    \right\|_2
    =
    \Omega_H(T).
\]
\end{proposition}

Intuitively, the failure comes from a directional tradeoff: to avoid positive residual on one coordinate, the learner must place substantial weight on that coordinate, but the constraint $\|p_t\|_2\le 1$ then forces the other coordinate to be small.  
Crucially, local feasibility has \emph{zero slack}: for some windows and directions, the best feasible fixed action achieves directional residual exactly zero.   
The adversary can therefore exploit the learner's causal mismatch without any negative buffer to offset it, causing these local errors to accumulate into a nonvanishing gradient residual.

\subsection{Learning with strong local feasibility}\label{sec:strong-local-feasibility}

The preceding impossibility motivates a condition that gives the learner some slack to compensate for learning the hindsight comparator. We say that the adversary satisfies \textit{strong local feasibility with margin $\gamma>0$}, if for every consecutive length-$H$ interval $I$,
\begin{equation}
\label{eq:local-feasibility}
    \sup_{u\in\mathbb S^{d-1}}
    \min_{x\in C}
    \sum_{t\in I}
        \langle u,g_t(x)\rangle
    \le -\gamma H. 
\end{equation}
This condition requires that, along each direction $u$, a hindsight comparator $x_u\in C$ can attain a strictly negative contribution in the gradient growth.

In the remaining section, we assume that the window length $H$ is known and develop an algorithm that establishes CGEQ. Since a sliding window adversary must be blockwise, it suffices to learn against a blockwise adversary. Therefore, we partition the entire horizon into blocks $I_k = ((k-1)H,kH]$, for $k=1,\ldots,K=\lceil T/H\rceil$, and track the cumulative gradient across blocks.

Specifically, when $S_{(k-1)H}\ne0$, let $u_k = S_{(k-1)H}/\|S_{(k-1)H}\|_2$ denote the direction of the cumulative gradient, where $S_{t'}=\sum_{t=1}^{t'}g_t(x_t)$, and choose any $u_k\in\Sphere^{d-1}$ otherwise. An off-the-shelf online learning subroutine $\mathsf A_H$ provides a uniform regret guarantee $R_H$ under the directional loss $\langle u_k, g_t(x)\rangle$. Together with \eqref{eq:local-feasibility}, if $\gamma \ge \frac{R_H}{H}$, the regret guarantee implies
\begin{align}\label{eq:negative-correlation}
\sum_{t\in I_k}\langle u_k, g_t(x_t)\rangle \le 0.
\end{align}
The inequality \eqref{eq:negative-correlation} provides an appropriate condition for the telescoping analysis of $\|S_{kH}\|_2^2$. For a no-regret algorithm, $R_H/H\to 0$ as $H\to\infty$, so the condition $\gamma \ge \frac{R_H}{H}$ eventually holds for long windows. Therefore, this is a general algorithmic reduction that converts a no-regret learner within the window into a CGEQ algorithm; this procedure is detailed in Algorithm~\ref{alg:known-H-window}.

\begin{algorithm}[ht]
\caption{CGEQ reduction to no-regret learning}
\label{alg:known-H-window}
\begin{algorithmic}[1]

\State \textbf{Input:} Horizon $T$, window length $H$, online subroutine $\mathsf{A}_H$.

\State Initialize $S_0\gets 0$, $K\gets \lceil \frac{T}{H}\rceil$.

\For{block $k=1,2,\ldots,K$}
    \If{$S_{(k-1)H}\neq0$}
        \State Set direction $u_k\gets S_{(k-1)H}/\|S_{(k-1)H}\|_2$.
    \Else
        \State Choose any $u_k\in\mathbb S^{d-1}$.
    \EndIf
    \State Restart $\mathsf A_H$ with loss $f_t(x)=\langle u_k, g_t(x)\rangle$.
    
   \For{$t=(k-1)H+1,\ldots,\min\{kH,T\}$}
        \State Play the action $x_t$ selected by $\mathsf A_H$.
        \State Observe $g_t(x_t)$ and update
        $S_t\gets S_{t-1}+g_t(x_t)$.
    \EndFor
\EndFor
\end{algorithmic}
\end{algorithm}

\begin{theorem}[CGEQ under strong local feasibility]\label{thm:CGEQ-under-local-feasibility}
Suppose \eqref{eq:local-feasibility} holds with $\gamma>0$. If for any length-$H$ interval $I$ and fixed direction $u\in\mathbb{S}^{d-1}$, the subroutine $\mathsf{A}_H$ guarantees
\[
\sum_{t\in I}\langle u, g_t(x_t)\rangle - \min_{x^\star\in C}\sum_{t\in I}\langle u, g_t(x^\star)\rangle \le R_H,
\] 
then Algorithm~\ref{alg:known-H-window} guarantees
\[
\frac{1}{T}\left\|\sum_{t=1}^Tg_t(x_t)\right\|_2 \le L\sqrt{\frac{H}{T}} + \frac{LH}{T} = o(1)
\]
when $\gamma \ge \frac{R_H}{H}$.
\end{theorem}

For concrete examples of the online learning subroutine, we refer the readers to Appendix \ref{app:FLT_instance}.

\vspace{-0.5em}
\section{Experiments}
\label{sec:experiments}
\vspace{-0.6em}
\paragraph{Metrics.}
In practice, the constraint feedback $c_s(\cdot)$ may be observed only after the learner chooses the action $x_s$, making per-round feasibility difficult to guarantee.  
We therefore study CGEQ under soft constraints (i.e. the setting in Appendix \ref{app:soft-cgeq}), allowing individual violations while controlling their cumulative magnitude.  
We report the cumulative gradient imbalance $\|S_t\|_2$, where $S_t=\sum_{s=1}^t g_s(x_s)$, and the cumulative positive constraint violation $V_t=\sum_{s=1}^t[c_s(x_s)]_+$.

\vspace{-0.6em}
\paragraph{\softblackwell.}
For the soft-constraint setting, we propose \softblackwell, a practical CGEQ algorithm that combines a Blackwell-style outer controller with an FTRL inner learner (Algorithms~\ref{alg:soft-window} and~\ref{alg:window-soft-ftrl}, Appendix~\ref{app:exp-algorithms}).  
The outer controller adaptively balances gradient equilibrium and constraint satisfaction using cumulative feedback, while the inner learner minimizes regularized scalarized losses within each local window.  
This construction connects the joint control of gradient imbalance and constraint violation to local feasibility and the regret of the inner learner.

\vspace{-0.6em}
\paragraph{Baselines.}
Vanilla GD (Algorithm~\ref{alg:projected-gd}), used in \cite{angelopoulos2025gradient}, takes a gradient step from the current action, ignores the constraint in the update, and projects the resulting point onto the fixed action domain $\mathcal X$.  
Lazy GD (Algorithm~\ref{alg:lazy-gd}), used in \cite{ding2026calibratedmultilevelquantileforecasting}, instead maintains an unconstrained latent iterate and selects each action by projecting this iterate onto a specified feasible set.  
In our experiments, this feasible set is constructed from the constraints observed in the preceding completed window.  
The COCO baseline (Algorithm~\ref{alg:coco-baseline}), proposed in \cite{sinha2024optimal}, is designed to achieve both small regret and small cumulative constraint violation (CCV).  
Both GD baselines use constant step sizes, and Appendix~\ref{app:exp-algorithms} provides the complete update rules.

\subsection{A controlled binary forecasting example}
\label{sec:exp-toy}
We begin with a synthetic binary forecasting problem in which the learner predicts $p_t\in[0,1]$ for binary outcomes $z_t$, with gradient $g_t(p)=p-z_t$.  
The outcome sequence repeats $A^M B^M$, where $A$ and $B$ are distinct length-10 sequences.  
We impose a time-varying lower bound $q_t$ through the soft constraint $c_t(p)=q_t-p$.

Figure~\ref{fig:toy-main} shows that \softblackwell and Lazy GD maintain both small gradient imbalance and low cumulative violation as the regimes switch.  
Vanilla GD also keeps the gradient imbalance relatively small, but accumulates substantially larger constraint violation.  
COCO exhibits growing calibration excursions and the largest cumulative violation.  
Appendix~\ref{app:exp-toy} provides the detailed experimental setup, validation protocol, parameter choices, and results across different regime durations.

\begin{figure}[H]
  \centering
  \includegraphics[width=0.90\textwidth]{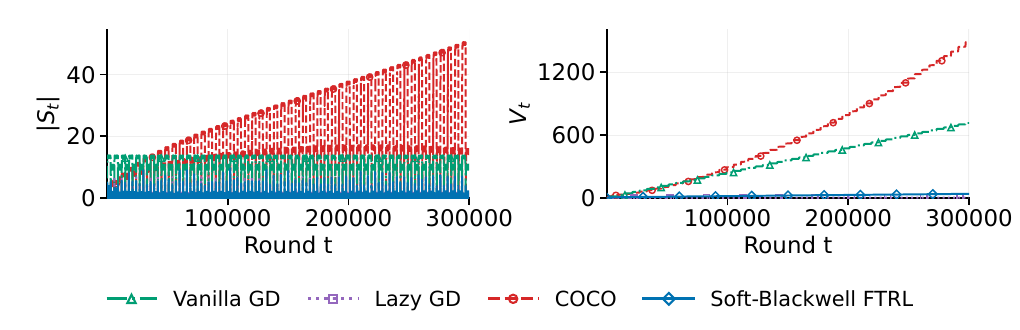}
  \caption{Binary forecasting with alternating outcome regimes: cumulative gradient imbalance $|S_t|$ (left) and constraint violation $V_t$ (right). All axes are linear.}
  \label{fig:toy-main}
  \vspace{-1em}
\end{figure}

\subsection{Calibrating LLM response toxicity}
\label{sec:exp-toxicity}
We use the public paired toxicity scores released by \citet{chen2026adversarially}.  
Each prompt has 40 candidate responses, with a machine score used for filtering and a reference score used for evaluation.  
Each prompt constitutes one online round.

The learner chooses $x_t\in[0,1]$ and accepts responses whose machine scores $a_{tj}$ do not exceed $x_t$.  
Writing $b_{tj}$ for the corresponding reference scores, define
\[
R_t(x) = \mathbf{1}_{\{\exists j: a_{tj} \le x\}}\max_{j: a_{tj} \le x} b_{tj},\qquad g_t(x)=R_t(x)-\alpha.
\]
Gradient equilibrium therefore calibrates the average maximum toxicity among accepted responses to the target level $\alpha$.  
For each window $w$, we set the budget to $q_w=(1-\tau)x^\star_w$, where $x^\star_w$ is the largest representable threshold whose mean risk over that window is at most $\alpha$.  
The soft constraint is $c_t(x)=q_w-x$ and is revealed only after the learner chooses its action.  
We use $H=10$ prompts per window and $\tau=0.1$, corresponding to 10\% relative slack in threshold space.  
Algorithm parameters are selected using the validation prompts and then held fixed during evaluation.

Figure~\ref{fig:toxicity-main} shows the evolution of the gradient imbalance and cumulative constraint violation.  
\softblackwell maintains a small cumulative gradient imbalance, while all three baselines accumulate substantially larger imbalances over the evaluation horizon.  
It also achieves the smallest final cumulative constraint violation.  
Appendix~\ref{app:toxicity-sensitivity} reports results for all four combinations of $H\in\{10,20\}$ and $\tau\in\{0.1,0.2\}$.

\begin{figure}[t]
  \centering
  \includegraphics[width=0.90\textwidth]{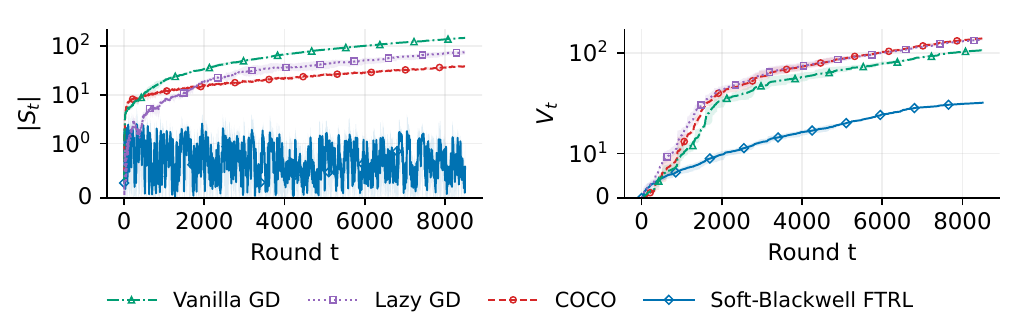}
  \caption{LLM toxicity calibration ($H=10$, $\tau=10\%$): cumulative gradient imbalance $|S_t|$ (left) and constraint violation $V_t$ (right). Lines and shading show means and min--max ranges over three prompt orders. Vertical axes use symmetric-log scaling.}
  \label{fig:toxicity-main}
  \vspace{-1em}
\end{figure}

\label{sec:exp-stock-2d}
\begin{figure}[t]
  \centering
  \includegraphics[width=0.90\textwidth]{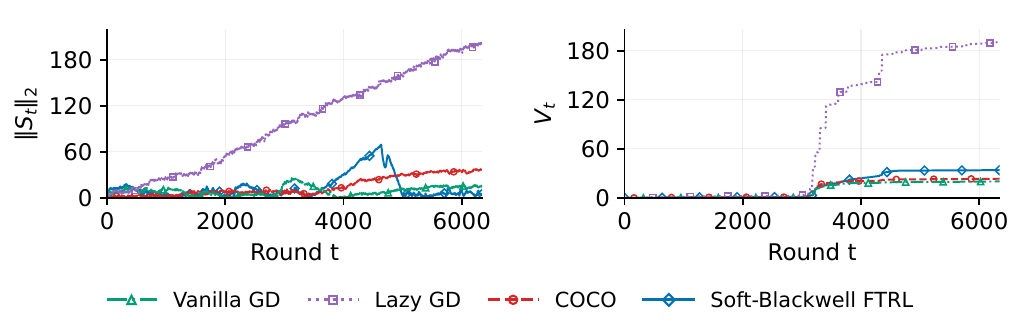}
  \caption{Simultaneous $0.9$ and $0.7$ stock coverage on FNMA with window length $H=10$: cumulative gradient imbalance $\|S_t\|_2$ (left) and constraint violation $V_t$ (right). All axes are linear.}
  \label{fig:stock-multilevel-fnma}
  \vspace{-1em}
\end{figure}

\subsection{Calibrating stock volatility prediction}

We extend the stock-volatility setting of \cite{gibbs2021adaptiveconformalinferencedistribution} to target \(0.9\) and \(0.7\) coverage simultaneously.  
For a rolling GARCH\((1,1)\) forecast \(\widehat{v}_t\) of squared return \(v_t\), the action \(x_t = (x_t^{(0.9)}, x_t^{(0.7)})\) defines intervals $I_t^{(\alpha)} = [\widehat{v}_t(1 - x_t^{(\alpha)}), \widehat{v}_t(1 + x_t^{(\alpha)})]$, for $\alpha\in\{0.9,0.7\}$.  
Writing $e_t=|v_t-\widehat v_t|/\widehat v_t$, the gradient is $g_t^{(\alpha)}(x)=1-\alpha-\mathbf{1}\{e_t>x^{(\alpha)}\}$.  
Gradient equilibrium therefore calibrates both coverage levels simultaneously.

We impose the hard constraint $0\le x_t^{(0.7)}\le x_t^{(0.9)}$ to ensure $I_t^{(0.7)}\subseteq I_t^{(0.9)}$, analogous to the noncrossing constraints in multi-level quantile forecasting \cite{ding2026calibratedmultilevelquantileforecasting}.  
We additionally impose a soft constraint to discourage excessively wide and therefore uninformative intervals.  
We partition the sequence into non-overlapping windows of length \(H\), allowing the width budgets to adapt to local volatility rather than being determined by a single global threshold.  
The soft constraint is defined as the total excess physical width,
\[
c_t(x) = 2\widehat{v}_t \sum_{\alpha\in\{0.9,0.7\}} [x^{(\alpha)} - q_w^{(\alpha+\tau)}]_+,
\]
where \(q_w^{(r)}\) denotes the empirical \(r\)-quantile of the scores in window \(w\).  
We consider window lengths $H\in\{10,20\}$, set $\tau=0.02$, and use a common chronological validation/test split across all methods.

Figure~\ref{fig:stock-multilevel-fnma} summarizes the FNMA results for $H=10$.  
\softblackwell achieves the smallest final calibration error, although it incurs larger cumulative width violation than Vanilla GD and COCO.  
Lazy GD exhibits substantially larger gradient imbalance and constraint violation.  
Vanilla GD remains competitive because the pinball subgradient directly encodes coverage error, so constant-step GD yields a telescoping calibration identity up to projection corrections.  
Appendix~\ref{app:exp-stock-2d} provides the hyperparameters, validation protocol, and extended results on three additional stocks.

\clearpage

\bibliographystyle{plainnat}
\bibliography{Arxiv/refs}

\section{Extension to soft constraints}
\label{app:soft-cgeq}

The main text focuses primarily on hard learner-side constraints.
In this section, we show that the Blackwell and vector-payoff game 
viewpoints also provide a unified treatment of soft constraints and
comparator-relative notions of gradient equilibrium.

We proceed in four steps.
First, we consider the blockwise model.
Because admissibility is imposed separately on each block,
Blackwell approachability gives an exact characterization of a general
soft-constraint benchmark, and in particular of the benchmark induced
by a constraint-feasible blockwise comparator.

Second, before returning to sliding windows, we extend the exact
sliding-window characterization in the main text from the zero target
to an arbitrary closed convex target.
This result will allow us to treat the local-comparator benchmark under
sliding-window constraints as an ordinary convex-set approachability
problem.

Third, we consider sliding-window adversaries with a local comparator.
Since this benchmark depends only on the current sliding transition,
the original sliding state remains sufficient, and under the finite-step
bridging assumption we obtain an exact characterization.

Finally, we consider the best fixed comparator in hindsight.
This benchmark depends on cumulative learner and comparator payoffs, so
we formulate a time-homogeneous augmented game. We characterize uniform
approachability of this augmented game and then obtain the original
zero-initialized best-fixed objective as a sufficient consequence.

We conclude the appendix with a binary-forecasting counterexample
showing that pointwise compatibility alone does not guarantee both comparator-relative GEQ and sublinear cumulative constraint violation.

\subsection{The blockwise model}
\label{app:soft-blockwise}

\subsubsection{A general blockwise soft-constraint benchmark}
\label{app:soft-blockwise-benchmark}

For block $I_k$, define the cumulative vector payoff
\[
    \blk_k(\pi,z)
    :=
    \sum_{t\in I_k}
    \poff(x_t^\pi,z_t),
\]
and the block constraint violation
\[
    V_k(\pi,z)
    :=
    \sum_{t\in I_k}
    [\cstr(x_t^\pi,z_t)]_+.
\]
We assume that all blocks are copies of the same $H$-round game,
with a common admissible block-outcome set $\WinSet$.
After relabeling rounds within each block, write
\[
\blk_k(\pi,z)=\blk_H(\pi,z),
\qquad
V_k(\pi,z)=V_H(\pi,z).
\]
Let $b_H:\WinSet\to[0,\infty)$ be a bounded benchmark, and set
\[
b_k(z):=b_H(z)
\qquad\text{for every block }k.
\]

We consider the simultaneous objectives
\begin{equation}
\label{eq:soft-benchmark-objective}
\left[
    \left\|
        \sum_{k=1}^K \blk_k
    \right\|_2
    -
    \sum_{k=1}^K b_k
\right]_+
=
o(KH)
\end{equation}
and
\begin{equation}
\label{eq:soft-ccv-objective}
\sum_{k=1}^K V_k
=
o(KH).
\end{equation}

Several notions are obtained by different choices of $b_k$.
For example, $b_k\equiv0$ gives ordinary GEQ together with sublinear
cumulative constraint violation.

Define the augmented payoff
\[
    \Gamma_k(\pi,z)
    :=
    \bigl(
        \blk_k(\pi,z),
        b_k(z),
        V_k(\pi,z)
    \bigr)
    \in
    \mathbb R^{d+2},
\]
and the closed convex cone
\begin{equation}
\label{eq:soft-target-cone}
\mathcal K
:=
\left\{
    (r,s,v)\in
    \mathbb R^d\times\mathbb R\times\mathbb R:
    \|r\|_2\le s,\;
    v\le0
\right\}.
\end{equation}

Let
\[
    \mathbb B_2^d
    :=
    \{w\in\mathbb R^d:\|w\|_2\le1\}.
\]
\begin{theorem}[Exact Blackwell characterization for soft constrained equilibrium]
\label{thm:soft-blockwise-blackwell}
Suppose that the augmented block payoff
\[
    \Gamma_H(\pi,z)
    :=
    \bigl(
        \blk_H(\pi,z),
        b_H(z),
        V_H(\pi,z)
    \bigr)
\]
is uniformly bounded over $\pi\in\Pi_H$ and $z\in\WinSet$.
Then the following are equivalent.

\begin{enumerate}
\item
The learner can guarantee
\eqref{eq:soft-benchmark-objective} and
\eqref{eq:soft-ccv-objective}
against every blockwise adversary.

\item
The cone $\mathcal K$ in \eqref{eq:soft-target-cone}
is approachable in the induced augmented block game.

\item
The following condition holds:
\begin{equation}
\label{eq:soft-blackwell-condition}
    \sup_{\substack{
        \vartheta\in[0,1]\\
        w\in\mathbb B_2^d
    }}
    \inf_{\pi\in\Pi_H}
    \sup_{z\in\WinSet}
    \left\{
        \vartheta
        \left[
            \left\langle
                w,\blk_H(\pi,z)
            \right\rangle
            -
            b_H(z)
        \right]
        +
        (1-\vartheta)V_H(\pi,z)
    \right\}
    \le0.
\end{equation}
\end{enumerate}
\end{theorem}

\begin{proof}
Since $b_k\ge0$ and $V_k\ge0$, the product structure of
$\mathcal K$ gives
\[
\begin{aligned}
    \operatorname{dist}
    \left(
        \frac1K\sum_{k=1}^K\Gamma_k,
        \mathcal K
    \right)^2
    =
    \frac{1}{2K^2}
    \left[
        \left\|\sum_{k=1}^K\blk_k\right\|_2
        -
        \sum_{k=1}^K b_k
    \right]_+^2+
    \frac{1}{K^2}
    \left(\sum_{k=1}^K V_k\right)^2.
\end{aligned}
\]
This proves the equivalence
of the first two statements.

Regard each original $H$-round block as one stage.
Since blockwise admissibility imposes no restrictions across
blocks, the induced game has a single effective state.
Every potential on this singleton state space is constant,
so the potential correction vanishes.
Applying Theorem~\ref{thm:sliding-convex-approachability}
to this game, we conclude
that $\mathcal K$ is uniformly approachable if and only if
\begin{equation}\label{eq:support_condition}
    \inf_{\pi\in\Pi_H}
    \sup_{z\in\WinSet}
    \left\langle
        a,\Gamma_H(\pi,z)
    \right\rangle
    \le
    \sigma_{\mathcal K}(a),
    \qquad
    \forall a\in\mathbb R^{d+2}.
\end{equation}

For $a=(u,\lambda,\mu)\in\mathbb R^d\times\mathbb R\times\mathbb R$,
the support function of $\mathcal K$ is
\[
    \sigma_{\mathcal K}(u,\lambda,\mu)
    =
    \begin{cases}
        0,
        & \lambda\le-\|u\|_2
          \text{ and }\mu\ge0,\\
        +\infty,
        & \text{otherwise}.
    \end{cases}
\]
Consequently, only directions satisfying
$\lambda\le-\|u\|_2$ and $\mu\ge0$ impose a restriction.

Every nonzero such direction can be written as
\[
    (u,\lambda,\mu)
    =
    \rho
    \bigl(
        \vartheta w,
        -\vartheta,
        1-\vartheta
    \bigr),
    \qquad
    \rho:=-\lambda+\mu>0,
\]
where
\[
    \vartheta:=\frac{-\lambda}{-\lambda+\mu}\in[0,1],
    \qquad
    w\in\mathbb B_2^d.
\]
Indeed, when $\lambda<0$, take $w=u/(-\lambda)$;
when $\lambda=0$, necessarily $u=0$, and we may take
$\vartheta=0$ and $w=0$.
Conversely, every
$(\vartheta,w)\in[0,1]\times\mathbb B_2^d$
generates a valid support direction
$(\vartheta w,-\vartheta,1-\vartheta)$.

Positive homogeneity therefore identifies \eqref{eq:support_condition} with \eqref{eq:soft-blackwell-condition},
proving the equivalence with the third statement.
\end{proof}

For an admissible window
$z=(z_1,\ldots,z_H)\in\WinSet$, define
\[
\cQ(z)
:=
\bigcap_{h=1}^H
\{q\in\cX:\cstr(q,z_h)\le0\},
\]
and assume that
$\cQ(z)\neq\varnothing$ for every $z\in\WinSet$.
One benchmark is
\[
b_H(z)
:=
\inf_{q\in\cQ(z)}\norm{\blk_H(q,z)},
\] where we identify a fixed action $q$ with the constant block policy so that $\blk_H(q,z)=\sum_{h=1}^H g(q,z_h)$,
 measuring the smallest block residual achievable by a comparator
that is feasible throughout the block.

Similar to \eqref{eq:weak_local_feasibility}, we say that in the soft constraint setting local feasibility holds if,
for every admissible $z$, there exists a comparator
$q\in\cQ(z)$ such that $\blk_H(q,z)=0$.
Under this condition, the comparator-induced benchmark defined
above satisfies $b_H(z)=0$ for all admissible $z$.
We next derive an equivalent formulation of \eqref{eq:soft-blackwell-condition}
under local feasibility.

\begin{corollary}[Equivalent form under local feasibility]
\label{cor:soft blackwell equivalent form}
Fix a block length $H\ge1$ and suppose that the set of admissible
adversarial block trajectories is independent of the learner's actions
within the block.
Let $\Pi_H$ denote the space of causal $H$-round block policies, and
suppose that $\Pi_H$ is compact under a topology such that, for every
admissible block trajectory $z$, the maps
\[
    \pi
    \longmapsto
    G_H(\pi,z)
\]
and
\[
    \pi
    \longmapsto
    V_H(\pi,z)
\]
are continuous, where
\[
    V_H(\pi,z)
    :=
    \sum_{h=1}^H
    [c(x_h^\pi(z),z_h)]_+.
\]
Assume further that the block vector payoff is uniformly bounded:
there exists $M_H<\infty$ such that
\[
    \|G_H(\pi,z)\|_2\le M_H
\]
for every $\pi\in\Pi_H$ and every admissible $z$.

For $w\in\mathbb S^{d-1}$, define
\[
    \Gamma_{H,w}(\pi)
    :=
    \sup_z
    \left\langle
        w,G_H(\pi,z)
    \right\rangle
\]
and
\[
    \mathsf V_H(\pi)
    :=
    \sup_z
    V_H(\pi,z),
\]
where the suprema are over all admissible adversarial trajectories.

Then
\[
    \inf_{\pi\in\Pi_H}
    \sup_z
    \left\{
        \vartheta
        \left\langle
            w,G_H(\pi,z)
        \right\rangle
        +
        (1-\vartheta)V_H(\pi,z)
    \right\}
    \le0,
    \qquad
    \forall\vartheta\in[0,1],\
    \forall w\in\mathbb S^{d-1},
    \tag{\ref{eq:soft-blackwell-condition}$'$}
\]
is equivalent to
\[
    \forall w\in\mathbb S^{d-1},
    \qquad
    \exists\pi_w\in\Pi_H
\]
such that
\begin{equation}
\label{eq:robust-feasible-directional-block}
    \mathsf V_H(\pi_w)=0
    \qquad\text{and}\qquad
    \Gamma_{H,w}(\pi_w)\le0.
\end{equation}
\end{corollary}

\begin{proof}
Since
$\pi\mapsto G_H(\pi,z)$ and
$\pi\mapsto V_H(\pi,z)$ are continuous for every fixed admissible
trajectory $z$, the functions
\[
    \Gamma_{H,w}(\pi)
    =
    \sup_z
    \langle w,G_H(\pi,z)\rangle
\]
and
\[
    \mathsf V_H(\pi)
    =
    \sup_zV_H(\pi,z)
\]
are lower semicontinuous on $\Pi_H$.

Suppose first that
\eqref{eq:robust-feasible-directional-block} holds.
Fix $\vartheta\in[0,1]$ and
$w\in\mathbb S^{d-1}$, and choose $\pi_w$ such that
\[
    \mathsf V_H(\pi_w)=0,
    \qquad
    \Gamma_{H,w}(\pi_w)\le0.
\]
Then
\[
\begin{aligned}
&
\sup_z
\left\{
    \vartheta
    \langle w,G_H(\pi_w,z)\rangle
    +
    (1-\vartheta)V_H(\pi_w,z)
\right\}
\\
&=
\vartheta
\sup_z
\langle w,G_H(\pi_w,z)\rangle
\\
&=
\vartheta\Gamma_{H,w}(\pi_w)
\le0.
\end{aligned}
\]

Conversely, suppose the zero-budget Blackwell condition holds and fix
$w\in\mathbb S^{d-1}$.
Let
\[
    \vartheta_n\downarrow0
\]
and choose $\varepsilon_n>0$ such that
\[
    \frac{\varepsilon_n}{\vartheta_n}
    \longrightarrow0.
\]
For each $n$, choose $\pi_n\in\Pi_H$ satisfying
\[
    \sup_z
    \left\{
        \vartheta_n
        \langle w,G_H(\pi_n,z)\rangle
        +
        (1-\vartheta_n)V_H(\pi_n,z)
    \right\}
    \le
    \varepsilon_n.
\]

Since $V_H(\pi_n,z)\ge0$,
\[
    \Gamma_{H,w}(\pi_n)
    \le
    \frac{\varepsilon_n}{\vartheta_n}
    \longrightarrow0.
\]
Also,
\[
    \langle w,G_H(\pi_n,z)\rangle
    \ge-M_H,
\]
and therefore
\[
    \mathsf V_H(\pi_n)
    \le
    \frac{
        \varepsilon_n+\vartheta_nM_H
    }{
        1-\vartheta_n
    }
    \longrightarrow0.
\]

By compactness of $\Pi_H$, the sequence $\{\pi_n\}$ admits a
convergent subnet
$\{\pi_{n_\alpha}\}_{\alpha\in A}$.
For some $\pi_w\in\Pi_H$,
\[
    \pi_{n_\alpha}
    \longrightarrow
    \pi_w.
\]
Cofinality of the subnet implies $n_\alpha\to\infty$.
By lower semicontinuity,
\[
    \Gamma_{H,w}(\pi_w)\le0,
    \qquad
    \mathsf V_H(\pi_w)\le0.
\]
Since $\mathsf V_H\ge0$,
\[
    \mathsf V_H(\pi_w)=0.
\]
This proves
\eqref{eq:robust-feasible-directional-block}.
\end{proof}

\begin{remark}
\label{rem:robust-feasible-block-policy}
The equation
\[
    \mathsf V_H(\pi)=0
\]
is equivalent to requiring
\[
    c(x_h^\pi(z),z_h)\le0,
    \qquad
    h\in[H],
\]
along every admissible block trajectory $z$.
Hence
\[
    x_h^\pi
    \in
    C_h^{\mathrm{rob}}(H_{h-1}),
\]
where $\mathcal H_{h-1}$ denotes the history available to the learner before round $h$ of the block, and
\[
    C_h^{\mathrm{rob}}(H_{h-1})
    :=
    \left\{
        x\in\mathcal X:
        c(x,z_h)\le0
        \text{ for every admissible} \ z\  \text{following } H_{h-1}
    \right\}.
\]
\end{remark}

\begin{remark}[When the compactness and continuity assumptions hold]
\label{rem:When the compactness and continuity assumptions hold}
When $H=1$, a block policy is simply an action $x\in\mathcal X$.
Hence it is enough to assume that $\mathcal X$ is compact and
\[
    x\mapsto g_t(x),
    \qquad
    x\mapsto[c_t(x)]_+
\]
are continuous.

For $H>1$, since the learner's past actions are recursively determined
by the policy together with the previously observed adversarial
outcomes, a deterministic causal block policy may equivalently be
represented as
\[
    \pi=(\pi_1,\ldots,\pi_H),
    \qquad
    \pi_h:\mathcal Z^{h-1}\to\mathcal X.
\]
Accordingly, the full causal policy space may be written as
\[
    \Pi_H
    =
    \prod_{h=1}^H
    \mathcal X^{\mathcal Z^{h-1}},
\]
equipped with the product topology.
If $\mathcal X$ is compact, then $\Pi_H$ is compact by Tychonoff's
theorem.

Now fix an admissible adversarial trajectory
\[
    z=(z_1,\ldots,z_H).
\]
For this fixed trajectory, the learner's action at round $h$ is
\[
    x_h^\pi(z)
    =
    \pi_h(z_1,\ldots,z_{h-1}).
\]
Since $(z_1,\ldots,z_{h-1})$ is fixed, the map
\[
    \pi
    \longmapsto
    \pi_h(z_1,\ldots,z_{h-1})
\]
is simply a coordinate evaluation on the product space $\Pi_H$, and is
therefore continuous.

Hence, if the one-step payoff and constraint mappings are continuous in
the learner's action, then for every fixed admissible trajectory $z$,
\[
    \pi\longmapsto G_H(\pi,z),
    \qquad
    \pi\longmapsto V_H(\pi,z)
\]
are continuous.

Therefore
\[
    \pi
    \longmapsto
    \sup_z
    \left\langle
        w,G_H(\pi,z)
    \right\rangle
\]
and
\[
    \pi
    \longmapsto
    \sup_z
    V_H(\pi,z)
\]
are lower semicontinuous, since they are pointwise suprema of continuous
functions. Hence the conditions of Corollary~\ref{cor:soft blackwell equivalent form} hold.
\end{remark}

\subsubsection{A blockwise comparator}

For each admissible block outcome $z\in\WinSet$, set
\[
\cQ_k(z):=\cQ(z),
\qquad
b_H(z):=\inf_{q\in\cQ(z)}\norm{\blk_H(q,z)},
\qquad
b_k(z):=b_H(z).
\]

\begin{corollary}[Blockwise comparator-relative GEQ]
\label{cor:blockwise-relative-geq}
The learner can guarantee
\[
    \left[
        \left\|
            \sum_{k=1}^K \blk_k
        \right\|_2
        -
        \sum_{k=1}^K
        \inf_{q\in Q_k(z_k)}
        \|\blk_k(q,z_k)\|_2
    \right]_+
    =
    o(KH)
\]
together with
\[
    \sum_{k=1}^K V_k=o(KH)
\]
if and only if the corresponding condition
\eqref{eq:soft-blackwell-condition} holds.
\end{corollary}

\begin{proof}
This is Theorem~\ref{thm:soft-blockwise-blackwell} with
\[
    b_k(z)
    =
    \inf_{q\in Q_k(z)}
    \|\blk_k(q,z)\|_2.
\]
\end{proof}

\subsection{Approachability of a closed convex set under sliding-window constraints}
\label{subsec:sliding-convex-approachability}

We next extend the sliding-window characterization from the target
$\{0\}$ to an arbitrary nonempty closed convex set
$\mathcal D\subseteq\mathbb R^d$.

The learner is said to approach $\mathcal D$ if
\[
    \operatorname{dist}
    \left(
        \frac{\cum_T}{T},
        \mathcal D
    \right)
    \longrightarrow0
\]
uniformly over sliding-window admissible adversaries.

\begin{theorem}[Sliding-window approachability of a closed convex set]
\label{thm:sliding-convex-approachability}
A nonempty closed convex set
$\mathcal D\subseteq\mathbb R^d$ is uniformly approachable if and only if
\begin{equation}
\label{eq:sliding-convex-CW}
    \inf_{\Potential\in\ell^\infty(\WinSet)}
    \sup_{s\in\WinSet}
    \left\{
        (\mathcal T_{u,H}\Potential)(s)
        -
        \Potential(s)
    \right\}
    \le
    H\sigma_{\mathcal D}(u),
    \qquad
    \forall u\in\mathbb S^{d-1}.
\end{equation}
\end{theorem}
\begin{proof}
We first record the support-function representation
\begin{equation}
\label{eq:distance-support-representation}
\operatorname{dist}(x,T\mathcal D)
=
\sup_{u\in\mathbb B_2^d}
\left\{
    \langle u,x\rangle
    -
    T\sigma_{\mathcal D}(u)
\right\},
\qquad
T\ge1.
\end{equation}
Indeed, for every $u\in\mathbb B_2^d$ and $c\in\mathcal D$,
\[
    \langle u,x\rangle
    -
    T\sigma_{\mathcal D}(u)
    \le
    \langle u,x-Tc\rangle
    \le
    \|x-Tc\|_2.
\]
Taking the infimum over $c$ proves one inequality.
For the reverse inequality, let
\[
    y:=\proj_{T\mathcal D}(x).
\]
If $x\notin T\mathcal D$, define
\[
    u^\star
    :=
    \frac{x-y}{\|x-y\|_2}.
\]
The projection optimality condition gives
\[
    \langle u^\star,z-y\rangle\le0,
    \qquad
    \forall z\in T\mathcal D.
\]
Hence
\[
    T\sigma_{\mathcal D}(u^\star)
    =
    \langle u^\star,y\rangle,
\]
and therefore
\[
    \langle u^\star,x\rangle
    -
    T\sigma_{\mathcal D}(u^\star)
    =
    \|x-y\|_2.
\]
If $x\in T\mathcal D$, both sides of
\eqref{eq:distance-support-representation} are zero.

\paragraph{Necessity.}
Suppose $\mathcal D$ is uniformly approachable.
Then there exists $r(T)=o(T)$ such that, uniformly over the initial
state and every sliding-window admissible adversary,
\[
    \operatorname{dist}(\cum_T,T\mathcal D)
    \le
    r(T).
\]
Fix $u\in\mathbb S^{d-1}$.
If $\sigma_{\mathcal D}(u)=+\infty$, the desired inequality is
automatic. Otherwise, by
\eqref{eq:distance-support-representation},
\[
    \langle u,\cum_T\rangle
    -
    T\sigma_{\mathcal D}(u)
    \le
    r(T).
\]
In particular, truncating the uniformly approaching strategy after
$n$ blocks gives, for every initial state $s$,
\[
    V_n^u(s)
    \le
    nH\sigma_{\mathcal D}(u)
    +
    r(nH).
\]
Hence
\[
\lim_{n\to\infty}
\sup_s\frac{V_n^u(s)}{n}
\le
H\sigma_{\mathcal D}(u).
\]
Lemma~\ref{lem:supinf-relation} identifies the left-hand side with
the upper value, proving
\eqref{eq:sliding-convex-CW}.

\paragraph{Sufficiency.}
Suppose \eqref{eq:sliding-convex-CW} holds.

Set
\[
    \overline L:=\frac{M_H}{H}.
\]
Every time-average payoff $x=\cum_T/T$ satisfies
\[
    \|x\|_2\le\overline L.
\]
Fix $c_0\in\mathcal D$ and define
\[
    R_{\mathcal D}
    :=
    2\overline L+\|c_0\|_2.
\]
If $\|x\|_2\le\overline L$ and
$c=\proj_{\mathcal D}(x)$, then
\[
\begin{aligned}
\|c\|_2
&\le
\|c-x\|_2+\|x\|_2
\\
&\le
\|c_0-x\|_2+\|x\|_2
\\
&\le
R_{\mathcal D}.
\end{aligned}
\]

We therefore introduce the set of relevant supporting directions
\[
\mathcal N_{\mathcal D}
:=
\left\{
    u\in\mathbb S^{d-1}:
    \begin{array}{l}
    \text{there exists }c\in\mathcal D
    \text{ with }\|c\|_2\le R_{\mathcal D},
    \\[-1mm]
    \sigma_{\mathcal D}(u)=\langle u,c\rangle
    \end{array}
\right\}.
\]
Hence every nonzero projection residual generated by a bounded average payoff
has its normalized direction in $\mathcal N_{\mathcal D}$.

The set $\mathcal N_{\mathcal D}$ is compact.
Indeed, if $u_j\to u$ in $\mathbb S^{d-1}$ and
$c_j\in\mathcal D$, $\|c_j\|_2\le R_{\mathcal D}$ satisfy
\[
    \sigma_{\mathcal D}(u_j)
    =
    \langle u_j,c_j\rangle,
\]
then, after passing to a subsequence,
$c_j\to c\in\mathcal D$ with
$\|c\|_2\le R_{\mathcal D}$.
For every $c'\in\mathcal D$,
\[
    \langle u_j,c'\rangle
    \le
    \langle u_j,c_j\rangle.
\]
Passing to the limit gives
\[
    \langle u,c'\rangle
    \le
    \langle u,c\rangle,
\]
so $u\in\mathcal N_{\mathcal D}$.

Moreover, $\sigma_{\mathcal D}$ is
$R_{\mathcal D}$-Lipschitz on
$\mathcal N_{\mathcal D}$.
Indeed, if $u,v\in\mathcal N_{\mathcal D}$ and $c_u$ is a supporting
point for $u$ with $\|c_u\|_2\le R_{\mathcal D}$, then
\[
\begin{aligned}
\sigma_{\mathcal D}(u)-\sigma_{\mathcal D}(v)
&=
\langle u,c_u\rangle-\sigma_{\mathcal D}(v)
\\
&\le
\langle u-v,c_u\rangle
\\
&\le
R_{\mathcal D}\|u-v\|_2.
\end{aligned}
\]
Interchanging $u$ and $v$ gives the reverse bound.

For $u\in\mathcal N_{\mathcal D}$, define
\[
    f_n(u)
    :=
    \frac1n
    \sup_{s\in\WinSet}V_n^u(s)
    -
    H\sigma_{\mathcal D}(u).
\]
For any $u,v\in\mathbb S^{d-1}$, boundedness of the block payoff gives
\[
\left|
    \frac1n\sup_sV_n^u(s)
    -
    \frac1n\sup_sV_n^v(s)
\right|
\le
M_H\|u-v\|_2.
\]
Consequently,
\[
    |f_n(u)-f_n(v)|
    \le
    (M_H+HR_{\mathcal D})
    \|u-v\|_2,
    \qquad
    u,v\in\mathcal N_{\mathcal D},
\]
uniformly in $n$.

By Lemma~\ref{lem:supinf-relation} and
\eqref{eq:sliding-convex-CW},
\[
\lim_{n\to\infty}f_n(u)
=
\overline\lambda_H(u)
-
H\sigma_{\mathcal D}(u)
\le0
\qquad
\forall u\in\mathcal N_{\mathcal D}.
\]
Compactness of $\mathcal N_{\mathcal D}$ and the common Lipschitz
constant imply, by a finite-net argument,
\begin{equation}
\label{eq:convex-uniform-directional-error}
    \delta_n
    :=
    \left[
        \sup_{u\in\mathcal N_{\mathcal D}}
        f_n(u)
    \right]_++ \frac{1}{n}
    \longrightarrow0.
\end{equation}
For completeness, given $\varepsilon>0$, choose a finite
$\varepsilon/(2(M_H+HR_{\mathcal D}))$-net
$\{u^1,\ldots,u^J\}$ of $\mathcal N_{\mathcal D}$.
For all sufficiently large $n$,
$f_n(u^j)\le\varepsilon/2$ for every $j$.
The Lipschitz bound then gives
$f_n(u)\le\varepsilon$ for every
$u\in\mathcal N_{\mathcal D}$, proving
\eqref{eq:convex-uniform-directional-error}.

We now construct a single anytime learner strategy.
Let epoch $e$ contain
\[
    n_e:=e
\]
blocks. Set
\[
    K_0:=0,
    \qquad
    K_e:=\sum_{j=1}^e j=\frac{e(e+1)}2,
\]
and let
\[
    T_e:=HK_{e-1}
\]
be the number of rounds completed before epoch $e$.
Write
\[
    S_e:=\cum_{T_e},
    \qquad
    D_e:=\operatorname{dist}(S_e,T_e\mathcal D).
\]

Suppose first that $D_e>0$.
Write
\[
    \proj_{T_e\mathcal D}(S_e)
    =
    T_ec_e
\]
and set
\[
    u_e
    :=
    \frac{S_e-T_ec_e}{D_e}.
\]
Then
\[
    u_e\in\mathcal N_{\mathcal D},
    \qquad
    \sigma_{\mathcal D}(u_e)
    =
    \langle u_e,c_e\rangle.
\]
At the beginning of epoch $e$, choose a $1$-optimal $e$-block policy for direction $u_e$
at the current state.
If $\Delta_e$ denotes the cumulative vector payoff over that epoch,
then, uniformly over the current state and every admissible
continuation,
\[
\begin{aligned}
    \langle u_e,\Delta_e\rangle
    &\le \sup_s V_e^{u_e}(s)+1\\
    &\le eH\sigma_D(u_e)+e\delta_e.
\end{aligned}
\]
Therefore
\begin{equation}
\label{eq:convex-epoch-cross-term}
    \langle
        u_e,
        \Delta_e-eHc_e
    \rangle
    \le
    e\delta_e.
\end{equation}

If $D_e=0$, choose $c_e\in\mathcal D$ such that
\[
    S_e=T_ec_e
\]
when $T_e>0$, and take $c_e=c_0$ when $T_e=0$.
In this case choose any admissible $e$-block policy.
The cross term in the distance recursion below is then zero.

In either case,
\[
    \|c_e\|_2\le R_{\mathcal D},
\]
and
\[
    \|\Delta_e\|_2\le eM_H.
\]
Thus, with
\[
    B_{\mathcal D}
    :=
    M_H+HR_{\mathcal D},
\]
we have
\[
    \|\Delta_e-eHc_e\|_2
    \le
    eB_{\mathcal D}.
\]

Since
\[
    (T_e+eH)c_e
    \in
    (T_e+eH)\mathcal D,
\]
the next projection distance satisfies
\[
\begin{aligned}
D_{e+1}^2
&\le
\left\|
    S_e-T_ec_e
    +
    \Delta_e-eHc_e
\right\|_2^2
\\
&\le
D_e^2
+
2D_e\,e\delta_e
+
B_{\mathcal D}^2e^2,
\end{aligned}
\]
where \eqref{eq:convex-epoch-cross-term} is used when $D_e>0$,
while the cross term vanishes when $D_e=0$.

An elementary induction yields
\[
    D_{E+1}
    \le
    \sum_{e=1}^E e\delta_e
    +
    B_{\mathcal D}
    \left(
        \sum_{e=1}^E e^2
    \right)^{1/2}.
\]
Since $\delta_e\to0$,
\[
    \sum_{e=1}^E e\delta_e=o(E^2).
\]

Hence
\[
    D_{E+1}=o(E^2).
\]
On the other hand,
\[
    HK_E=\Theta(HE^2),
\]
so
\[
    \frac{D_{E+1}}{HK_E}\longrightarrow0.
\]

Finally, consider a time lying inside epoch $e$ after $q\le e$
completed blocks.
Let $\Delta_{e,q}$ be the corresponding partial epoch payoff.
Then
\[
\begin{aligned}
&
\operatorname{dist}
\left(
    S_e+\Delta_{e,q},
    (T_e+qH)\mathcal D
\right)
\\
&\qquad\le
D_e
+
\|\Delta_{e,q}-qHc_e\|_2
\\
&\qquad\le
D_e+B_{\mathcal D}e.
\end{aligned}
\]
Since $T_e=\Theta(He^2)$, the right-hand side is $o(T_e)$.
To cover a final incomplete block, let
$N=H\lfloor T/H\rfloor$ and $r=T-N<H$.
For $N>0$, let
$c_N=\proj_{\cD}(\cum_N/N)$, so that
$\norm{c_N}\le R_{\cD}$.
Then
\[
\operatorname{dist}(\cum_T,T\cD)
\le
\operatorname{dist}(\cum_N,N\cD)
+
r(L+R_{\cD})
=
o(T).
\]
Therefore
\[
    \operatorname{dist}
    \left(
        \frac{\cum_T}{T},
        \mathcal D
    \right)
    \longrightarrow0
\]
uniformly over  all sliding-window admissible
adversaries.
\end{proof}

\subsection{Sliding-window adversaries with soft constraints}
\label{subsec:sliding-soft}

We now return to sliding-window adversaries.
We use the same state space $\WinSet$, successor correspondence
$\WinSet(s)$, and causal block-policy space $\Pi_H$ as in
Subsection~\ref{subsec:sliding-window-theory}.

\subsubsection{A blockwise comparator}

We first extend the blockwise comparator benchmark to the sliding-window setting.
This benchmark depends only on the current sliding transition and
therefore requires no additional history state.

For
\[
    s=(z_1,\ldots,z_H),
    \qquad
    s'=(z'_1,\ldots,z'_H)\in\WinSet(s),
\]
define the $r$th sliding window generated during the transition
from $s$ to $s'$ by
\[
    s^{(r)}(s,s')
    :=
    (z_{r+1},\ldots,z_H,z'_1,\ldots,z'_r),
    \qquad
    r\in[H],
\]
where
\[
    s^{(H)}(s,s')=s'.
\]
Let
\[
    \mathcal Q(s^{(r)}(s,s'))
\]
denote the corresponding feasible comparator set.
For a realized outcome sequence, we use the time-indexed notation
\[
\cQ_t
:=
\cQ(z_{t-H+1},\ldots,z_t)
=
\bigcap_{\tau=t-H+1}^{t}
\{q\in\cX:\cstr(q,z_\tau)\le0\},
\qquad t\ge H.
\]

Define the local-comparator benchmark associated with the transition
$s\to s'$ by
\begin{equation}
\label{eq:sliding-local-benchmark}
    b_H^{\mathrm{loc}}(s,s')
    :=
    \frac1H
    \sum_{r=1}^H
    \inf_{y\in\mathcal Q(s^{(r)}(s,s'))}
    \left\|
        \blk_H\bigl(y,s^{(r)}(s,s')\bigr)
    \right\|_2.
\end{equation}

For
\[
    \theta\in[0,1],
    \qquad
    w\in\mathbb B_2^d,
\]
define
\begin{equation}
\label{eq:sliding-local-operator}
\begin{aligned}
(\mathcal T^{\mathrm{loc}}_{\theta,w,H}\Potential)(s)
:=
\inf_{\pi\in\Pi_H}
\sup_{s'\in\WinSet(s)}
\Bigl\{
&
\theta
\Bigl[
    \langle w,\blk_H(\pi,s')\rangle
    -
    b_H^{\mathrm{loc}}(s,s')
\Bigr]
\\
&+
(1-\theta)V_H(\pi,s')
+
\Potential(s')
\Bigr\}.
\end{aligned}
\end{equation}

\begin{theorem}[Sliding-window soft constraints with a local comparator]
\label{thm:sliding-soft-local}
Under the same regularity and finite-step bridging assumptions as in
the sliding-window characterization, the learner can guarantee
\[
    \left[
        \|\cum_T\|_2
        -
        \frac1H
        \sum_{t=H}^{T}
        \inf_{y\in\mathcal Q_t}
        \left\|
            \sum_{\tau=t-H+1}^{t}
            \poff(y,z_\tau)
        \right\|_2
    \right]_+
    =
    o(T)
\]
and
\[
    V_T=o(T)
\]
against every sliding-window admissible adversary if and only if
\begin{equation}
\label{eq:sliding-local-explicit}
    \inf_{\Potential\in\ell^\infty(\WinSet)}
    \sup_{s\in\WinSet}
    \left\{
        (\mathcal T^{\mathrm{loc}}_{\theta,w,H}\Potential)(s)
        -
        \Potential(s)
    \right\}
    \le0
\end{equation}
for every
\[
    \theta\in[0,1],
    \qquad
    w\in\mathbb B_2^d.
\]
\end{theorem}

\begin{proof}
Let
\[
    B_T^{\mathrm{loc}}
    :=
    \frac1H
    \sum_{t=H}^{T}
    \inf_{y\in\mathcal Q_t}
    \left\|
        \sum_{\tau=t-H+1}^{t}
        \poff(y,z_\tau)
    \right\|_2.
\]
By construction of the local benchmark, the cumulative augmented payoff
is
\[
    \bigl(
        \cum_T,
        B_T^{\mathrm{loc}},
        V_T
    \bigr),
\]
up to $O(H)$ initial boundary terms, which are immaterial for an
$o(T)$ guarantee.

Consider
\[
    \mathcal D_{\mathrm{loc}}
    :=
    \left\{
        (r,b,c)\in
        \mathbb R^d\times\mathbb R\times\mathbb R:
        \|r\|_2\le b,\;
        c\le0
    \right\}.
\]
Since
\[
    B_T^{\mathrm{loc}}\ge0,
    \qquad
    V_T\ge0,
\]
we have
\[
\operatorname{dist}
\left(
    (\cum_T,B_T^{\mathrm{loc}},V_T),
    \mathcal D_{\mathrm{loc}}
\right)^2
=
\frac12
\left[
    \|\cum_T\|_2-B_T^{\mathrm{loc}}
\right]_+^2
+
V_T^2.
\]
Hence approachability of $\mathcal D_{\mathrm{loc}}$ is equivalent to
the two stated guarantees.

It therefore remains to apply
Theorem~\ref{thm:sliding-convex-approachability} and compute the support
function of $\mathcal D_{\mathrm{loc}}$.
For
\[
    \xi=(u,\lambda,\mu)
    \in
    \mathbb R^d\times\mathbb R\times\mathbb R,
\]
\[
\begin{aligned}
\sigma_{\mathcal D_{\mathrm{loc}}}(u,\lambda,\mu)
&=
\sup_{\substack{\|r\|_2\le b\\c\le0}}
\left\{
    \langle u,r\rangle+\lambda b+\mu c
\right\}.
\end{aligned}
\]
The supremum over $c\le0$ is finite if and only if
\[
    \mu\ge0,
\]
in which case its contribution is zero.

If $\lambda>0$, letting $b\to\infty$ makes the support function
infinite.
If $\lambda\le0$, then for fixed $r$ the maximizing choice is
$b=\|r\|_2$, so
\[
    \sup_{r\in\mathbb R^d}
    \left\{
        \langle u,r\rangle+\lambda\|r\|_2
    \right\}
\]
is finite if and only if
\[
    \|u\|_2+\lambda\le0.
\]
In that case it equals zero.
Therefore
\[
    \sigma_{\mathcal D_{\mathrm{loc}}}(u,\lambda,\mu)
    =
    \begin{cases}
    0,
    &
    \lambda\le-\|u\|_2,\quad \mu\ge0,
    \\[1mm]
    +\infty,
    &
    \text{otherwise}.
    \end{cases}
\]

Write
\[
    a:=-\lambda.
\]
If $a>0$, define
\[
    w:=\frac{u}{a},
    \qquad
    \theta:=\frac{a}{a+\mu}.
\]
Then
\[
    w\in\mathbb B_2^d,
    \qquad
    \theta\in(0,1],
\]
and
\[
\begin{aligned}
&
\frac1{a+\mu}
\left\langle
    (u,\lambda,\mu),
    \bigl(
        \blk_H(\pi,s'),
        b_H^{\mathrm{loc}}(s,s'),
        V_H(\pi,s')
    \bigr)
\right\rangle
\\
&\qquad=
\theta
\left[
    \langle w,\blk_H(\pi,s')\rangle
    -
    b_H^{\mathrm{loc}}(s,s')
\right]
+
(1-\theta)V_H(\pi,s').
\end{aligned}
\]
The case $a=0$ forces $u=0$ and gives the pure feasibility direction
$\theta=0$.

Conversely, every
\[
    \theta\in[0,1],
    \qquad
    w\in\mathbb B_2^d
\]
corresponds to the finite support direction
\[
    (u,\lambda,\mu)
    =
    (\theta w,-\theta,1-\theta).
\]
Since $\mathcal D_{\mathrm{loc}}$ is a cone, its support function is zero
in every such direction.
Theorem~\ref{thm:sliding-convex-approachability} therefore gives
precisely \eqref{eq:sliding-local-explicit}.
\end{proof}

\subsubsection{A best fixed comparator}
\label{app:sliding-soft-fixed}

We finally consider the best fixed comparator.
The essential difference from the preceding local benchmark is that
both the learner's residual and the best-fixed benchmark depend on
cumulative payoffs.
To obtain a time-homogeneous dynamic game, we therefore include these
cumulative quantities in the information state through their directions
and magnitudes.

Let $Q$ be a compact metric comparator space and $s_0\in \WinSet$ be any admissible initial sliding state and let $s_0,s_1,\ldots$ denote consecutive sliding states, with
\[
    s_k\in\WinSet(s_{k-1}),
    \qquad k\ge1.
\]
For a sequence of completed $H$-round transitions, define the
learner's cumulative block payoff
\[
    S_k
    :=
    \sum_{j=1}^k
    \blk_H(\pi_j,s_j),
    \qquad
    S_0:=0,
\]
and, for each $q\in Q$, the cumulative comparator payoff
\[
    L_k(q)
    :=
    \sum_{j=1}^k
    \blk_H(q,s_j),
    \qquad
    L_0(q):=0.
\]

Write
\[
    R_k:=\|S_k\|_2
\]
and
\[
    u_k
    :=
    \begin{cases}
    \dfrac{S_k}{\|S_k\|_2},
    &
    S_k\neq0,
    \\[2mm]
    0,
    &
    S_k=0.
    \end{cases}
\]
Similarly, for every $q\in Q$, write
\[
    r_k(q):=\|L_k(q)\|_2
\]
and
\[
    v_k(q)
    :=
    \begin{cases}
    \dfrac{L_k(q)}{\|L_k(q)\|_2},
    &
    L_k(q)\neq0,
    \\[2mm]
    0,
    &
    L_k(q)=0.
    \end{cases}
\]
Thus
\[
    S_k=R_ku_k,
    \qquad
    L_k(q)=r_k(q)v_k(q).
\]

At the beginning of a block, define the augmented information state
\[
    \widetilde s
    :=
    \bigl(
        s,u,R,v(\cdot),r(\cdot)
    \bigr).
\]
Let $\widetilde\WinSet$ denote the collection of augmented states
reachable from finite sliding-window admissible histories.

The sliding component $s$ records the exogenous temporal information.
The coordinates $(u,R)$ record the direction and magnitude of the
learner's cumulative payoff, while
$(v(\cdot),r(\cdot))$ record the corresponding quantities for every
fixed comparator.

Suppose the current augmented state is
\[
    \widetilde s
    =
    \bigl(
        s,u,R,v(\cdot),r(\cdot)
    \bigr),
\]
the learner chooses $\pi\in\Pi_H$, and the adversary produces
\[
    s'\in\WinSet(s).
\]
Define
\[
    G
    :=
    \blk_H(\pi,s').
\]
The learner's cumulative payoff is updated by
\[
    S^+
    :=
    Ru+G,
    \qquad
    R^+
    :=
    \|S^+\|_2,
\]
and
\[
    u^+
    :=
    \begin{cases}
    \dfrac{S^+}{\|S^+\|_2},
    &
    S^+\neq0,
    \\[2mm]
    0,
    &
    S^+=0.
    \end{cases}
\]

For every $q\in Q$, define
\[
    L^+(q)
    :=
    r(q)v(q)+\blk_H(q,s'),
\]
\[
    r^+(q)
    :=
    \|L^+(q)\|_2,
\]
and
\[
    v^+(q)
    :=
    \begin{cases}
    \dfrac{L^+(q)}{\|L^+(q)\|_2},
    &
    L^+(q)\neq0,
    \\[2mm]
    0,
    &
    L^+(q)=0.
    \end{cases}
\]
The successor augmented state is therefore
\[
    \widetilde s^+
    =
    \widetilde s^+(\pi,s')
    :=
    \bigl(
        s',
        u^+,
        R^+,
        v^+(\cdot),
        r^+(\cdot)
    \bigr).
\]
This transition rule contains no explicit dependence on the block
index and hence defines a time-homogeneous augmented game.

Importantly, the adversary's admissible continuation is still governed
only by
\[
    s'\in\WinSet(s).
\]
The remaining coordinates are endogenous bookkeeping variables
determined by the realized history.

Define
\[
    \beta(r)
    :=
    \inf_{q\in Q}r(q).
\]
The current best-fixed relative residual is
\[
    D(\widetilde s)
    :=
    R-\beta(r).
\]
Its one-block increment is
\begin{equation}
\label{eq:sliding-fixed-relative-increment}
\begin{aligned}
\Delta_H^{\mathrm{fix}}
(\pi,\widetilde s,s')
&:=
D(\widetilde s^+)-D(\widetilde s)
\\
&=
(R^+-R)
-
\bigl(
    \beta(r^+)-\beta(r)
\bigr).
\end{aligned}
\end{equation}
Starting from zero cumulative payoffs, these increments telescope
exactly:
\begin{equation}
\label{eq:sliding-fixed-exact-telescope}
\begin{aligned}
\sum_{k=1}^K
\Delta_H^{\mathrm{fix}}
(\pi_k,\widetilde s_{k-1},s_k)
&=
R_K-\beta(r_K)
\\
&=
\left\|
    \sum_{k=1}^K
    \blk_H(\pi_k,s_k)
\right\|_2
-
\inf_{q\in Q}
\left\|
    \sum_{k=1}^K
    \blk_H(q,s_k)
\right\|_2.
\end{aligned}
\end{equation}
Here
$\widetilde s_k=(s_k,u_k,R_k,v_k(\cdot),r_k(\cdot))$
denotes the augmented state after block $k$.
For the zero-initialized trajectory above, set
\[
\widetilde s_0
:=
\bigl(s_0,u_0,R_0,v_0(\cdot),r_0(\cdot)\bigr)\] with $
u_0=R_0=
v_0(q)=r_0(q)=0
$ for every $q\in Q$.
In particular, $D(\widetilde s_0)=0$.
To incorporate soft constraint violation, define the two-dimensional
augmented block payoff
\[
    \widehat G_H^{\mathrm{fix}}
    (\pi,\widetilde s,s')
    :=
    \left(
        \Delta_H^{\mathrm{fix}}
        (\pi,\widetilde s,s'),
        V_H(\pi,s')
    \right)
    \in\mathbb R^2.
\]
The relevant target is the negative orthant
\[
    \mathbb R_-^2
    :=
    (-\infty,0]\times(-\infty,0].
\]
Since the second coordinate is nonnegative, approachability of
$\mathbb R_-^2$ is equivalent to
\[
    [D(\widetilde s_K)]_+=o(KH)
    \qquad\text{and}\qquad
    \sum_{k=1}^K V_H(\pi_k,s_k)=o(KH).
\]

Every nonzero finite support direction of $\mathbb R_-^2$ is a positive
multiple of
\[
    (\theta,1-\theta),
    \qquad
    \theta\in[0,1].
\]
Accordingly, define the scalarized block payoff
\[
\begin{aligned}
\gamma_{\theta,H}^{\mathrm{fix}}
(\pi,\widetilde s,s')
:=
&
\theta
\Delta_H^{\mathrm{fix}}
(\pi,\widetilde s,s')
\\
&+
(1-\theta)V_H(\pi,s').
\end{aligned}
\]
The corresponding operator is
\begin{equation}
\label{eq:sliding-fixed-operator}
\begin{aligned}
&
(\widetilde{\mathcal T}^{\mathrm{fix}}_{\theta,H}
\Phi)(\widetilde s)
\\
&\quad:=
\inf_{\pi\in\Pi_H}
\sup_{s'\in\WinSet(s)}
\left\{
    \gamma_{\theta,H}^{\mathrm{fix}}
    (\pi,\widetilde s,s')
    +
    \Phi\bigl(
        \widetilde s^+(\pi,s')
    \bigr)
\right\}.
\end{aligned}
\end{equation}

Define 
\begin{equation}
\label{eq:sliding-fixed-lower-value}
\underline\lambda_H^{\mathrm{fix}}(\theta)
:=
\sup_{\Phi\in\ell^\infty(\widetilde\WinSet)}
\inf_{\widetilde s\in\widetilde\WinSet}
\left\{
    (\widetilde{\mathcal T}^{\mathrm{fix}}_{\theta,H}
    \Phi)(\widetilde s)
    -
    \Phi(\widetilde s)
\right\},
\end{equation}
and
\begin{equation}
\label{eq:sliding-fixed-upper-value}
\overline\lambda_H^{\mathrm{fix}}(\theta)
:=
\inf_{\Phi\in\ell^\infty(\widetilde\WinSet)}
\sup_{\widetilde s\in\widetilde\WinSet}
\left\{
    (\widetilde{\mathcal T}^{\mathrm{fix}}_{\theta,H}
    \Phi)(\widetilde s)
    -
    \Phi(\widetilde s)
\right\}.
\end{equation}

\paragraph{A uniform augmented-state characterization.}

For the augmented game, we say that uniform best-fixed approachability
holds if there exist a causal learner strategy and a deterministic
$r(K)=o(KH)$ such that, for every initial augmented state
$\widetilde s_0\in\widetilde\WinSet$ and every admissible adversary,
\[
    \operatorname{dist}
    \left(
        \sum_{k=1}^K
        \widehat G_H^{\mathrm{fix}}
        (\pi_k,\widetilde s_{k-1},s_k),
        \mathbb R_-^2
    \right)
    \le r(K).
\]

\begin{theorem}[Uniform characterization for the augmented best-fixed game]
\label{thm:sliding-fixed-uniform}
Uniform best-fixed approachability holds if and only if
\begin{equation}
\label{eq:sliding-fixed-uniform-characterization}
    \overline\lambda_H^{\mathrm{fix}}(\theta)
    \le0,
    \qquad
    \forall\theta\in[0,1],
\end{equation}
where
\[
\overline\lambda_H^{\mathrm{fix}}(\theta)
=
\inf_{\Phi\in\ell^\infty(\widetilde\WinSet)}
\sup_{\widetilde s\in\widetilde\WinSet}
\left\{
    (\widetilde{\mathcal T}^{\mathrm{fix}}_{\theta,H}\Phi)
    (\widetilde s)
    -
    \Phi(\widetilde s)
\right\}.
\]
No bridging assumption is imposed on the augmented state space.
\end{theorem}

\begin{proof}
This is Theorem~\ref{thm:sliding-convex-approachability} applied to the
time-homogeneous augmented payoff
\[
    \widehat G_H^{\mathrm{fix}}
    =
    \left(
        \Delta_H^{\mathrm{fix}},
        V_H
    \right)
\]
and the closed convex target $\mathbb R_-^2$.
The nonzero finite support directions of $\mathbb R_-^2$ are, up to
positive scaling,
\[
    (\theta,1-\theta),
    \qquad
    \theta\in[0,1],
\]
which gives exactly
\eqref{eq:sliding-fixed-uniform-characterization}.
\end{proof}

\begin{corollary}[Best-fixed comparator from zero initialization]
\label{cor:sliding-fixed-zero-initialization}
If \eqref{eq:sliding-fixed-uniform-characterization} holds, then the
learner can guarantee
\[
    \left[
        \|\cum_T\|_2
        -
        \inf_{q\in Q}
        \left\|
            \sum_{t=1}^T
            \poff(q,z_t)
        \right\|_2
    \right]_+
    =
    o(T)
\]
and
\[
    V_T=o(T)
\]
uniformly over all admissible initial sliding states and adversaries.
\end{corollary}

\begin{proof}
Start the augmented game from
\[
    S_0=0,
    \qquad
    L_0(q)=0,
    \quad q\in Q.
\]
Then $D(\widetilde s_0)=0$, and
\eqref{eq:sliding-fixed-exact-telescope} gives
\[
\sum_{k=1}^K
\widehat G_H^{\mathrm{fix}}
=
\left(
    \left\|
        \sum_{k=1}^K\blk_H(\pi_k,s_k)
    \right\|_2
    -
    \inf_{q\in Q}
    \left\|
        \sum_{k=1}^K\blk_H(q,s_k)
    \right\|_2,
    \;
    \sum_{k=1}^KV_H(\pi_k,s_k)
\right).
\]
Approachability of $\mathbb R_-^2$ therefore gives the claimed
guarantees. The final incomplete block contributes only $O(H)$.
\end{proof}
\begin{remark}
\label{rem:sliding-fixed-rooted}
The preceding theorem characterizes a stronger, uniformly
augmented-state formulation.
The original best-fixed-comparator objective starts only from the
distinguished initialization
\[
    S_0=0,
    \qquad
    L_0(\cdot)=0.
\]

Accordingly,
Corollary~\ref{cor:sliding-fixed-zero-initialization} is stated only as
a sufficient consequence.
\end{remark}
\begin{remark}[The blockwise case as a special case]
\label{rem:split-fixed-as-special-case}
The blockwise formulation is recovered by removing the continuation
restriction
\[
    s'\in\WinSet(s).
\]
The sliding-state coordinate then becomes unnecessary, while the
cumulative learner and comparator coordinates evolve exactly as above.
\end{remark}

\begin{remark}[No bridging on the augmented state]
\label{rem:sliding-fixed-no-augmented-bridging}
Bridging is imposed, when needed, only on the original sliding state
$s$, which governs admissible future continuations.
The additional augmented coordinates record cumulative learner and
comparator payoffs and need not be mutually reachable in a uniformly
bounded number of blocks.
\end{remark}

\subsection{A negative example under pointwise compatibility}
\label{app:negative-pointwise-feasible}
We conclude the appendix with a binary-forecasting counterexample
showing that, even when every round admits a feasible action
with zero payoff, no learner can guarantee both comparator-relative
GEQ and sublinear cumulative constraint violation against all
admissible adversaries.

Recall that in the binary probabilistic forecasting example,
\[
\mathcal X=[0,1],
\]
and the adversarial outcome at round $t$ is
\[
z_t=(y_t,\alpha_t),
\qquad
y_t\in\{0,1\},
\quad
\alpha_t\in[0,1].
\]
The vector payoff is given by
\[
\poff(p_t,z_t)
:=
y_t-p_t.
\]
Hence
\[
\cum_T
=
\sum_{t=1}^{T}(y_t-p_t).
\]
The learner's scalar constraint function is
\[
\cstr(p_t,z_t)
:=
p_t-\alpha_t
\] so that
\[
V_T
=
\sum_{t=1}^{T}
[\cstr(p_t,z_t)]_+
=
\sum_{t=1}^{T}
(p_t-\alpha_t)_+,
\]
and the adversary is required to satisfy the pointwise compatibility
condition
\begin{equation}
\label{eq:app-pointwise-compatibility}
y_t\le\alpha_t.
\end{equation}

Thus, for every admissible outcome $z_t$, the action $q_t=y_t$
satisfies both $\cstr(q_t,z_t)\le0$ and $\poff(q_t,z_t)=0$.

Define
\[
\underline\alpha_T
:=
\min_{t\le T}\alpha_t.
\]
The best fixed feasible-comparator benchmark is
\[
B_T
:=
\inf_{0\le x\le\underline\alpha_T}
\left|
\sum_{t=1}^{T}
\poff(x,z_t)
\right|
=
\inf_{0\le x\le\underline\alpha_T}
\left|
\sum_{t=1}^{T}(y_t-x)
\right|.
\]

We consider the comparator-relative GEQ objective
\begin{equation}
\label{eq:app-relative-geq}
\bigl[\lvert\cum_T\rvert-B_T\bigr]_+
=
o(T).
\end{equation}

\begin{proposition}
\label{prop:adversarial-constraint-impossibility}
For every learner strategy, there exists a causal adversary satisfying
\eqref{eq:app-pointwise-compatibility} such that the learner cannot
simultaneously guarantee
\[
V_T=o(T)
\]
and
\[
\bigl[\lvert\cum_T\rvert-B_T\bigr]_+=o(T).
\]
\end{proposition}

\begin{proof}
Fix
\[
\bar\alpha\in(1/2,1)
\]
and define
\begin{equation}
\label{eq:app-threshold}
q
:=
\frac{1+\bar\alpha}{2},
\qquad
\delta
:=
q-\bar\alpha
=
\frac{1-\bar\alpha}{2}
>0.
\end{equation}

Set
\[
a_1=1,
\qquad
a_{r+1}=a_r+2a_r^2.
\]
Initialize the process at $t=1$ by
\[
(\alpha_1,y_1)=(\bar\alpha,0).
\]
The construction below only uses
\[
\alpha_t\in\{\bar\alpha,1\},
\]
and therefore
\[
\underline\alpha_T=\bar\alpha,
\qquad
\forall T.
\]

Epoch $r$ consists of
\[
\{a_r+1,\ldots,a_r+2a_r^2\}
\]
and is divided into two phases of equal length $a_r^2$.

During the first phase
\[
I_r^{(1)}
:=
\{a_r+1,\ldots,a_r+a_r^2\},
\]
the adversary chooses
\begin{equation}
\label{eq:app-first-phase}
(\alpha_t,y_t)
=
(\bar\alpha,0).
\end{equation}

During the second phase
\[
I_r^{(2)}
:=
\{a_r+a_r^2+1,\ldots,a_r+2a_r^2\},
\]
the adversary chooses
\[
(\alpha_t,y_t)
=
\begin{cases}
(1,1),
&
p_t<q,
\\[1mm]
(\bar\alpha,0),
&
p_t\ge q.
\end{cases}
\]
The pointwise compatibility condition
\eqref{eq:app-pointwise-compatibility}
is satisfied in both cases.

Define
\[
c_r
:=
\frac1{a_r^2}
\sum_{t\in I_r^{(1)}}p_t
\]
and
\[
d_r
:=
\frac1{a_r^2}
\sum_{t\in I_r^{(2)}}
\mathbf 1\{p_t\ge q\}.
\]

\paragraph{Cumulative feasibility implies $d_r\to0$.}
For every $t\in I_r^{(2)}$ such that $p_t\ge q$, the adversary chooses
$\alpha_t=\bar\alpha$. Hence, by \eqref{eq:app-threshold},
\[
[\cstr(p_t,z_t)]_+
=
(p_t-\alpha_t)_+
\ge
q-\bar\alpha
=
\delta.
\]
Therefore
\begin{equation}
\label{eq:app-ccv-lower}
V_{a_{r+1}}
\ge
\delta a_r^2d_r.
\end{equation}
Since
\[
a_{r+1}
=
a_r+2a_r^2
\asymp
2a_r^2,
\]
the assumption $V_T=o(T)$ together with
\eqref{eq:app-ccv-lower} implies
\begin{equation}
\label{eq:app-ds-zero}
d_r\longrightarrow0.
\end{equation}

\paragraph{Comparator-relative GEQ implies $c_r\to0$.}
Consider the end of the first phase,
\[
T_r^{(1)}
:=
a_r+a_r^2.
\]
Since $y_t=0$ throughout $I_r^{(1)}$ and only $a_r$ rounds precede the
current epoch,
\[
\sum_{t=1}^{T_r^{(1)}}y_t
\le
a_r.
\]
Consequently,
\[
\frac1{T_r^{(1)}}
\sum_{t=1}^{T_r^{(1)}}y_t
\le
\frac{a_r}{a_r+a_r^2}
=
\frac1{1+a_r}.
\]
For all sufficiently large $r$, this quantity is smaller than
$\bar\alpha$. Thus
\[
x_r
:=
\frac1{T_r^{(1)}}
\sum_{t=1}^{T_r^{(1)}}y_t
\]
is a feasible fixed comparator, and therefore
\begin{equation}
\label{eq:app-benchmark-first-zero}
B_{T_r^{(1)}}=0.
\end{equation}

By \eqref{eq:app-relative-geq},
\eqref{eq:app-benchmark-first-zero} implies
\begin{equation}
\label{eq:app-first-residual-small}
|\cum_{T_r^{(1)}}|
=
o(T_r^{(1)})
=
o(a_r^2).
\end{equation}

Since every summand of $\cum_{a_r}$ lies in $[-1,1]$,
\begin{equation}
\label{eq:app-past-bound}
|\cum_{a_r}|
\le
a_r.
\end{equation}
By \eqref{eq:app-first-phase},
\[
\cum_{T_r^{(1)}}
=
\cum_{a_r}
-
\sum_{t\in I_r^{(1)}}p_t
=
\cum_{a_r}
-
a_r^2c_r.
\]
Combining this identity with
\eqref{eq:app-first-residual-small} and
\eqref{eq:app-past-bound}, and using
$a_r=o(a_r^2)$, yields
\begin{equation}
\label{eq:app-cs-zero}
c_r\longrightarrow0.
\end{equation}

\paragraph{The second phase creates a positive payoff drift.}
Consider $t\in I_r^{(2)}$.
If $p_t<q$, then the adversary chooses $y_t=1$, and hence
\[
\poff(p_t,z_t)
=
1-p_t
>
1-q
=
\delta.
\]
If $p_t\ge q$, then the adversary chooses $y_t=0$, and therefore
\[
\poff(p_t,z_t)
=
-p_t
\ge
-1.
\]
Hence
\[
\sum_{t\in I_r^{(2)}}
\poff(p_t,z_t)
\ge
a_r^2
\left[
(1-d_r)\delta-d_r
\right].
\]

During the first phase,
\[
\sum_{t\in I_r^{(1)}}
\poff(p_t,z_t)
=
-a_r^2c_r.
\]
Therefore the total vector payoff accumulated during epoch $r$
satisfies
\[
\frac1{2a_r^2}
\sum_{t=a_r+1}^{a_{r+1}}
\poff(p_t,z_t)
\ge
\frac12
\left[
-c_r+(1-d_r)\delta-d_r
\right].
\]
Using
\eqref{eq:app-ds-zero} and
\eqref{eq:app-cs-zero},
we obtain
\begin{equation}
\label{eq:app-positive-drift}
\liminf_{r\to\infty}
\frac1{2a_r^2}
\sum_{t=a_r+1}^{a_{r+1}}
\poff(p_t,z_t)
\ge
\frac{\delta}{2}
=
\frac{1-\bar\alpha}{4}
>0.
\end{equation}

\paragraph{The comparator benchmark is zero at the end of the epoch.}
All positive outcomes during epoch $r$ occur in the second phase, so
there can be at most $a_r^2$ such outcomes. Since the preceding history
contains $a_r$ rounds,
\[
\sum_{t=1}^{a_{r+1}}y_t
\le
a_r+a_r^2.
\]
Therefore
\[
\frac1{a_{r+1}}
\sum_{t=1}^{a_{r+1}}y_t
\le
\frac{a_r+a_r^2}{a_r+2a_r^2}
\longrightarrow
\frac12.
\]
Since $\bar\alpha>1/2$, for all sufficiently large $r$,
\[
\frac1{a_{r+1}}
\sum_{t=1}^{a_{r+1}}y_t
\le
\bar\alpha
=
\underline\alpha_{a_{r+1}}.
\]
Thus the empirical outcome mean is a feasible fixed comparator, and
\begin{equation}
\label{eq:app-benchmark-end-zero}
B_{a_{r+1}}=0
\end{equation}
for all sufficiently large $r$.

On the other hand, the payoff accumulated before epoch $r$ is bounded
below by $-a_r$, whereas \eqref{eq:app-positive-drift} gives a positive
contribution of order $a_r^2$ during epoch $r$. Hence there exists a
constant $c_{\bar\alpha}>0$ such that, for all sufficiently large $r$,
\[
\cum_{a_{r+1}}
\ge
c_{\bar\alpha}a_r^2.
\]
Since
\[
a_{r+1}
=
a_r+2a_r^2
\asymp
a_r^2,
\]
we conclude that
\[
|\cum_{a_{r+1}}|
=
\Omega(a_{r+1}).
\]
Together with
\eqref{eq:app-benchmark-end-zero},
this contradicts
\eqref{eq:app-relative-geq}.
Therefore no learner can simultaneously guarantee sublinear cumulative
constraint violation and comparator-relative GEQ.
\end{proof}

\section{Proofs of main results}\label{app:proofs}

\subsection{Proof of Theorem~\ref{thm:lowerbound}}

\begin{proof}
Let $n\coloneqq T/2$. Denote the two possible adversarial
sequences by
\[
    z^{(0)}
    \coloneqq
    (0,\ldots,0),
    \qquad
    z^{(1)}
    \coloneqq
    (\underbrace{0,\ldots,0}_{n},
     \underbrace{1,\ldots,1}_{n}).
\]
We first condition on any internal randomness of the learner, so
that the learner may be treated as deterministic. Since
$z^{(0)}$ and $z^{(1)}$ coincide on the first $n$ rounds, causality
implies that the learner takes exactly the same actions on those
rounds under the two sequences. Let
\[
    A
    \coloneqq
    \sum_{t=1}^{n} p_t
\]
denote the learner's cumulative prediction over this common
prefix.

Let
\[
    B_0
    \coloneqq
    \sum_{t=n+1}^{T} p_t^{(0)},
    \qquad
    B_1
    \coloneqq
    \sum_{t=n+1}^{T} p_t^{(1)}
\]
be the cumulative predictions over the second half under
$z^{(0)}$ and $z^{(1)}$, respectively. By the feasibility
constraint $p_t\in[0,1/2]$,
\[
    0\le A,B_0,B_1\le \frac n2.
\]

Under the all-zero sequence $z^{(0)}$,
\[
    \sum_{t=1}^T (p_t-z_t)
    =
    A+B_0
    \ge0,
\]
and hence
\[
    \left|
        \sum_{t=1}^T (p_t-z_t)
    \right|
    =
    A+B_0.
\]

Under $z^{(1)}$, the second half contains $n$ ones, so
\[
    \sum_{t=1}^T (p_t-z_t)
    =
    A+B_1-n.
\]
Since $A+B_1\le n$, this quantity is nonpositive, and therefore
\[
    \left|
        \sum_{t=1}^T (p_t-z_t)
    \right|
    =
    n-A-B_1.
\]

Averaging over the adversary's two equiprobable choices gives
\begin{align*}
\bE\left|
    \frac1T\sum_{t=1}^T(p_t-z_t)
\right|
&=
\frac{1}{2T}
\left[
    (A+B_0)+(n-A-B_1)
\right]\\
&=
\frac{1}{2T}
\left(
    n+B_0-B_1
\right).
\end{align*}
Using $B_0\ge0$ and $B_1\le n/2$,
\[
    n+B_0-B_1
    \ge \frac n2.
\]
Consequently,
\[
\bE\left|
    \frac1T\sum_{t=1}^T(p_t-z_t)
\right|
\ge
\frac{n}{4T}
=
\frac18.
\]

Thus the claim holds with the universal constant
\(
    c=\frac18.
\)

The argument holds for every realization of the learner's
internal randomness. Taking expectation over that randomness as
well proves the result for arbitrary randomized causal learners.
\end{proof}

\subsection{Proof of Theorem~\ref{thm:split-blackwell}}
\begin{proof}
We first prove sufficiency.  For block $k$, let
$s^{(k)}\in\WinSet$ denote its adversarial outcome sequence and let
$\pi_k\in\Pi_H$ denote the learner's block policy.  Write
\[
G_k:=\blk_H(\pi_k,s^{(k)}),
\qquad
S_{kH}:=\sum_{j=1}^k G_j,
\qquad S_0:=0.
\]
For an arbitrary horizon $T\ge1$, write
\[
    T=KH+r,
    \qquad
    K=\lfloor T/H\rfloor,
    \qquad
    0\le r<H.
\]
At the beginning of each block $k=1,\ldots,K$, choose
\[
u_k
:=
\begin{cases}
\dfrac{S_{(k-1)H}}{\|S_{(k-1)H}\|_2},
& S_{(k-1)H}\neq0,\\[2mm]
\text{any }u\in\mathbb S^{d-1},
& S_{(k-1)H}=0.
\end{cases}
\]
Set
\[
    \varepsilon_k:=\frac{LH}{k\,2^{k+1}},
    \qquad k\ge1.
\]
By \eqref{eq:split-blackwell}, the learner can select
a causal block policy $\pi_k\in\Pi_H$ such that
\[
    \sup_{s\in\WinSet}
    \iprod{u_k,\blk_H(\pi_k,s)}
    \le \varepsilon_k.
\]
Consequently,
\[
    \iprod{S_{(k-1)H},G_k}
    \le \norm{S_{(k-1)H}}\varepsilon_k.
\]
Since
\[
    \norm{S_{(k-1)H}}\le(k-1)LH
    \qquad\text{and}\qquad
    \norm{G_k}\le LH,
\]
we obtain
\[
\begin{aligned}
    \norm{S_{kH}}^2
    &\le
    \norm{S_{(k-1)H}}^2
    +2\norm{S_{(k-1)H}}\varepsilon_k
    +(LH)^2\\
    &\le
    \norm{S_{(k-1)H}}^2
    +(LH)^2(1+2^{-k}).
\end{aligned}
\]
Iterating gives
\[
    \norm{S_{KH}}^2
    \le
    (LH)^2\left(K+\sum_{k=1}^K2^{-k}\right)
    \le (LH)^2(K+1).
\]

The remaining $r$ rounds contribute at most $rL$ in norm.
Therefore,
\[
\begin{aligned}
    \norm{S_T}
    &\le LH\sqrt{K+1}+rL\\
    &\le LH\left(\sqrt{K+\frac rH}+1\right)\\
    &=L\sqrt{HT}+LH.
\end{aligned}
\]
Dividing by $T$ proves the stated bound.
For necessity, suppose that \eqref{eq:split-blackwell} fails.
Then there exists $u\in\mathbb S^{d-1}$ such that
\[
\rho_H
:=
\frac12
\inf_{\pi\in\Pi_H}
\sup_{s\in\WinSet}
\langle u,\blk_H(\pi,s)\rangle
>0.
\]
After observing the learner's policy $\pi_k$ for each block, the
adversary can choose an admissible block sequence $s^{(k)}$ such that
\[
\langle u,G_k\rangle\ge\rho_H.
\]
Consequently,
\[
\left\langle u,S_{KH}\right\rangle
=
\left\langle u,\sum_{k=1}^K G_k\right\rangle
\ge K\rho_H.
\]
It follows that
\[
\left\|\sum_{k=1}^K G_k\right\|_2
\ge K\rho_H,
\]
and hence the average payoff cannot converge to zero.
\end{proof}

\subsection{Proof of Proposition~\ref{prop:sufficient-but-unnecessary-conditions}}
\label{app:inward-flow-blackwell}

In this appendix, we prove
Proposition~\ref{prop:sufficient-but-unnecessary-conditions}.
We consider two structural conditions used in prior work:
restorativity and inward flow. Both imply Blackwell feasibility,
although in different settings. Restorativity is naturally an
unconstrained radial condition, whereas inward flow is defined
relative to the geometry of a constrained action set.

We also show that neither condition is necessary for
approachability.

\subsubsection{Restorativity implies Blackwell feasibility}

We first consider the unconstrained setting $C=\mathbb R^d$.
Suppose the payoff family $g:\mathbb R^d\times\mathcal Z
\to\mathbb R^d$ is $(h,\phi)$-restorative for some $h\ge0$ and some $\phi:\R^d\to[0,\infty)$, i.e.,
\[
    \langle x,-g(x,z)\rangle
    \le -\phi(x),
    \qquad
    \forall z\in\mathcal Z,\quad
    \forall x:\|x\|_2\ge h.
\]

\begin{proposition}[Restorativity implies the Blackwell condition]
\label{prop:restorative-blackwell}
Suppose $C=\mathbb R^d$ and the payoff family is
$(h,\phi)$-restorative for some $\phi\ge0$.
Then, for every $u\in\mathbb S^{d-1}$,
\[
    \inf_{x\in\mathbb R^d}
    \sup_{z\in\mathcal Z}
    \langle u,g(x,z)\rangle
    \le 0.
\]
Consequently, the blockwise Blackwell condition
\eqref{eq:split-blackwell} holds for every $H$.
\end{proposition}

\begin{proof}
Fix $u\in\mathbb S^{d-1}$ and choose any $r\ge h+1$.
Set
\[
    x_u=-ru.
\]
By restorativity, for every $z\in\mathcal Z$,
\[
    \langle -ru,-g(-ru,z)\rangle
    \le -\phi(-ru).
\]
Therefore
\[
    \langle u,g(-ru,z)\rangle
    \le -\frac{\phi(-ru)}{r}
    \le 0.
\]
Hence
\[
    \inf_x\sup_z
    \langle u,g(x,z)\rangle
    \le0.
\]

For the $H$-round game, let the learner use the constant
causal policy $x_t=x_u$ throughout the block. Then, for every
admissible window $s=(z_1,\dots,z_H)\in\mathcal W_H$,
\[
    \left\langle
        u,G_H(x_u,s)
    \right\rangle
    =
    \sum_{t=1}^H
    \langle u,g(x_u,z_t)\rangle\le0.
\]
Thus
\[
    \inf_{\pi\in\Pi_H}
    \sup_{s\in\mathcal W_H}
    \left\langle
        u,G_H(\pi,s)
    \right\rangle
    \le0
\]
for every $u\in\mathbb S^{d-1}$, proving
\eqref{eq:split-blackwell}.
\end{proof}

The proof also applies to a constrained set $C$ whenever, for
each direction $u$, the set $C$ contains some radial point
$-ru$ with $r>0$ and $r\ge h$. In general, however, restorativity alone
does not imply constrained Blackwell feasibility for an
arbitrary $C$.

\subsubsection{Inward flow implies Blackwell feasibility}

We next turn to a fixed constrained action set
$C\subseteq\mathbb R^d$. Let
\[
N_C(x)
=
\left\{
n\in\mathbb R^d:
\langle n,y-x\rangle\le0,\ \forall y\in C
\right\}
\]
denote the normal cone, and define the tangent cone by
\[
T_C(x)
=
\left\{
v\in\mathbb R^d:
\langle v,n\rangle\le0,\ \forall n\in N_C(x)
\right\}.
\]
The inward-flow condition requires
\[
    -g(x,z)\in T_C(x),
    \qquad
    \forall x\in\partial C,\ \forall z\in\mathcal Z.
\]

\begin{proposition}[Inward flow implies the Blackwell condition]
\label{prop:inward-blackwell}
Let $C\subseteq\mathbb R^d$ be nonempty, compact, and convex,
and suppose
\[
    \|g(x,z)\|_2\le L,
    \qquad
    \forall x\in C,\ z\in\mathcal Z.
\]
If the inward-flow condition holds, then
\[
    \inf_{x\in C}
    \sup_{z\in\mathcal Z}
    \langle u,g(x,z)\rangle
    \le0,
    \qquad
    \forall u\in\mathbb S^{d-1}.
\]
Consequently, the blockwise Blackwell condition
\eqref{eq:split-blackwell} holds for every $H$.
\end{proposition}

\begin{proof}
Fix $u\in\mathbb S^{d-1}$. For $K>0$, let
\[
    x_K=\Pi_C(-Ku).
\]
Since $C$ is compact, for all sufficiently large $K$,
$x_K\in\partial C$. By the projection optimality condition,
\[
    -Ku-x_K\in N_C(x_K).
\]
On the other hand, inward flow gives
\[
    -g(x_K,z)\in T_C(x_K),
    \qquad \forall z\in\mathcal Z.
\]
By polarity of the tangent and normal cones,
\[
    \langle -Ku-x_K,-g(x_K,z)\rangle\le0.
\]
Rearranging,
\[
    \langle u,g(x_K,z)\rangle
    \le
    -\frac{1}{K}
    \langle x_K,g(x_K,z)\rangle
    \le
    \frac{R_CL}{K},
\]
where
\[
    R_C\coloneqq\sup_{x\in C}\|x\|_2<\infty.
\]
Thus
\[
    \inf_{x\in C}\sup_{z\in\mathcal Z}
    \langle u,g(x,z)\rangle
    \le
    \frac{R_CL}{K}.
\]
Letting $K\to\infty$ proves the one-step Blackwell condition.

For an $H$-round block, use the resulting action for the fixed
direction $u$ on every round. Summing the one-step inequalities
gives the blockwise Blackwell condition.
\end{proof}

The two sufficient conditions have different geometric content.
Restorativity controls the component of $-g(x,z)$ in the radial
direction $x$, while inward flow controls its component along
every outward normal to $C$. Hence neither condition generally
implies the other.

\subsubsection{Restorativity is not necessary}

We next show that exact GEQ may be trivial even though no finite
restorativity radius exists.

\begin{proposition}[Restorativity is not necessary]
\label{prop:restorativity-not-necessary}
There exists an unconstrained problem with smooth convex losses
and uniformly bounded gradients for which exact GEQ is
achievable, but the payoff family is not $(h,0)$-restorative
for any finite $h$.
\end{proposition}

\begin{proof}
Let $C=\mathbb R^2$, let $z\in[1,2]$, and define
\[
    \ell(x,z)
    =
    z\sqrt{1+(x_1-1)^2}.
\]
Then
\[
    g(x,z)
    =
    \nabla_x\ell(x,z)
    =
    z\frac{x_1-1}
    {\sqrt{1+(x_1-1)^2}}e_1,
\]
so
\[
    \|g(x,z)\|_2\le2
\]
uniformly over $x$ and $z$. Each loss is smooth and convex.

Let
\[
    x^\star=(1,0).
\]
Then
\[
    g(x^\star,z)=0,
    \qquad \forall z\in[1,2].
\]
Therefore the constant strategy $x_t=x^\star$ achieves
\[
    \sum_{t=1}^T g(x_t,z_t)=0
\]
for every horizon and every adversarial sequence. Hence exact
GEQ, and in particular Blackwell approachability, holds.

We now show that no finite restorativity radius exists.
For any $R>0$, let
\[
    x_R=\left(\frac12,R\right).
\]
Then
\[
    g(x_R,z)
    =
    -\frac{z}{\sqrt5}e_1,
\]
and hence
\[
    \langle x_R,g(x_R,z)\rangle
    =
    -\frac{z}{2\sqrt5}<0.
\]
Since $\|x_R\|_2\to\infty$ as $R\to\infty$, for every finite
$h$ there exists $R$ with $\|x_R\|_2\ge h$ for which
\[
    \langle x_R,-g(x_R,z)\rangle>0.
\]
Thus even $(h,0)$-restorativity fails for every finite $h$.
A fortiori, $(h,\phi)$-restorativity cannot hold for any
$\phi\ge0$.
\end{proof}

This example illustrates why restorativity is stronger than
what GEQ itself requires. Restorativity imposes a global radial
condition on the vector field at all sufficiently large
locations, whereas GEQ may be achievable by remaining at a
single common stationary point.

\subsubsection{Inward flow is not necessary}

Finally, we show that inward flow is likewise not necessary,
even for uniformly strongly convex losses.

\begin{proposition}[Inward flow is not necessary]
\label{prop:inward-not-necessary}
There exist a compact convex constraint set and a family of
uniformly strongly convex losses for which exact CGEQ is
achievable, while the inward-flow condition fails.
\end{proposition}

\begin{proof}
Let
\[
    C=\{x\in\mathbb R^2:\|x\|_2\le1\},
    \qquad
    x^\star=
    \begin{pmatrix}
        0\\[1mm]
        3/4
    \end{pmatrix},
    \qquad
    a=
    \begin{pmatrix}
        1\\
        2
    \end{pmatrix}.
\]
For $z\in[1,2]$, define
\[
    \ell(x,z)
    =
    \frac{z}{2}
    \bigl(a^\top(x-x^\star)\bigr)^2
    +
    \frac18\|x-x^\star\|_2^2.
\]
Its gradient is
\[
    g(x,z)
    =
    zaa^\top(x-x^\star)
    +
    \frac14(x-x^\star),
\]
and
\[
    \nabla_x^2\ell(x,z)
    =
    zaa^\top+\frac14 I
    \succeq\frac14 I.
\]
Thus every loss is uniformly strongly convex.

Since $x^\star\in\operatorname{int}(C)$ and
\[
    g(x^\star,z)=0,
    \qquad \forall z\in[1,2],
\]
the learner can play $x_t=x^\star$ on every round, yielding
exact CGEQ for every adversarial sequence.

However, inward flow fails. Consider
\[
    \bar x=e_1=(1,0)\in\partial C.
\]
For the unit ball,
\[
    T_C(\bar x)
    =
    \{d\in\mathbb R^2:d_1\le0\}.
\]
Moreover,
\[
    \bar x-x^\star
    =
    \begin{pmatrix}
        1\\
        -3/4
    \end{pmatrix},
    \qquad
    a^\top(\bar x-x^\star)
    =
    -\frac12.
\]
Therefore
\[
    g(\bar x,z)
    =
    -\frac{z}{2}
    \begin{pmatrix}
        1\\
        2
    \end{pmatrix}
    +
    \frac14
    \begin{pmatrix}
        1\\
        -3/4
    \end{pmatrix}.
\]
In particular,
\[
    [-g(\bar x,z)]_1
    =
    \frac z2-\frac14
    >0,
    \qquad
    \forall z\in[1,2].
\]
Hence
\[
    -g(\bar x,z)\notin T_C(\bar x),
\]
so the inward-flow condition fails even though exact CGEQ is
achievable.
\end{proof}

Combining
Propositions~\ref{prop:restorative-blackwell} and
\ref{prop:inward-blackwell} with
Propositions~\ref{prop:restorativity-not-necessary} and
\ref{prop:inward-not-necessary}
proves
Proposition~\ref{prop:sufficient-but-unnecessary-conditions}.

\subsection{Proof of Theorem~\ref{thm:slide-blackwell}}

We prove the three claims separately. We first introduce the
finite-horizon directional values used in the proofs.

Fix a direction $u\in\mathbb S^{d-1}$. Define the directional quantities
\begin{align*}
\underline\lambda_H(u)
&\coloneqq
\sup_{\Potential\in\ell^\infty(\WinSet)}
\inf_{s\in\WinSet}
\left\{
    (\mathcal T_{u,H}\Potential)(s)-\Potential(s)
\right\},
\\
\overline\lambda_H(u)
&\coloneqq
\inf_{\Potential\in\ell^\infty(\WinSet)}
\sup_{s\in\WinSet}
\left\{
    (\mathcal T_{u,H}\Potential)(s)-\Potential(s)
\right\}.
\end{align*}
Also fix an initial sliding
state $s\in\WinSet$. Set $s_0=s$. For $n\ge1$, define the set of
admissible $n$-block successor trajectories from $s$ by
\[
    \mathcal J_n(s)
    \coloneqq
    \left\{
        (s_1,\ldots,s_n)\in\WinSet^n:
        s_k\in\WinSet(s_{k-1}),
        \quad k=1,\ldots,n
    \right\}.
\]
Thus, each $s_k$ is an admissible $H$-round successor block of
$s_{k-1}$, and the concatenation
\[
    s_0\oplus s_1\oplus\cdots\oplus s_n
\]
satisfies the sliding-window restriction throughout.

Let $\Pi_{n,H}$ denote the class of causal learner policies over
$n$ consecutive $H$-round blocks. A policy
\[
    \pi=(\pi_1,\ldots,\pi_n)\in\Pi_{n,H}
\]
may condition $\pi_k$ on the history revealed before block $k$,
while within block $k$ it chooses each round-level action
causally based only on observations revealed earlier in that
block.

The $n$-block directional value starting from state $s$ is
\begin{equation}
\label{eq:finite-horizon-directional-value}
    V_n^u(s)
    \coloneqq
    \inf_{\pi\in\Pi_{n,H}}
    \sup_{(s_1,\ldots,s_n)\in\mathcal J_n(s)}
    \sum_{k=1}^n
    \left\langle
        u,\blk_H(\pi_k,s_k)
    \right\rangle,
\end{equation}
with
\[
    V_0^u(s)\coloneqq0.
\]
Equivalently, $V_n^u(s)$ is the minimax cumulative payoff in
direction $u$ over the next $n$ blocks, conditional on the current
sliding state being $s$.

These values satisfy the dynamic-programming recursion
\begin{equation}
\label{eq:directional-value-recursion}
    V_{n+1}^u
    =
    \mathcal T_{u,H}V_n^u,
    \qquad
    V_0^u=0,
\end{equation}
and therefore
\[
    V_n^u
    =
    \mathcal T_{u,H}^n 0.
\]
Lemma~\ref{lem:supinf-relation} relates the asymptotic lower and
upper growth rates of $V_n^u$ to
$\underline\lambda_H(u)$ and $\overline\lambda_H(u)$,
respectively.

\begin{proposition}[Necessary condition for pointwise approachability]
\label{thm:sliding-necessary}
If pointwise approachability holds against every
$H$-sliding-window adversary, then
\[
    \underline\lambda_H\le0.
\]
Moreover,
\[
    \underline\lambda_H(u)
    =
    \sup_{n\ge1}
    \inf_{s\in\WinSet}
    \frac{V_n^u(s)}{n}.
\]
\end{proposition}

\begin{proof}
Fix $u\in\mathbb S^{d-1}$. Suppose, toward a contradiction, that
\[
    \underline\lambda_H(u)>0.
\]
Then there exist
$\Potential\in\ell^\infty(\WinSet)$ and $\eta>0$ such that
\[
    (\mathcal T_{u,H}\Potential)(s)
    \ge
    \Potential(s)+\eta,
    \qquad
    \forall s\in\WinSet.
\]
Thus, for every current state $s$ and every causal block policy
$\pi\in\Pi_H$,
\[
    \sup_{s'\in\WinSet(s)}
    \left\{
        \left\langle u,\blk_H(\pi,s')\right\rangle
        +\Potential(s')-\Potential(s)
    \right\}
    \ge\eta.
\]
Hence the adversary can choose an admissible successor
$s'\in\WinSet(s)$ such that
\begin{equation}
\label{eq:pointwise-positive-drift}
    \left\langle u,\blk_H(\pi,s')\right\rangle
    +\Potential(s')-\Potential(s)
    \ge\frac{\eta}{2}.
\end{equation}
The factor $1/2$ only accounts for the possibility that the
supremum is not attained.

Fix any admissible initial state $s_0$. Against an arbitrary
causal learner strategy, construct the adversary recursively.
At block $k$, the learner's strategy and the public history induce
a causal $H$-round policy $\pi_k$. Choose
\[
    s_k\in\WinSet(s_{k-1})
\]
according to \eqref{eq:pointwise-positive-drift}. Then
\[
    \left\langle u,\blk_H(\pi_k,s_k)\right\rangle
    \ge
    \frac{\eta}{2}
    +\Potential(s_{k-1})-\Potential(s_k).
\]
Summing over $k=1,\ldots,K$ gives
\[
\begin{aligned}
    \langle u,\cum_{KH}\rangle
    &=
    \sum_{k=1}^K
    \left\langle u,\blk_H(\pi_k,s_k)\right\rangle\\
    &\ge
    \frac{K\eta}{2}
    +\Potential(s_0)-\Potential(s_K)\\
    &\ge
    \frac{K\eta}{2}
    -\operatorname{range}(\Potential).
\end{aligned}
\]
Therefore
\[
    \frac{\|\cum_{KH}\|_2}{KH}
    \ge
    \frac{\eta}{2H}
    -
    \frac{\operatorname{range}(\Potential)}{KH},
\]
which stays bounded away from zero as $K\to\infty$, contradicting
pointwise approachability. Hence
\[
    \underline\lambda_H(u)\le0
    \qquad
    \forall u\in\mathbb S^{d-1},
\]
and therefore
\[
    \underline\lambda_H
    =
    \sup_{u\in\mathbb S^{d-1}}
    \underline\lambda_H(u)
    \le0.
\]
The finite-horizon representation
\[
    \underline\lambda_H(u)
    =
    \sup_{n\ge1}
    \inf_{s\in\WinSet}
    \frac{V_n^u(s)}{n}
\]
follows from Lemma~\ref{lem:supinf-relation}.
\end{proof}

\begin{proposition}[Exact characterization of uniform approachability]
\label{thm:sliding-uniform-exact}
Uniform approachability holds if and only if
\[
    \overline\lambda_H\le0.
\]
Moreover,
\[
    \overline\lambda_H(u)
    =
    \inf_{n\ge1}
    \sup_{s\in\WinSet}
    \frac{V_n^u(s)}{n}
    =
    \lim_{n\to\infty}
    \sup_{s\in\WinSet}
    \frac{V_n^u(s)}{n}.
\]
\end{proposition}

\begin{proof}
Apply Theorem~\ref{thm:sliding-convex-approachability} with the
singleton target
\[
    \mathcal D=\{0\}.
\]
Since
\[
    \sigma_{\{0\}}(u)=0,
    \qquad
    \forall u\in\mathbb S^{d-1},
\]
and
\[
    \operatorname{dist}(\cum_T,T\{0\})
    =
    \|\cum_T\|_2,
\]
uniform approachability of $\{0\}$ is exactly the existence of a
single causal learner strategy and a deterministic function
$\widetilde r(T)=o(T)$ such that
\[
    \|\cum_T\|_2\le \widetilde r(T)
\]
uniformly over all admissible initial states and all
$H$-sliding-window adversaries.

The condition in
Theorem~\ref{thm:sliding-convex-approachability} therefore reduces to
\[
    \inf_{\Potential\in\ell^\infty(\WinSet)}
    \sup_{s\in\WinSet}
    \left\{
        (\mathcal T_{u,H}\Potential)(s)-\Potential(s)
    \right\}
    \le0,
    \qquad
    \forall u\in\mathbb S^{d-1},
\]
which is precisely
\[
    \overline\lambda_H(u)\le0,
    \qquad
    \forall u\in\mathbb S^{d-1}.
\]
Equivalently,
\[
    \overline\lambda_H\le0.
\]

Dividing the cumulative bound by $T$ and setting
\[
    r(T)\coloneqq \frac{\widetilde r(T)}{T}=o(1)
\]
gives the uniform guarantee in
\eqref{eq:strong_approachability}.

Finally, Lemma~\ref{lem:supinf-relation} gives
\[
    \overline\lambda_H(u)
    =
    \inf_{n\ge1}
    \sup_{s\in\WinSet}
    \frac{V_n^u(s)}{n}
    =
    \lim_{n\to\infty}
    \sup_{s\in\WinSet}
    \frac{V_n^u(s)}{n}.
\]
\end{proof}

\begin{proposition}[Exact characterization under finite-step bridging]
\label{thm:sliding-exact}
Under Assumption~\ref{ass:finite-step-bridging}, for every
$u\in\mathbb S^{d-1}$,
\[
    \underline\lambda_H(u)
    =
    \overline\lambda_H(u)
    \eqqcolon
    \lambda_H(u).
\]
Consequently,
\[
    \underline\lambda_H
    =
    \overline\lambda_H
    \eqqcolon
    \lambda_H,
\]
and the following are equivalent:
\begin{enumerate}
    \item pointwise approachability;
    \item uniform approachability;
    \item $\lambda_H\le0$.
\end{enumerate}
\end{proposition}

\begin{proof}
Fix $u\in\mathbb S^{d-1}$. By
Lemma~\ref{lem:supinf-relation},
\[
    \underline\lambda_H(u)
    =
    \lim_{n\to\infty}
    \inf_{s\in\WinSet}
    \frac{V_n^u(s)}{n},
\]
while
\[
    \overline\lambda_H(u)
    =
    \lim_{n\to\infty}
    \sup_{s\in\WinSet}
    \frac{V_n^u(s)}{n}.
\]

Under Assumption~\ref{ass:finite-step-bridging}, any admissible
state can be reached from any other in at most $C_B$ block
transitions. Let $M_H$ be a uniform bound on the magnitude of an
$H$-round directional block payoff. Then
Lemma~\ref{lem:supinf-relation} gives
\[
    \sup_{s\in\WinSet}V_n^u(s)
    -
    \inf_{s\in\WinSet}V_n^u(s)
    \le
    2C_BM_H,
    \qquad
    \forall n\ge1.
\]
Therefore
\[
0\le
    \sup_{s\in\WinSet}\frac{V_n^u(s)}{n}
    -
    \inf_{s\in\WinSet}\frac{V_n^u(s)}{n}
\le
    \frac{2C_BM_H}{n}
\longrightarrow0.
\]
Hence
\[
    \underline\lambda_H(u)
    =
    \overline\lambda_H(u)
    \eqqcolon
    \lambda_H(u).
\]
Taking the supremum over directions yields
\[
    \underline\lambda_H
    =
    \overline\lambda_H
    \eqqcolon
    \lambda_H.
\]

If pointwise approachability holds, then
Proposition~\ref{thm:sliding-necessary} gives
\[
    \underline\lambda_H(u)\le0,
    \qquad
    \forall u\in\mathbb S^{d-1}.
\]
Since the lower and upper values coincide under bridging,
\[
    \overline\lambda_H(u)\le0,
    \qquad
    \forall u\in\mathbb S^{d-1},
\]
and Proposition~\ref{thm:sliding-uniform-exact} implies uniform
approachability.

Conversely, uniform approachability immediately implies
pointwise approachability.
\end{proof}

Combining
Propositions~\ref{thm:sliding-necessary},
\ref{thm:sliding-uniform-exact}, and
\ref{thm:sliding-exact}
proves Theorem~\ref{thm:slide-blackwell}.

\subsection{Proof of Theorem~\ref{thm:slide-quantitative} }

\begin{proof}
Fix a horizon $T\ge1$ and write
\[
K:=\left\lfloor\frac{T}{H}\right\rfloor,
\qquad
r:=T-KH,
\qquad
0\le r<H.
\]
If $K=0$, any sequence of actions in $C$ satisfies
\[
\left \|\sum_{t=1}^{T}\poff(x_t,z_t)\right\|_2
\le LT\le LH,
\]
so the claimed bound holds.
Henceforth assume that $K\ge1$.
Fix $m\in[K]$ and partition the first $K$ blocks into
\[
    N:=\left\lceil\frac{K}{m}\right\rceil
\]
consecutive epochs.  Define
\[
    k_0:=0,
    \qquad
    k_e:=\min\{em,K\},
    \qquad
    e=1,\ldots,N.
\]
Thus epoch $e$ consists of blocks
\[
    k=k_{e-1}+1,\ldots,k_e
\]
and contains at most $m$ blocks.

Let
\[
    S_e
    :=
    \sum_{k=1}^{k_{e-1}}
    \blk_H(\pi_k,s_k)
\]
denote the cumulative vector payoff at the beginning of epoch $e$, and
let
\[
    \Delta_e
    :=
    \sum_{k=k_{e-1}+1}^{k_e}
    \blk_H(\pi_k,s_k)
\]
denote the payoff accumulated during that epoch.
The corresponding initial and terminal sliding states are $s_{k_{e-1}}$
    and $
    s_{k_e}$
respectively, where $s_0$ is the initial state.

If $S_e\neq0$, set
\[
    u_e:=\frac{S_e}{\|S_e\|_2};
\]
otherwise choose any $u_e\in\mathbb S^{d-1}$.
The direction $u_e$ is kept fixed throughout epoch $e$.

Set
\[
    \varepsilon:=\frac{HL}{K}.
\]
At each block $k=k_{e-1}+1,\ldots,k_e$, choose an
$\varepsilon$-optimal policy in the definition of
$(T_{u_e,H}\Potential_{u_e})(s_{k-1})$.
Since
\[
    T_{u_e,H}\Potential_{u_e}\le\Potential_{u_e},
\]
the realized block payoff satisfies
\[
    \iprod{u_e,\blk_H(\pi_k,s_k)}
    +\Potential_{u_e}(s_k)
    -\Potential_{u_e}(s_{k-1})
    \le\varepsilon.
\]

Writing $q_e:=k_e-k_{e-1}$ and summing over the epoch gives
\[
\begin{aligned}
    \iprod{u_e,\Delta_e}
    &\le
    \Potential_{u_e}(s_{k_{e-1}})
    -\Potential_{u_e}(s_{k_e})
    +\varepsilon q_e\\
    &\le B_H+\varepsilon q_e.
\end{aligned}
\]
Moreover,
\[
    \norm{\Delta_e}\le HLq_e,
    \qquad q_e\le m.
\]
Consequently,
\[
    \norm{S_e+\Delta_e}^2
    \le
    \norm{S_e}^2
    +2(B_H+\varepsilon q_e)\norm{S_e}
    +(HL)^2q_e^2.
\]
This inequality also holds when $S_e=0$.
An elementary induction yields
\[
\begin{aligned}
    \norm{S_{N+1}}
    &\le
    \sum_{e=1}^N(B_H+\varepsilon q_e)
    +HL\left(\sum_{e=1}^Nq_e^2\right)^{1/2}\\
    &\le
    B_HN+\varepsilon K+HLm\sqrt N\\
    &=B_HN+HL+HLm\sqrt N.
\end{aligned}
\]
The additional term $HL$ is bounded by $HL\sqrt K$.
Optimizing over the integer $m\in[K]$ gives
\[
    \left\|
        \sum_{k=1}^K
        \blk_H(\pi_k,s_k)
    \right\|_2
    =
    O\!\left(
        B_H^{1/3}(HL)^{2/3}K^{2/3}
        +
       HL\sqrt K
        +
        B_H
    \right).
\]

We also have the trivial bound
\[\left\|\sum_{k=1}^{K}G_H(\pi_k,s_k)\right\|_2
\le KHL.
\]
If $B_H\le KHL$, then
$B_H\le B_H^{1/3}(HL)^{2/3}K^{2/3}$.
If $B_H>KHL$, the trivial bound is itself at most
$B_H^{1/3}(HL)^{2/3}K^{2/3}$.
Thus
\[
\left\|\sum_{k=1}^{K}G_H(\pi_k,s_k)\right\|_2
=
O\!\left(
B_H^{1/3}(HL)^{2/3}K^{2/3}
+
HL\sqrt K
\right).
\]

For the final $r$ rounds, choose arbitrary actions in $C$.
By the per-round payoff bound,
\[
\left\|\sum_{t=KH+1}^{T}\poff(x_t,z_t)\right\|_2
\le rL.
\]

Using $K\le T/H$ and $r<H$, we obtain
\[
\begin{aligned}
\left\|\sum_{t=1}^{T}\poff(x_t,z_t)\right\|_2
&\le
\left\|\sum_{k=1}^{K}G_H(\pi_k,s_k)\right\|_2
+rL
\\
&=
O\!\left(
B_H^{1/3}(HL)^{2/3}K^{2/3}
+
HL\sqrt K
+
rL
\right)
\\
&=
O\!\left(
B_H^{1/3}L^{2/3}T^{2/3}
+
L\sqrt{HT}
+
HL
\right).
\end{aligned}
\]
This proves the claimed bound for every $T\ge1$.
\end{proof}

\subsection{Proof of Proposition~\ref{prop:comparison}}
\begin{proof}

We give a simple scalar example as follows. 

Let $H=2$. The learner chooses
\[
    p_t \in [-1,1],
\]
and the adversary chooses
\[
    z_t \in \{+,-,0\}.
\]
The scalar payoff is
\[
    F_+(p)
    :=
    \frac14 (p-1)^2,
    \qquad
    F_-(p)
    :=
    \frac14 (p+1)^2,
    \qquad
    F_0(p)
    :=
    -\frac14.
\]

Recall that a two-round adversary window $(z_1,z_2)$ is locally feasible 
if there exists a fixed comparator $p\in[-1,1]$ whose cumulative payoff
over that window is zero:
\[
    \exists p\in[-1,1]
    \quad\text{such that}\quad
    F_{z_1}(p)+F_{z_2}(p)=0.
\]
Accordingly, define
\[
    \mathcal W_2
    :=
    \left\{
        (z_1,z_2)\in\{+,-,0\}^2:
        \exists p\in[-1,1],\
        F_{z_1}(p)+F_{z_2}(p)=0
    \right\}.
\]
A direct calculation gives
\[
    \mathcal W_2
    =
    \{
        (+,+),\,(-,-),\,(+,0),\,(0,+),\,(-,0),\,(0,-)
    \}.
\]

Consider first the blockwise model, in which each two-round block is
required only to belong to $\mathcal W_2$, independently of the previous
blocks. At the first round of a fresh block, after observing the
learner's choice $p_1$, the adversary may choose either $+$ or $-$ and
subsequently repeat the same symbol, since both $(+,+)$ and $(-,-)$
belong to $\mathcal W_2$. Therefore,
\[
    \sup_{z_{1:2}\in\mathcal W_2}
    \sum_{t=1}^2 F_{z_t}(p_t)
    \ge
    \max\{F_+(p_1),F_-(p_1)\}.
\]
Consequently,
\[
    \inf_{p_1\in[-1,1]}
    \max\{F_+(p_1),F_-(p_1)\}
    =
    \frac14,
\]
where the minimum is attained at $p_1=0$. Thus, for every causal
two-round learner policy $\pi$,
\[
    \sup_{z_{1:2}\in\mathcal W_2}
    G_2(\pi,z_{1:2})
    \ge
    \frac14,
\]
where
\[
    G_2(\pi,z_{1:2})
    :=
    \sum_{t=1}^2 F_{z_t}(p_t).
\]

Hence
\[
    \inf_{\pi}
    \sup_{z_{1:2}\in\mathcal W_2}
    G_2(\pi,z_{1:2})
    \ge
    \frac14>0,
\]
indicating that the blockwise Blackwell condition fails. Hence approachability of the origin is impossible under the blockwise adversary.

Now impose the same local-feasibility condition on every consecutive
two-round window:
\[
    (z_{t-1},z_t)\in\mathcal W_2
    \qquad
    \forall t\ge2.
\]
Thus the only admissible transitions are
\[
    +\to+,\qquad +\to 0,
\]
\[
    -\to-,\qquad -\to 0,
\]
and
\[
    0\to+,\qquad 0\to-.
\]
In particular, the adversary cannot switch directly between $+$ and
$-$; every such switch must pass through the intermediate state $0$.

Set $p_1=0$. For $t\ge2$, use the state-dependent learner strategy
\[
    p_t
    =
    \begin{cases}
        1, & z_{t-1}=+,\\
        -1, & z_{t-1}=-,\\
        0, & z_{t-1}=0.
    \end{cases}
\]
Define the bias function
\[
    h(+)=h(-)=0,
    \qquad
    h(0)=\frac14.
\]
Then every admissible transition satisfies
\[
    F_{z_t}(p_t)
    +
    h(z_t)-h(z_{t-1})
    =
    0.
\]
Indeed,
\[
    +\to+:
    \qquad
    F_+(1)+h(+)-h(+)=0,
\]
\[
    +\to 0:
    \qquad
    F_0(1)+h(0)-h(+)
    =
    -\frac14+\frac14=0,
\]
and symmetrically
\[
    -\to-:
    \qquad
    F_-(-1)+h(-)-h(-)=0,
\]
\[
    -\to 0:
    \qquad
    F_0(-1)+h(0)-h(-)=0.
\]
Finally,
\[
    0\to+:
    \qquad
    F_+(0)+h(+)-h(0)
    =
    \frac14-\frac14=0,
\]
and
\[
    0\to-:
    \qquad
    F_-(0)+h(-)-h(0)
    =
    \frac14-\frac14=0.
\]

Therefore the cumulative payoff telescopes:
\[
    \sum_{t=1}^T F_{z_t}(p_t)
    =
    F_{z_1}(0)+h(z_1)-h(z_T).
\]
Since
\[
    F_{z_1}(0)+h(z_1)\in\{0,1/4\}
    \qquad\text{and}\qquad
    h(z_T)\in\{0,1/4\},
\]
we obtain
\[
    \left|
        \sum_{t=1}^T F_{z_t}(p_t)
    \right|
    \le
    \frac14
    \qquad
    \forall T.
\]
Hence
\[
    \frac1T
    \sum_{t=1}^T F_{z_t}(p_t)
    \longrightarrow 0.
\]
Then approachability of the origin is possible under the sliding window adversary.

\end{proof}

\subsection{Proof of Proposition~\ref{prop:1d_example_sufficiency}}
\begin{proof}
Let $C=[a,b]\subset\mathbb R$, and for each length-$H$ window $I$ define
\[
    G_I(x)\coloneqq \sum_{t\in I} g_t(x).
\]
When each $\ell_t$ is convex, $G_I(\cdot)$ is nondecreasing. Under local feasibility, every admissible window admits some $x_I\in[a,b]$ satisfying
\(
    G_I(x_I)=0.
\)
Consequently,
\(
    G_I(a)\le 0\le G_I(b)
\)
for every admissible window.

Since $\mathbb S^0=\{-1,+1\}$, the two endpoints provide universal Blackwell responses: for $u=+1$ the learner may use the constant action $a$, whereas for $u=-1$ she may use $b$.
Hence
\[
    \sup_{u\in\mathbb S^0}\inf_{\pi\in\Pi_H}
    \sup_{s\in\mathcal W_H}
    u\,G_H(\pi,s)
    \le 0.
\]
Thus the blockwise Blackwell condition holds, and Theorem~\ref{thm:split-blackwell} implies CGEQ. Since every sliding window adversary is also blockwise, the same conclusion holds for the sliding window adversary.
\end{proof}

\subsection{Proof of Proposition~\ref{thm:invalid_weak_feasibility}}
\label{app:proof_of_invalid_weak_feasibility}

\begin{proof}
Let
\[
C
:=
\left\{
p=(p_1,p_2)\in\mathbb R_+^2:
\|p\|_2\le1
\right\}.
\]
At round $t$, the learner chooses $p_t\in C$.  After observing $p_t$, the
adversary chooses
\[
J_t\in\{1,2\},
\qquad
b_t\in\{0,1\}.
\]
Define
\[
A_t:=e_{J_t}e_{J_t}^\top,
\qquad
y_t:=b_t e_{J_t}.
\]
Here $J_t$ and $b_t$ correspond to $j_t$ and $z_t$ in
Proposition~\ref{thm:invalid_weak_feasibility}.  Define the residual vector
\[
r_t(p_t)
:=
A_t(y_t-p_t).
\]
Equivalently,
\[
r_t(p_t)
=
-\nabla\ell_t(p_t)
\]
for the quadratic loss
\[
\ell_t(p)
:=
\frac12(p-y_t)^\top A_t(p-y_t).
\]
In the notation of the main text,
$g_t(p):=\nabla\ell_t(p)=A_t(p-y_t)$, so $r_t(p_t)=-g_t(p_t)$.
Thus driving the average gradients to zero is equivalent to driving the
average residuals $r_t$ to zero.

For a window $I$, define
\[
I_i
:=
\{t\in I:J_t=i\},
\qquad i=1,2,
\]
and
\[
\bar b_i(I)
:=
\begin{cases}
\displaystyle
\frac1{|I_i|}
\sum_{t\in I_i}b_t,
& |I_i|>0,\\[2mm]
0,
& |I_i|=0.
\end{cases}
\]
Consider the sliding-window restriction
\begin{equation}
\bar b_1(I)^2+\bar b_2(I)^2\le1
\label{eq:two-coordinate-window}
\end{equation}
for every consecutive window $I$ of length $H$.

Condition~\eqref{eq:two-coordinate-window} implies the local feasibility
condition~(9) in the main text.  Indeed, for every such window,
\[
q_I
:=
\bigl(\bar b_1(I),\bar b_2(I)\bigr)
\in C
\]
satisfies
\[
\sum_{t\in I}g_t(q_I)
=
-\sum_{t\in I}A_t(y_t-q_I)
=
0.
\]
Hence every individual length-$H$ window admits a fixed action with
\emph{exactly zero vector residual}.

Nevertheless, this does not imply GEQ.  We construct a deterministic causal
adversary satisfying \eqref{eq:two-coordinate-window} such that, against every
learner sequence $p_t\in C$,
\[
\max_{i\in\{1,2\}}
\sum_{t:J_t=i}(b_t-p_{t,i})
\ge
\frac{\rho_H}{2}T-O_H(1)
\]
for some $\rho_H>0$.

Let
\[
u^\star:=\frac{(1,1)}{\sqrt2},
\qquad
c_\star:=1-\frac1{\sqrt2}.
\]
For sufficiently small $\gamma>0$, define
\[
D_0(\gamma):=0,
\qquad
D_{q+1}(\gamma)
:=
D_q(\gamma)
+
\sqrt{
1-
\bigl(1-D_q(\gamma)-\gamma\bigr)^2
}.
\]
Since $D_q(\gamma)\to0$ as $\gamma\downarrow0$ for every fixed $q$,
choose $\gamma_H>0$ such that
\[
D_{H-1}(\gamma_H)+\gamma_H\le c_\star,
\]
and set
\[
\rho_H:=\frac{\gamma_H}{H}.
\]

Define
\[
\mathcal I_H
:=
\left\{
s\in\WinSet:
\text{all rounds with $b=1$ in $s$ lie on the same coordinate}
\right\}.
\]
For $s\in\mathcal I_H$, let
\[
q(s)\in\{0,\ldots,H-1\}
\]
be the number of trailing rounds with $b=0$, truncated at $H-1$.
If $q(s)<H-1$, let $i(s)$ be the unique coordinate on which rounds with
$b=1$ occur.  Define
\[
\widetilde\Potential(s)
:=
D_{q(s)}(\gamma_H)+q(s)\rho_H.
\]

We claim that for every $s\in\mathcal I_H$ and every learner action
$p_t$, the adversary can choose the next pair $(J_t,b_t)$ such that the
one-round rolling successor state $s^+$, obtained from $s$ by
dropping its oldest round and appending $(J_t,b_t)$, belongs to
$\mathcal I_H$ and satisfies
\begin{equation}
\label{eq:two-coordinate-round-drift}
\left\langle
(1,1),r_t(p_t)
\right\rangle
+
\widetilde\Potential(s^+)
-
\widetilde\Potential(s)
\ge
\rho_H.
\end{equation}
Here $s^+$ denotes only this single-round update; it is not the
$H$-round successor in $\WinSet(s)$ used in the block-level notation of
the main text.

Write $q=q(s)$.  If $q<H-1$, let $i=i(s)$.  When
\[
1-p_{t,i}-D_q(\gamma_H)\ge\gamma_H,
\]
choose
\[
(J_t,b_t)=(i,1).
\]
Then $q(s^+)=0$ and
\[
\begin{aligned}
&
\left\langle(1,1),r_t(p_t)\right\rangle
+
\widetilde\Potential(s^+)
-
\widetilde\Potential(s)
\\
&=
1-p_{t,i}-D_q(\gamma_H)-q\rho_H
\\
&\ge
\gamma_H-q\rho_H
\ge
\rho_H.
\end{aligned}
\]

Otherwise,
\[
p_{t,i}
>
1-D_q(\gamma_H)-\gamma_H.
\]
Let $j\neq i$.  Since $\|p_t\|_2\le1$,
\[
\begin{aligned}
p_{t,j}
&\le
\sqrt{1-p_{t,i}^2}
\\
&<
\sqrt{
1-
\bigl(1-D_q(\gamma_H)-\gamma_H\bigr)^2
}
\\
&=
D_{q+1}(\gamma_H)-D_q(\gamma_H).
\end{aligned}
\]
Choose
\[
(J_t,b_t)=(j,0).
\]
Then $q(s^+)=q+1$ and
\[
\begin{aligned}
&
\left\langle(1,1),r_t(p_t)\right\rangle
+
\widetilde\Potential(s^+)
-
\widetilde\Potential(s)
\\
&=
-p_{t,j}
+
D_{q+1}(\gamma_H)-D_q(\gamma_H)
+
\rho_H
\\
&\ge
\rho_H.
\end{aligned}
\]

Finally, if $q=H-1$, choose a coordinate
\[
i\in\arg\min_{\ell\in\{1,2\}}p_{t,\ell}
\]
and set
\[
(J_t,b_t)=(i,1).
\]
Since $\|p_t\|_2\le1$,
\[
p_{t,i}\le\frac1{\sqrt2},
\]
and therefore
\[
\begin{aligned}
&
\left\langle(1,1),r_t(p_t)\right\rangle
+
\widetilde\Potential(s^+)
-
\widetilde\Potential(s)
\\
&\ge
c_\star
-
D_{H-1}(\gamma_H)
-
(H-1)\rho_H
\\
&\ge
\rho_H.
\end{aligned}
\]
In all three cases, the one-round rolling successor satisfies
$s^+\in\mathcal I_H$, so \eqref{eq:two-coordinate-round-drift} follows.

Now use the block-state notation of the main text.  For block $k$, let
\[
\sigma_{k,0}:=s_{k-1},
\qquad
\sigma_{k,H}:=s_k,
\]
where $\sigma_{k,r}$ is the rolling length-$H$ state after the first
$r$ rounds of block $k$.
If $s_{k-1}\in\mathcal I_H$, applying
\eqref{eq:two-coordinate-round-drift} successively for
$r=1,\ldots,H$ yields an admissible $H$-round successor
\[
s_k\in\WinSet(s_{k-1})\cap\mathcal I_H.
\]
Since $\blk_H(\pi_k,s_k)$ is the cumulative gradient payoff,
the block residual sum is its negative.  Thus
\begin{equation}
\label{eq:two-coordinate-block-drift}
\left\langle
(1,1),-\blk_H(\pi_k,s_k)
\right\rangle
+
\widetilde\Potential(s_k)
-
\widetilde\Potential(s_{k-1})
\ge
\gamma_H.
\end{equation}

It remains to enter $\mathcal I_H$ from an arbitrary
$s_0\in\WinSet$.  During the first $H$ rounds, whenever the oldest
pair in the current window is $(J,b)$, append $(J,0)$.  Each such
replacement preserves admissibility, and after $H$ rounds the resulting
state $s_1$ has $b=0$ throughout.  Hence
\[
s_1\in\mathcal I_H.
\]
The directional residual accumulated during this warm-up is at least
$-H$.

Applying \eqref{eq:two-coordinate-block-drift} for blocks
$k=2,\ldots,K$ and telescoping gives
\[
\begin{aligned}
\left\langle
(1,1),-\cum_{KH}
\right\rangle
&\ge
-H
+
(K-1)\gamma_H
+
\widetilde\Potential(s_1)
-
\widetilde\Potential(s_K).
\end{aligned}
\]
Since $\widetilde\Potential$ is bounded on $\mathcal I_H$, there exists
$C_H<\infty$ such that
\[
\left\langle
(1,1),-\cum_{KH}
\right\rangle
\ge
(K-1)\gamma_H-C_H.
\]
Therefore
\[
\left\langle
u^\star,-\cum_{KH}
\right\rangle
\ge
\frac{\gamma_H}{\sqrt2\,H}\,KH-C_H',
\]
for some $C_H'<\infty$.

For an arbitrary $T$, let $K=\lfloor T/H\rfloor$.  Each of the fewer than
$H$ remaining rounds contributes at least $-1$ to
$\langle(1,1),r_t(p_t)\rangle$.  Thus the preceding bound extends to
\[
\left\langle u^\star,-\cum_T\right\rangle
\ge
\frac{\gamma_H}{\sqrt2\,H}T-O_H(1),
\]
and hence $\|\cum_T\|_2=\Omega_H(T)$.
Since $-\cum_T=\sum_{t=1}^T r_t(p_t)$ and
$r_t(p_t)=-\nabla\ell_t(p_t)$,
\[
\left\|
\sum_{t=1}^T\nabla\ell_t(p_t)
\right\|_2
=
\Omega_H(T).
\]
Moreover, the two coordinate sums of $-\cum_T$ are precisely
$\sum_{t:J_t=i}(b_t-p_{t,i})$, $i=1,2$.  Their maximum is at least
half their sum, proving the stated lower bound with
$\rho_H=\gamma_H/H$.  Thus GEQ is not achievable.
\end{proof}
\subsection{Proof of Theorem~\ref{thm:CGEQ-under-local-feasibility}}

\begin{proof}
Let
\[
S_t:=\sum_{s=1}^t g_s(x_s).
\]
For each complete block, let
\[
I_k:=\{(k-1)H+1,\ldots,kH\}.
\]
At the beginning of block $I_k$, Algorithm~\ref{alg:known-H-window}
chooses
\[
u_k
:=
\begin{cases}
\dfrac{S_{(k-1)H}}{\|S_{(k-1)H}\|_2},
& S_{(k-1)H}\neq0,\\[2mm]
\text{any }u\in\mathbb S^{d-1},
& S_{(k-1)H}=0.
\end{cases}
\]
By \eqref{eq:local-feasibility},
\[
\min_{x^\star\in C}
\sum_{t\in I_k}
\langle u_k,g_t(x^\star)\rangle
\le
-\gamma H.
\]
Applying the regret guarantee of $\mathsf A_H$ on the same block and
direction gives
\[
\sum_{t\in I_k}
\langle u_k,g_t(x_t)\rangle
\le
-\gamma H+R_H.
\]
Hence, whenever
\[
\gamma\ge\frac{R_H}{H},
\]
we have
\[
\sum_{t\in I_k}
\langle u_k,g_t(x_t)\rangle
\le0.
\]
If $S_{(k-1)H}\neq0$, this implies
\[
\begin{aligned}
\left\langle
S_{(k-1)H},
S_{kH}-S_{(k-1)H}
\right\rangle
&=
\|S_{(k-1)H}\|_2
\sum_{t\in I_k}
\langle u_k,g_t(x_t)\rangle
\\
&\le0.
\end{aligned}
\]
The same inequality is immediate when $S_{(k-1)H}=0$.
Therefore,
\[
\begin{aligned}
\|S_{kH}\|_2^2
&=
\|S_{(k-1)H}\|_2^2
+
2
\left\langle
S_{(k-1)H},
S_{kH}-S_{(k-1)H}
\right\rangle
\\
&\qquad+
\|S_{kH}-S_{(k-1)H}\|_2^2
\\
&\le
\|S_{(k-1)H}\|_2^2+L^2H^2,
\end{aligned}
\]
where we used $\|g_t(x_t)\|_2\le L$ to obtain
\[
\|S_{kH}-S_{(k-1)H}\|_2
\le
\sum_{t\in I_k}\|g_t(x_t)\|_2
\le LH.
\]
Iterating over the complete blocks yields
\[
\|S_{kH}\|_2^2
\le
kL^2H^2,
\]
and hence
\[
\|S_{kH}\|_2
\le
LH\sqrt{k}.
\]
For an arbitrary $T$, let $k=\lfloor T/H\rfloor$.  Since
$0\le T-kH<H$,
\[
\begin{aligned}
\|S_T\|_2
&\le
\|S_{kH}\|_2
+
\|S_T-S_{kH}\|_2
\\
&\le LH\sqrt{k}+LH
\\
&\le L\sqrt{HT}+LH.
\end{aligned}
\]
Dividing by $T$ gives
\[
\frac1T
\left\|
\sum_{t=1}^T g_t(x_t)
\right\|_2
\le
L\sqrt{\frac HT}
+
\frac{LH}{T},
\]
which converges to zero for every fixed $H$.  Thus
Algorithm~\ref{alg:known-H-window} achieves CGEQ.
\end{proof}

\subsection{Instantiating Algorithm~\ref{alg:known-H-window} with FTL}
\label{app:FLT_instance}

To give a concrete example of the online learning subroutine in Algorithm~\ref{alg:known-H-window}, consider a constrained set $C$ with strictly positive curvature and smooth convex losses with Lipschitz Hessians in Assumption~\ref{ass:regularity for set and loss}. The curvature allows us to apply the follow-the-leader algorithm with a logarithmic regularized regret \cite{hazan2007logarithmic,huang2017following}, as shown in Algorithm~\ref{alg:window-ftl}.
\begin{algorithm}[H]
\caption{FTL on shifted directional losses within an $H$-round block}
\label{alg:window-ftl}
\begin{algorithmic}[1]
\State Input a block $I_k=((k-1)H,kH]$, a fixed direction $u_k$,
an anchor $x_0\in C$, and $\lambda\ge0$
\State Choose any $x_{(k-1)H+1}\in C$
\For{$t=(k-1)H+2,\ldots,kH$}
    \State
    \[
        x_t
        \gets
        \arg\min_{x\in C}
        \sum_{s=(k-1)H+1}^{t-1}
        \left\langle
            u_k,
            g_s(x)+\lambda(x-x_0)
        \right\rangle
    \]
\EndFor
\end{algorithmic}
\end{algorithm}
\begin{assumption}
\label{ass:regularity for set and loss}
Assume that the constrained action set $C$ is compact, convex, and has a continuously differentiable boundary.  Let $n(x)$ denote the outward unit normal at
$x\in\partial C$, and assume that $\exists \kappa_C>0$ such that
\begin{equation*}
    -\langle n(x),y-x\rangle
    \ge
    \frac{\kappa_C}{2}\|y-x\|_2^2,
    \qquad
    \forall x\in\partial C,\quad \forall y\in C.
    \label{eq:boundary-curvature}
\end{equation*}
Also assume that $\exists M\ge m \ge 0$ and $\rho\ge 0$ such that
{\small
\begin{equation*}
    mI
    \preceq
    \nabla_x^2\ell_t(x)
    \preceq
    MI,
    \quad
    \left\|
        \nabla_x^2\ell_t(x)
        -
        \nabla_x^2\ell_t(y)
    \right\|_2
    \le
    \rho\|x-y\|_2.
    \label{eq:window-ftl-geometry}
\end{equation*}
}
\end{assumption}

\begin{corollary}\label{cor:log-regret-ftl}
Suppose \eqref{eq:local-feasibility} holds with $\gamma>0$ and Assumption~\ref{ass:regularity for set and loss} holds. We can use Algorithm~\ref{alg:window-ftl} with a regularization $\lambda\ge0$ satisfying $\lambda > \frac{\rho}{\kappa_C}-m$ for $\mathsf{A}_H$ in Algorithm~\ref{alg:known-H-window}, which gives $R_H\le (M+\lambda)\operatorname{diam}(C)
+
\frac{(M+\lambda)^2}
{2\bigl((m+\lambda)\kappa_C-\rho\bigr)}
\bigl(1+\log H\bigr)
+
\lambda H\operatorname{diam}(C).$

Therefore, CGEQ is achieved with sufficiently large $\gamma$.
\end{corollary}

\begin{proof}
Fix a length-$H$ window and a direction
$u\in\mathbb S^{d-1}$, and relabel the rounds in the window as
$t=1,\ldots,H$.  Define the regularized directional losses
\[
\widetilde q_t^u(x)
:=
\left\langle
u,
g_t(x)+\lambda(x-x_0)
\right\rangle,
\]
and let
\[
\widetilde F_t(x)
:=
\sum_{s=1}^t \widetilde q_s^u(x).
\]
Algorithm~\ref{alg:window-ftl} chooses
\[
x_{t+1}
\in
\arg\min_{x\in C}\widetilde F_t(x),
\qquad
t=1,\ldots,H-1,
\]
where we also define
\[
x_{H+1}
\in
\arg\min_{x\in C}\widetilde F_H(x).
\]

We first establish quadratic growth of each cumulative regularized
objective.  Let
\[
\overline{\widetilde F}_t(x)
:=
\frac1t\widetilde F_t(x).
\]
Since $g_s=\nabla\ell_s$,
\[
\nabla\overline{\widetilde F}_t(x)
=
\frac1t
\sum_{s=1}^t
\left(
\nabla_x^2\ell_s(x)+\lambda I
\right)u.
\]
By Assumption~\ref{ass:regularity for set and loss},
\[
\frac1t
\sum_{s=1}^t
\left(
\nabla_x^2\ell_s(x)+\lambda I
\right)
\succeq
(m+\lambda)I.
\]
Since $\|u\|_2=1$,
\[
\left\|
\nabla\overline{\widetilde F}_t(x)
\right\|_2
\ge
m+\lambda.
\]
Thus any minimizer $x_t^\star$ of
$\overline{\widetilde F}_t$ lies on $\partial C$.  By the KKT condition,
\[
\nabla\overline{\widetilde F}_t(x_t^\star)
=
-\mu_t n(x_t^\star)
\]
for some $\mu_t\ge0$, and
\[
\mu_t
=
\left\|
\nabla\overline{\widetilde F}_t(x_t^\star)
\right\|_2
\ge
m+\lambda.
\]

Let
\[
d:=x-x_t^\star.
\]
Then
\begin{align*}
\overline{\widetilde F}_t(x)
-
\overline{\widetilde F}_t(x_t^\star)
&=
\left\langle
\nabla\overline{\widetilde F}_t(x_t^\star),
d
\right\rangle
\\
&\quad+
\int_0^1
\left\langle
\nabla\overline{\widetilde F}_t(x_t^\star+s d)
-
\nabla\overline{\widetilde F}_t(x_t^\star),
d
\right\rangle
\,ds.
\end{align*}
The added term contributes the constant vector \(\lambda u\) to the gradient, and therefore cancels from the gradient difference. Therefore, by
Assumption~\ref{ass:regularity for set and loss},
\[
\left\|
\nabla\overline{\widetilde F}_t(x_t^\star+s d)
-
\nabla\overline{\widetilde F}_t(x_t^\star)
\right\|_2
\le
\rho s\|d\|_2.
\]
Using the curvature condition for $C$,
\begin{align*}
\overline{\widetilde F}_t(x)
-
\overline{\widetilde F}_t(x_t^\star)
&\ge
-\mu_t
\left\langle
n(x_t^\star),
x-x_t^\star
\right\rangle
-
\frac{\rho}{2}
\|x-x_t^\star\|_2^2
\\
&\ge
\frac{\mu_t\kappa_C-\rho}{2}
\|x-x_t^\star\|_2^2
\\
&\ge
\frac{(m+\lambda)\kappa_C-\rho}{2}
\|x-x_t^\star\|_2^2.
\end{align*}
Because
\[
\lambda>\frac{\rho}{\kappa_C}-m,
\]
the coefficient $(m+\lambda)\kappa_C-\rho$ is positive.  Multiplying by
$t$ gives
\begin{equation}
\widetilde F_t(x)
-
\min_{y\in C}\widetilde F_t(y)
\ge
\frac{t\bigl((m+\lambda)\kappa_C-\rho\bigr)}{2}
\|x-x_{t+1}\|_2^2.
\label{eq:regularized-window-quadratic-growth}
\end{equation}

For $t\ge2$, since $x_t$ minimizes $\widetilde F_{t-1}$,
\eqref{eq:regularized-window-quadratic-growth} yields
\[
\widetilde F_{t-1}(x_{t+1})
-
\widetilde F_{t-1}(x_t)
\ge
\frac{(t-1)\bigl((m+\lambda)\kappa_C-\rho\bigr)}{2}
\|x_{t+1}-x_t\|_2^2.
\]
Moreover,
\[
\begin{aligned}
\|\nabla\widetilde q_t^u(x)\|_2
&=
\left\|
\left(
\nabla_x^2\ell_t(x)+\lambda I
\right)u
\right\|_2
\\
&\le
M+\lambda,
\end{aligned}
\]
so $\widetilde q_t^u$ is $(M+\lambda)$-Lipschitz.  Hence
\begin{align*}
\widetilde F_t(x_t)-\widetilde F_t(x_{t+1})
&=
\widetilde F_{t-1}(x_t)
-
\widetilde F_{t-1}(x_{t+1})
\\
&\quad+
\widetilde q_t^u(x_t)
-
\widetilde q_t^u(x_{t+1})
\\
&\le
-\frac{(t-1)\bigl((m+\lambda)\kappa_C-\rho\bigr)}{2}
\|x_{t+1}-x_t\|_2^2
\\
&\quad+
(M+\lambda)\|x_{t+1}-x_t\|_2
\\
&\le
\frac{(M+\lambda)^2}
{2(t-1)\bigl((m+\lambda)\kappa_C-\rho\bigr)}.
\end{align*}
For the first round,
\[
\widetilde F_1(x_1)
-
\min_{x\in C}\widetilde F_1(x)
\le
(M+\lambda)\operatorname{diam}(C).
\]

Using the standard FTL telescoping identity,
\begin{align*}
&
\sum_{t=1}^H
\widetilde q_t^u(x_t)
-
\min_{x\in C}
\sum_{t=1}^H
\widetilde q_t^u(x)
\\
&=
\sum_{t=1}^H
\left[
\widetilde F_t(x_t)
-
\widetilde F_t(x_{t+1})
\right]
\\
&\le
(M+\lambda)\operatorname{diam}(C)
+
\frac{(M+\lambda)^2}
{2\bigl((m+\lambda)\kappa_C-\rho\bigr)}
\sum_{t=2}^H\frac1{t-1}
\\
&\le
(M+\lambda)\operatorname{diam}(C)
+
\frac{(M+\lambda)^2}
{2\bigl((m+\lambda)\kappa_C-\rho\bigr)}
\bigl(1+\log H\bigr).
\end{align*}

It remains to transfer this bound to the original directional losses.
Let
\[
x^\star
\in
\arg\min_{x\in C}
\sum_{t=1}^H
\langle u,g_t(x)\rangle.
\]
Applying the preceding regularized regret bound with comparator
$x^\star$ gives
\[
\sum_{t=1}^H
\widetilde q_t^u(x_t)
\le
\sum_{t=1}^H
\widetilde q_t^u(x^\star)
+
(M+\lambda)\operatorname{diam}(C)
+
\frac{(M+\lambda)^2}
{2\bigl((m+\lambda)\kappa_C-\rho\bigr)}
\bigl(1+\log H\bigr).
\]
Since
\[
\widetilde q_t^u(x)
=
\langle u,g_t(x)\rangle
+
\lambda\langle u,x-x_0\rangle,
\]
we obtain
\begin{align*}
&
\sum_{t=1}^H
\langle u,g_t(x_t)\rangle
-
\min_{x\in C}
\sum_{t=1}^H
\langle u,g_t(x)\rangle
\\
&\le
(M+\lambda)\operatorname{diam}(C)
+
\frac{(M+\lambda)^2}
{2\bigl((m+\lambda)\kappa_C-\rho\bigr)}
\bigl(1+\log H\bigr)
\\
&\quad+
\lambda
\sum_{t=1}^H
\langle u,x^\star-x_t\rangle
\\
&\le
(M+\lambda)\operatorname{diam}(C)
+
\frac{(M+\lambda)^2}
{2\bigl((m+\lambda)\kappa_C-\rho\bigr)}
\bigl(1+\log H\bigr)
+
\lambda H\operatorname{diam}(C).
\end{align*}
\end{proof}

\subsection{A key lemma}

\begin{lemma}
\label{lem:supinf-relation}
For every $u\in\mathbb S^{d-1}$,
\begin{equation}
\label{eq:dynamic programming representation}
    V_n^u
    =
    \mathcal T_{u,H}^n0,
\end{equation}
and
\begin{equation}
\label{eq:supinf relation}
\begin{aligned}
&
\sup_{\Potential\in\ell^\infty(\WinSet)}
\inf_{s\in\WinSet}
\left\{
    (\mathcal T_{u,H}\Potential)(s)-\Potential(s)
\right\}
\\
&\quad=
\sup_{n\ge1}
\inf_{s\in\WinSet}
\frac{V_n^u(s)}{n}
=
\lim_{n\to\infty}
\inf_{s\in\WinSet}
\frac{V_n^u(s)}{n}
\\
&\quad\le
\sup_{s\in\WinSet}
\inf_{n\ge1}
\frac{V_n^u(s)}{n}
\\
&\quad\le
\lim_{n\to\infty}
\sup_{s\in\WinSet}
\frac{V_n^u(s)}{n}
=
\inf_{n\ge1}
\sup_{s\in\WinSet}
\frac{V_n^u(s)}{n}
\\
&\quad=
\inf_{\Potential\in\ell^\infty(\WinSet)}
\sup_{s\in\WinSet}
\left\{
    (\mathcal T_{u,H}\Potential)(s)-\Potential(s)
\right\}.
\end{aligned}
\end{equation}

If, in addition,
Assumption~\ref{ass:finite-step-bridging} holds, then both inequalities
in \eqref{eq:supinf relation} are equalities. Moreover, for every
$n\ge1$,
\[
    \sup_{s\in\WinSet}V_n^u(s)
    -
    \inf_{s\in\WinSet}V_n^u(s)
    \le
    2C_BM_H.
\]
\end{lemma}

\begin{proof}
We first verify
\eqref{eq:dynamic programming representation}.
By definition,
\[
V_n^u(s)
=
\inf_{\boldsymbol\pi\in\Pi_{n,H}}
\sup_{s_{1:n}\in\mathcal J_n(s)}
\sum_{k=1}^n
\left\langle
    u,\blk_H(\pi_k,s_k)
\right\rangle,
\qquad
V_0^u=0.
\]
A causal $(n+1)$-block policy consists of a first-block policy together
with a continuation policy that may depend on the observed successor
state. Hence
\[
\begin{aligned}
V_{n+1}^u(s)
&=
\inf_{\pi\in\Pi_H}
\sup_{s'\in\WinSet(s)}
\left\{
    \left\langle
        u,\blk_H(\pi,s')
    \right\rangle
    +
    V_n^u(s')
\right\}
\\
&=
(\mathcal T_{u,H}V_n^u)(s).
\end{aligned}
\]
Induction gives
\[
    V_n^u=\mathcal T_{u,H}^n0.
\]

The operator $\mathcal T_{u,H}$ is monotone and additively homogeneous.
Thus, since
\[
    \inf_sV_m^u(s)
    \le
    V_m^u
    \le
    \sup_sV_m^u(s),
\]
applying $\mathcal T_{u,H}^n$ yields
\[
V_n^u+\inf_sV_m^u(s)
\le
V_{n+m}^u
\le
V_n^u+\sup_sV_m^u(s).
\]
Consequently,
$\inf_sV_n^u(s)$ is superadditive and
$\sup_sV_n^u(s)$ is subadditive.
Fekete's lemma gives
\[
\lim_{n\to\infty}
\inf_s\frac{V_n^u(s)}{n}
=
\sup_{n\ge1}
\inf_s\frac{V_n^u(s)}{n},
\]
and
\[
\lim_{n\to\infty}
\sup_s\frac{V_n^u(s)}{n}
=
\inf_{n\ge1}
\sup_s\frac{V_n^u(s)}{n}.
\]

We next identify the lower value.
For any bounded $\Potential$, let
\[
c_{\Potential}
:=
\inf_{s\in\WinSet}
\left\{
    (\mathcal T_{u,H}\Potential)(s)-\Potential(s)
\right\}.
\]
Then
\[
    \mathcal T_{u,H}\Potential
    \ge
    \Potential+c_{\Potential}.
\]
Iteration, monotonicity, and additive homogeneity imply
\[
    \mathcal T_{u,H}^n\Potential
    \ge
    \Potential+nc_{\Potential}.
\]
On the other hand,
\[
    \Potential\le\sup_s\Potential(s)
\]
gives
\[
    \mathcal T_{u,H}^n\Potential
    \le
    V_n^u+\sup_s\Potential(s).
\]
Therefore
\[
    \inf_sV_n^u(s)
    \ge
    nc_{\Potential}
    -
    \operatorname{range}(\Potential).
\]
Dividing by $n$ and letting $n\to\infty$ yields
\[
\lim_{n\to\infty}
\inf_s\frac{V_n^u(s)}{n}
\ge
c_{\Potential}.
\]
Taking the supremum over bounded $\Potential$ gives one inequality.

For the reverse inequality, fix $n\ge1$ and write
\[
    a_n:=\frac1n\inf_sV_n^u(s).
\]
Define, pointwise,
\[
    \Potential
    :=
    \max_{0\le j\le n-1}
    \left\{
        V_j^u-ja_n
    \right\}.
\]
By monotonicity and additive homogeneity,
\[
\begin{aligned}
\mathcal T_{u,H}\Potential
&\ge
a_n
+
\max_{1\le j\le n}
\left\{
    V_j^u-ja_n
\right\}.
\end{aligned}
\]
Since
\[
    V_n^u-na_n
    =
    V_n^u-\inf_sV_n^u(s)
    \ge0=V_0^u,
\]
the last maximum dominates $\Potential$. Hence
\[
    \mathcal T_{u,H}\Potential
    \ge
    \Potential+a_n.
\]
It follows that
\[
\sup_{\Potential}
\inf_s
\left\{
    (\mathcal T_{u,H}\Potential)(s)-\Potential(s)
\right\}
\ge
\frac1n\inf_sV_n^u(s).
\]
Taking the supremum over $n$ proves the lower representation.

The upper representation is analogous.
For any bounded $\Potential$,
\[
\mathcal T_{u,H}\Potential
\le
\Potential+
\sup_s
\left\{
    (\mathcal T_{u,H}\Potential)(s)-\Potential(s)
\right\},
\]
and iteration gives
\[
\sup_sV_n^u(s)
\le
\operatorname{range}(\Potential)
+
n
\sup_s
\left\{
    (\mathcal T_{u,H}\Potential)(s)-\Potential(s)
\right\}.
\]
Conversely, for fixed $n$, let
\[
    b_n:=\frac1n\sup_sV_n^u(s)
\]
and define
\[
    \Potential
    :=
    \min_{0\le j\le n-1}
    \left\{
        V_j^u-jb_n
    \right\}.
\]
Then
\[
\mathcal T_{u,H}\Potential
\le
b_n
+
\min_{1\le j\le n}
\left\{
    V_j^u-jb_n
\right\}.
\]
Since
\[
    V_n^u-nb_n
    =
    V_n^u-\sup_sV_n^u(s)
    \le0=V_0^u,
\]
the last minimum is bounded above by $\Potential$, and therefore
\[
    \mathcal T_{u,H}\Potential
    \le
    \Potential+b_n.
\]
This proves
\[
\lim_{n\to\infty}
\sup_s\frac{V_n^u(s)}{n}
=
\inf_{\Potential}
\sup_s
\left\{
    (\mathcal T_{u,H}\Potential)(s)-\Potential(s)
\right\}.
\]

It remains only to place the middle quantity between the lower and
upper values.
For any bounded $\Potential$, the preceding argument gives
\[
V_n^u(s)
\ge
nc_{\Potential}
+
\Potential(s)-\sup_{s'}\Potential(s').
\]
Given $\varepsilon>0$, choose $s_\varepsilon$ such that
\[
    \Potential(s_\varepsilon)
    \ge
    \sup_s\Potential(s)-\varepsilon.
\]
Then
\[
    \inf_{n\ge1}
    \frac{V_n^u(s_\varepsilon)}{n}
    \ge
    c_{\Potential}-\varepsilon.
\]
Taking the supremum over $s$, then over $\Potential$, and finally
letting $\varepsilon\downarrow0$ gives
\[
\sup_{\Potential}\inf_s
\left\{
    (\mathcal T_{u,H}\Potential)(s)-\Potential(s)
\right\}
\le
\sup_s\inf_{n\ge1}\frac{V_n^u(s)}{n}.
\]
The inequality
\[
    \sup_s\inf_{n\ge1}\frac{V_n^u(s)}{n}
    \le
    \inf_{n\ge1}\sup_s\frac{V_n^u(s)}{n}
\]
is immediate. This proves \eqref{eq:supinf relation}.

Now suppose
Assumption~\ref{ass:finite-step-bridging} holds.
Fix $s,s'\in\WinSet$, and let
$\ell\le C_B$ be the length of an admissible bridge from $s$ to $s'$.
For $n\ge\ell$, the adversary may first follow this bridge and then
continue from $s'$. Since every directional block payoff has magnitude
at most $M_H$,
\[
    V_n^u(s)
    \ge
    V_{n-\ell}^u(s')-\ell M_H.
\]
Also, by monotonicity and additive homogeneity,
\[
    V_n^u(s')
    \le
    V_{n-\ell}^u(s')+\ell M_H.
\]
Hence
\[
    V_n^u(s)
    \ge
    V_n^u(s')-2C_BM_H.
\]
Interchanging $s$ and $s'$ gives
\[
    \sup_sV_n^u(s)-\inf_sV_n^u(s)
    \le
    2C_BM_H.
\]
For $n<C_B$, the same bound follows directly from
$|V_n^u(s)|\le nM_H$.

Thus
\[
0
\le
\sup_s\frac{V_n^u(s)}{n}
-
\inf_s\frac{V_n^u(s)}{n}
\le
\frac{2C_BM_H}{n}
\longrightarrow0.
\]
The lower and upper limits in
\eqref{eq:supinf relation} therefore coincide, and hence both
inequalities in that display are equalities.
\end{proof}
\begin{remark}[Relation to additive Collatz--Wielandt theory]
Lemma~\ref{lem:supinf-relation} is closely related to the additive
Collatz--Wielandt formula for monotone, additively homogeneous operators;
see, e.g., \cite{gaubert2004perron,akian2012tropical}. In particular, for
clarity and completeness, we adapt the 
argument of \cite[Lemma~2]{gaubert2004perron} to our setting and give a self-contained proof.
\end{remark}
\section{Experimental Details}
\label{app:exp}

\subsection{Algorithms and implementation}
\label{app:exp-algorithms}
We give the updates used in all three experiments. The public action domain
$\mathcal X$ is known in advance; the current constraint $c_t$ is revealed
only after $x_t$ is chosen. Experimental domains, initialization, and
validation-selected parameters are specified in the following subsections.

\subsubsection{\softblackwell}
Algorithms~\ref{alg:soft-window} and~\ref{alg:window-soft-ftrl} give the outer
controller and its FTRL window learner. The cumulative gradient and violation
statistics persist across windows. Each new FTRL learner starts at the last
\emph{played} action. Its parameter $\eta_H$ is selected on validation data;
$D_{\mathcal X}/(L\sqrt H)$ in the pseudocode is the theoretical default.
The final window is truncated at the evaluation horizon when necessary.

\begin{algorithm}[H]
\caption{Separate learning on each $H$-window with soft constraints}
\label{alg:soft-window}
\begin{algorithmic}[1]
\State Initialize $S_0\gets 0$ and $V_0\gets 0$
\For{$k=0,1,2,\ldots$}
    \If{$S_{kH}\neq 0$}
        \State
        $\displaystyle
        u_k\gets \frac{S_{kH}}{\|S_{kH}\|_2}$
    \Else
        \State Choose any $u_k\in\mathbb S^{d-1}$
    \EndIf
    \If{$\|S_{kH}\|_2+V_{kH}>0$}
        \State
        $\displaystyle
        \theta_k
        \gets
        \frac{\|S_{kH}\|_2}
        {\|S_{kH}\|_2+V_{kH}}$
    \Else
        \State $\theta_k\gets 0$
    \EndIf
    \State Restart $\mathsf A_H$ with fixed $(u_k,\theta_k)$
    \For{$t=kH+1,\ldots,(k+1)H$}
        \State Play the action $x_t$ selected by $\mathsf A_H$
        \State Observe $g_t$ and $c_t$
        \State Feed the feedback $(g_t,c_t)$ to $\mathsf A_H$
           \State Set
        $
        v_t \gets [c_t(x_t)]_+
        $
        \State Update
        $
        S_t \gets S_{t-1}+g_t(x_t),
        \qquad
        V_t \gets V_{t-1}+v_t
        $
    \EndFor
\EndFor
\end{algorithmic}
\end{algorithm}

\begin{algorithm}[H]
\caption{$H$-window learner $\mathsf A_H$: Soft-Blackwell FTRL}
\label{alg:window-soft-ftrl}
\begin{algorithmic}[1]
\Require Fixed $(u,\theta)\in\mathbb S^{d-1}\times[0,1]$, domain $\mathcal X$, scalarized-loss Lipschitz bound $L$
\State Choose $x_0\in\mathcal X$; set $\Psi(x)\gets\tfrac12\|x-x_0\|_2^2$
\State Set $x_1\gets x_0$ and $\eta_H\gets D_{\mathcal X}/(L\sqrt H)$
\For{$r=1,\ldots,H$}
  \State Play $x_r$, then observe $g_r$ and $c_r$
  \State $f_r^{(u,\theta)}(x)\gets\theta\langle u,g_r(x)\rangle+(1-\theta)[c_r(x)]_+$
  \State $\displaystyle x_{r+1}\in\arg\min_{x\in\mathcal X}\left\{\eta_H\sum_{s=1}^rf_s^{(u,\theta)}(x)+\Psi(x)\right\}$
\EndFor
\end{algorithmic}
\end{algorithm}

\subsubsection{Baselines}
\label{app:baseline-updates}
Both GD baselines use constant step sizes. All methods' violations are
measured against the current constraint $c_t$.

\begin{algorithm}[H]
\caption{Vanilla GD: public-domain projection only}
\label{alg:projected-gd}
\begin{algorithmic}[1]
\Require Public action domain $\mathcal X$, constant step size $\eta>0$
\State Initialize $x_1\in\mathcal X$
\For{$t=1,\ldots,T$}
  \State Play $x_t$, then observe $g_t(x_t)$ and $c_t$
  \State $x_{t+1}\gets\Pi_{\mathcal X}(x_t-\eta g_t(x_t))$
\EndFor
\end{algorithmic}
\end{algorithm}
Vanilla GD forms the candidate action $x_t-\eta g_t(x_t)$ and projects
this point onto the fixed action domain $\mathcal X$ to obtain $x_{t+1}$.
The update uses the loss gradient and the fixed domain, without a
constraint-violation penalty. 

Lazy GD maintains a separate latent iterate $y_t$.
It plays the projected action $x_t=\Pi_{P_t}(y_t)$ and then updates
$y_{t+1}=y_t-\eta g_t(x_t)$, without projecting $y_{t+1}$.
The experimental protocol specifies the available projection set $P_t$.

\begin{algorithm}[H]
\caption{Lazy GD}
\label{alg:lazy-gd}
\begin{algorithmic}[1]
\Require Public domain $\mathcal X$, horizon $T$, constant step size $\eta>0$; nonempty closed convex projection sets $\mathcal P_t\subseteq\mathcal X$, each available before selecting $x_t$
\State Initialize $y_1\in\mathcal X$
\For{$t=1,\ldots,T$}
  \State $x_t\gets\Pi_{P_t}(y_t)$
  \State Play $x_t$, then observe $g_t(x_t)$ and $c_t$
  \State $y_{t+1}\gets y_t-\eta g_t(x_t)$
\EndFor
\end{algorithmic}
\end{algorithm}

\paragraph{Projection sets used in our experiments.}
For our implementation, let $C_s=\{x\in\mathcal X:c_s(x)\le0\}$ and set,
for $t\in I_w$,
\[
 \mathcal P_t=
 \begin{cases}
  \mathcal X, & w=1,\\
  \displaystyle\bigcap_{s\in I_{w-1}}C_s, & w>1.
 \end{cases}
\]
Thus the projection set is fixed within each window and uses only constraints
from the preceding completed window. In the toxicity and stock experiments,
the feasible sets coincide within each window, so this is simply the previous
window's feasible set. In the binary experiment, the intersection instead
uses the maximum lower bound observed in that window.

\begin{algorithm}[H]
\caption{COCO baseline \citep{sinha2024optimal}}
\label{alg:coco-baseline}
\begin{algorithmic}[1]
\Require Domain $\mathcal X$ with diameter $D_{\mathcal X}$, Lipschitz bound $L$, horizon $T$, step multiplier $m>0$
\State $\beta\gets(2LD_{\mathcal X})^{-1}$, $\lambda\gets(2\sqrt T)^{-1}$
\State Initialize $x_1\in\mathcal X$, $Q_0\gets0$, and $A_0\gets0$
\For{$t=1,\ldots,T$}
  \State Play $x_t$, then observe $\ell_t$ and $c_t$
  \State $v_t\gets[c_t(x_t)]_+$; $Q_t\gets Q_{t-1}+\beta v_t$
  \State Choose $g_t\in\partial\ell_t(x_t)$ and $d_t\in\partial[c_t]_+(x_t)$
  \State $\zeta_t\gets\beta\bigl(g_t+\lambda e^{\lambda Q_t}d_t\bigr)$
  \State $A_t\gets A_{t-1}+\|\zeta_t\|_2^2$
  \If{$A_t>0$}
    \State $\eta_t\gets mD_{\mathcal X}/\sqrt{2A_t}$
    \State $x_{t+1}\gets\Pi_{\mathcal X}(x_t-\eta_t\zeta_t)$
  \Else
    \State $x_{t+1}\gets x_t$
  \EndIf
\EndFor
\end{algorithmic}
\end{algorithm}
The choice $m=1$ recovers the native COCO update. Its queue and adaptive
step-size accumulator run continuously across windows.

\subsection{Binary forecasting construction}
\label{app:exp-toy}
\paragraph{Sequence and splits.}
The binary example uses $\mathcal X=[0,1]$,
$\ell_t(x)=\tfrac12(x-z_t)^2$, and $g_t(x)=x-z_t$.
Block $A$ has a one followed by nine zeros; block $B$ has eight ones followed
by two zeros. Repeating $A^M B^M$ produces alternating regimes lasting $10M$
rounds each, with $20M$ rounds per cycle. We consider
$M\in\{3,10,30,100,300,1000\}$; Figure~\ref{fig:toy-main} uses $M=300$.
For each $M$, the first ten cycles ($200M$ rounds) are used for validation
and the next fifty ($T=1000M$ rounds) for testing. All methods receive the
same sequence and restart at $x_1=0.5$ at each split. The deterministic
cycles repeat, so validation tests transfer to a longer horizon rather than
to an independent outcome sample.

\paragraph{Lower bounds and feedback.}
With $H=10$ and $\gamma=0.05$, define
\[
 q_t=\min_{\substack{I\ni t\\ I\text{ consecutive},\ |I|=H}}
 \left\{\frac1H\sum_{s\in I}z_s-\gamma\right\},
 \qquad c_t(x)=q_t-x,\qquad V_t=\sum_{s=1}^t[q_s-x_s]_+.
\]
The construction uses the indefinitely continued outcome sequence when
forming windows near the end of the evaluation horizon. For any length-$H$
interval $I$, its mean outcome $\bar z_I$ satisfies
$\bar z_I-q_t\ge\gamma$ for every $t\in I$ and
$\sum_{t\in I}g_t(\bar z_I)=0$. Thus every sliding window, including those
crossing a regime transition, admits a common feasible forecast with zero
gradient sum. The learner receives $z_t$ and $q_t$ only after choosing $x_t$.

\paragraph{Algorithm parameters.}
\softblackwell retains global $S_t,V_t$, freezes its direction and mixing weight for
each nonoverlapping $H$-window, and restarts the FTRL history with its
quadratic centered at the previous played action. We set
$\eta_H=\kappa D_{\mathcal X}/(L\sqrt H)=\kappa/\sqrt{10}$.
COCO uses $L=D_{\mathcal X}=1$, $\beta=1/2$, and
$\lambda=1/(2\sqrt T)$ for the current split, with adaptive step
$\eta_t=mD_{\mathcal X}/(2\sum_{s\le t}\|\zeta_s\|_2^2)^{1/2}$.
Its queue and step accumulator persist across windows.
Both GD methods use constant rates. Lazy GD projects its hidden iterate onto
$[\max_{s\in I_{w-1}}q_s,1]$ using the previous completed nonoverlapping
window $I_{w-1}$, and onto $[0,1]$ in the first window; its hidden state
is not reset at window boundaries. The saved vanilla-GD runs use unclipped
updates. Every selected rate is at most one, making each update a convex
combination of the current forecast and outcome, so projection onto $[0,1]$
would leave these selected trajectories unchanged.

\paragraph{Validation and selected parameters.}
For each method and $M$, selection minimizes validation $|S_T|/T+V_T/T$.
Ties are broken by $V_T/T$, mean running $|S_t|/t$, distance of the parameter
from one on a logarithmic scale, and then its value. \softblackwell's multiplier
$\kappa$ and COCO's multiplier $m$ use the grid
$\{0.01,0.03,0.1,0.3,1,3,10,30,100\}$; both GD rates use
$\{0.001,0.003,0.01,0.03,0.1,0.3,1,1.5,1.9\}$.
Table~\ref{tab:toy-rates} reports the validation-selected parameters, fixed
before testing.

\begin{table}[htbp]
  \centering
  \caption{Validation-selected toy parameters. \softblackwell uses
  $\eta_H=\kappa/\sqrt{10}$; both GD rates are constant.}
  \label{tab:toy-rates}
  \begin{tabular}{rrrrr}
    \toprule
    $M$ & \softblackwell ($\kappa$) & COCO ($m$) & Vanilla GD ($\eta$) & Lazy GD ($\eta$) \\
    \midrule
    3    & 30 & 3 & 0.3  & 0.1 \\
    10   & 10 & 3 & 0.1  & 0.3 \\
    30   & 3  & 3 & 0.1  & 0.3 \\
    100  & 3  & 3 & 0.1  & 0.3 \\
    300  & 3  & 3 & 0.03 & 0.3 \\
    1000 & 3  & 3 & 0.03 & 0.3 \\
    \bottomrule
  \end{tabular}
\end{table}

\paragraph{Results across regime durations.}
Figure~\ref{fig:toy-regime-sensitivity} shows all six settings. \softblackwell has smaller final gradient imbalance and cumulative violation than COCO throughout, and less violation than vanilla GD. Lazy GD is competitive and has the smallest final violation for $M\ge10$; at $M=3$, \softblackwell has the lowest violation while vanilla GD has the smallest final gradient imbalance.

\begin{figure}[p]
  \centering
  \includegraphics[width=0.98\textwidth,height=0.86\textheight,keepaspectratio]{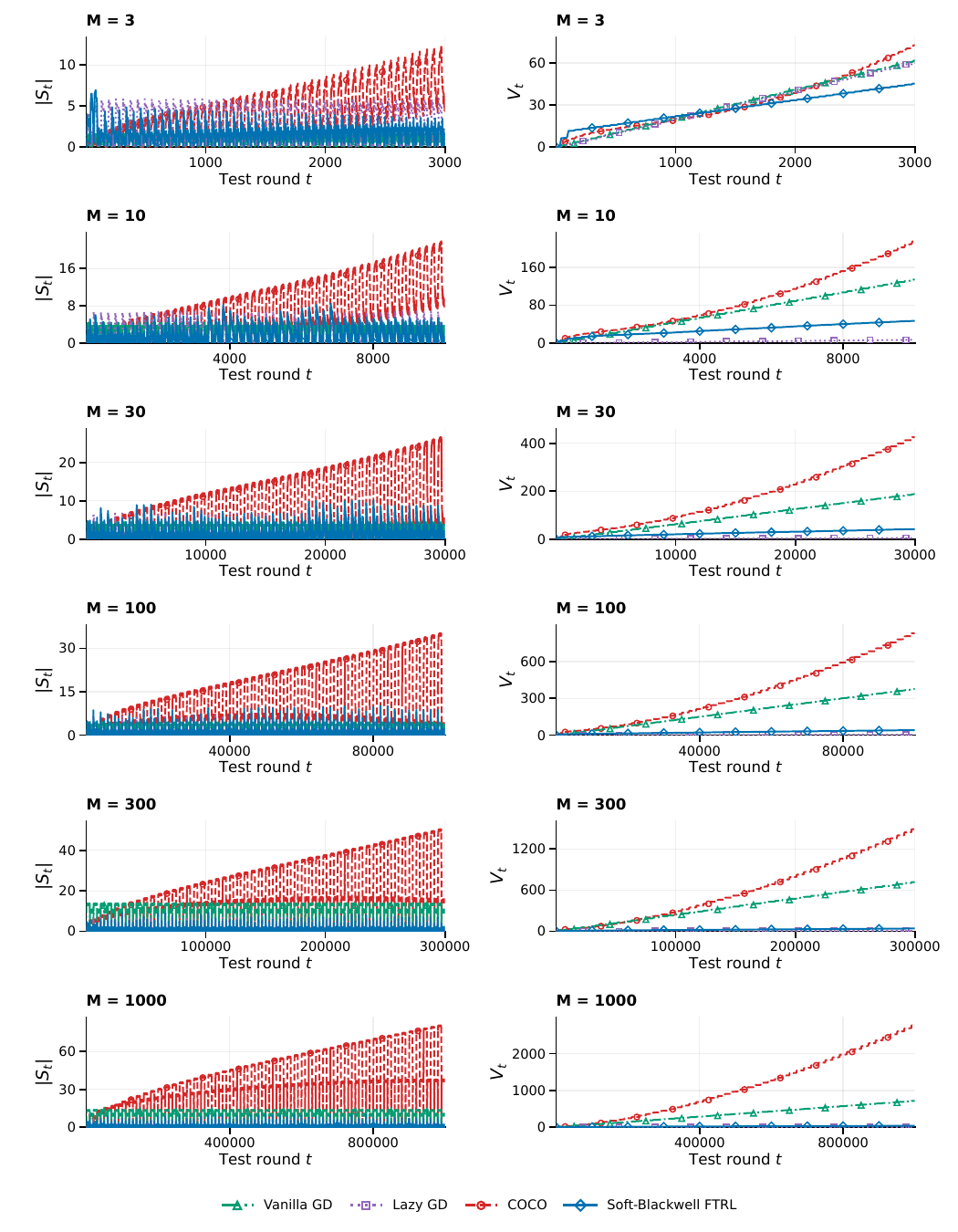}
  \caption{Binary forecasting across regime durations. Rows correspond to
  $M=3,10,30,100,300,1000$, with $H=10$ and fifty test cycles in each setting.
  Left: $|S_t|$; right: $V_t$. Both GD methods use validation-selected constant
  rates. All axes are linear; every test round is included.}
  \label{fig:toy-regime-sensitivity}
\end{figure}

\subsection{LLM toxicity: data and evaluation protocol}
\label{app:toxicity-data}
\paragraph{Data and splits.}
We use the paired scores released by \citet{chen2026adversarially}, comprising
9,500 prompts with 40 responses each.\footnote{Author release:
\url{https://github.com/jy-evangeline/online_cvar/tree/4c0e14b65ee13ac16d7335fb94fb99c30a01bb63}.}
The machine scores are produced by a fine-tuned, biased Detoxify model,
and the original Detoxify model provides the reference scores used to evaluate
toxicity. We conduct the experiment directly on the released score pairs.
A permutation with seed 42 assigns 1,000 prompts to validation and 8,500 to
testing. We use all 40 responses per prompt and apply thresholds to the raw
machine scores. Validation uses order seed 0, while test seeds 0, 1, and 2
give three orderings of the same held-out prompts. Within each run, all
methods receive the same prompt order, score pairs, and budgets.

\paragraph{Window budgets and feedback.}
For each current window $I_w$, define
\[
 F_w(x)=\frac{1}{|I_w|}\sum_{t\in I_w}R_t(x),
 \qquad q_w=(1-\tau)x_w^\star,
\]
where $x_w^\star$ is the largest representable threshold in $[0,1]$ with
$F_w(x_w^\star)\le\alpha$, for $\alpha=0.1$.
At a jump crossing the target, we use the value immediately below the jump.
Thus $\tau$ controls relative slack in threshold space, while the target
risk stays fixed. The environment constructs each budget from its own window,
including the first, and reveals $c_t(x)=q_w-x$ after the learner acts.
\softblackwell and COCO update from this observed feedback. Lazy GD projects onto
$[q_{w-1},1]$ during window $w>1$ and onto $[0,1]$ in the first window.
Every method's violation is evaluated against the current $q_w$.
The evaluated windows admit feasible thresholds that bracket the target risk;
all rounds are included.

\paragraph{Algorithm parameters.}
All methods start at $x_1=0.5$ on $[0,1]$. \softblackwell retains its cumulative gradient and violation sums across windows, and initializes each new FTRL window at the previous played action. Both GD baselines use constant step sizes. COCO uses $L=D=1$ and maintains its queue over the full horizon.

For each setting, we choose the GD step sizes and \softblackwell parameter $\eta_H$ from $\{0.001,0.01,0.1,1,10\}$, and COCO's adaptive-step multiplier $m$ from $\{0.1,1\}$. Selection minimizes final validation $|S_T|$, breaking ties by $V_T$ and then the parameter value. The selected parameters are fixed across test orders. In the order $(H,\tau)=(10,0.1),(10,0.2),(20,0.1),(20,0.2)$, \softblackwell uses $\eta_H=10,1,0.01,0.01$, and Lazy GD uses $\eta=10,10,10,1$. Vanilla GD uses $\eta=0.1$ and COCO uses $m=0.1$ in all four settings.

\paragraph{Results across budgets and window sizes.}
\label{app:toxicity-sensitivity}
Figure~\ref{fig:toxicity-sensitivity} compares the main setting
$(H,\tau)=(10,0.1)$ with the other three configurations. \softblackwell has the smallest final gradient imbalance and cumulative violation in each setting. Increasing the relative slack reduces its violation at both window lengths. Lazy GD's gradient balance improves substantially at $H=20$, $\tau=0.2$, while its cumulative violation remains higher than \softblackwell's. Vanilla GD's gradient trajectory is unchanged across settings because its selected step size is fixed and its update does not use the budget.

\begin{figure}[p]
  \centering
  \includegraphics[width=\textwidth,height=0.80\textheight,keepaspectratio]{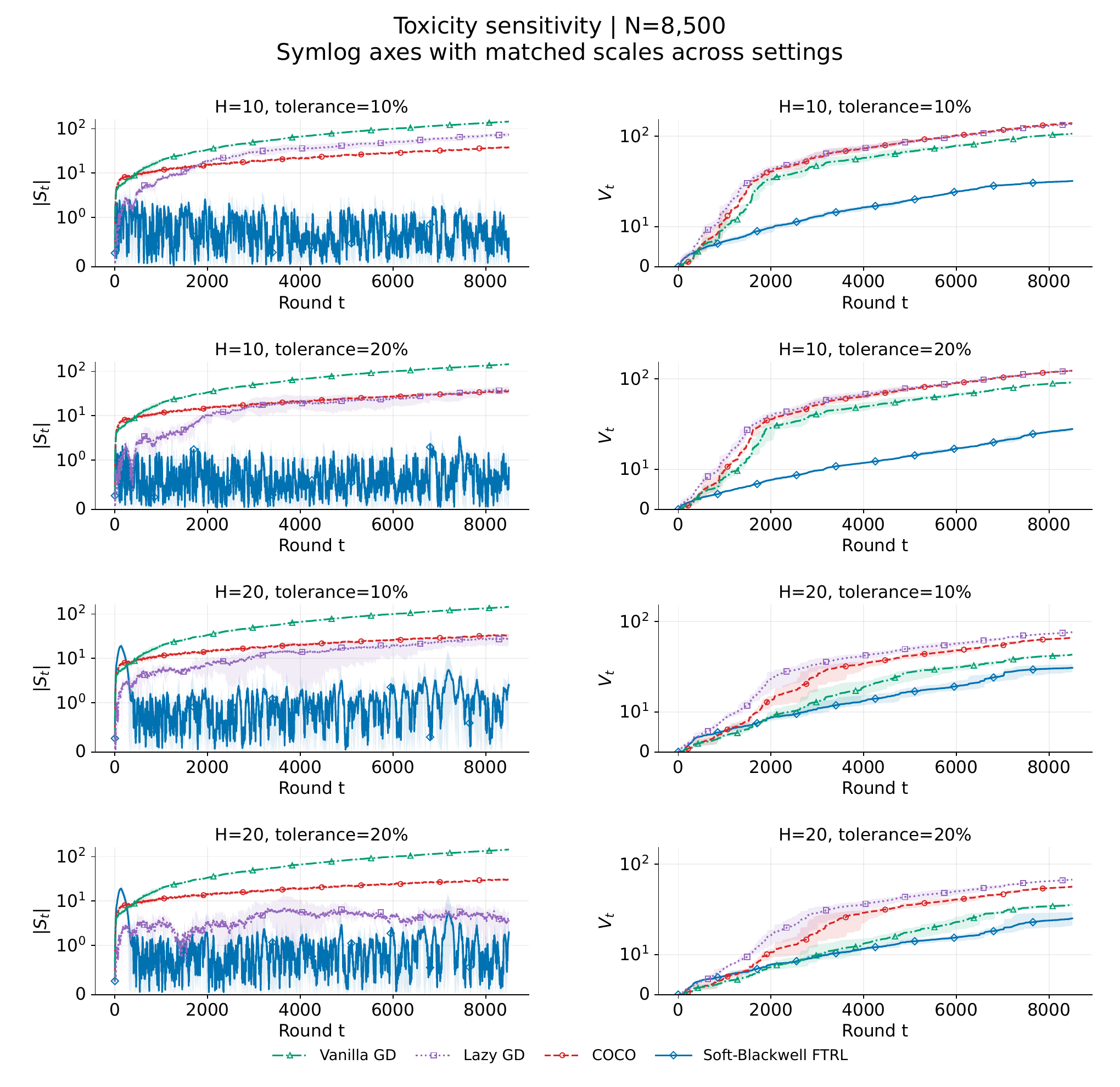}
  \caption{Window and budget sensitivity for LLM toxicity. Rows correspond to
  $(H,\tau)=(10,10\%),(10,20\%),(20,10\%),(20,20\%)$.
  The left column reports $|S_t|$ and the right column reports $V_t$.
  We compute the absolute gradient sum within each run before averaging;
  shading gives the minimum--maximum range over the three prompt orders.
  Axes use the same scales across settings, with a linear region near zero
  and a logarithmic region above.}
  \label{fig:toxicity-sensitivity}
\end{figure}
\newpage
\subsection{Two-dimensional stock experiment}
\label{app:exp-stock-2d}
\paragraph{Data and splits.}
Cached rolling GARCH$(1,1)$ forecasts use the preceding $1{,}250$ daily returns,
following \cite{gibbs2021adaptiveconformalinferencedistribution}.
The historical returns are from the public DtACI release.\footnote{\url{https://github.com/isgibbs/DtACI/tree/1288af7929c7d5c741735f4194272e9b3d9ffb39/MajorDailyReturns}.}
The supplement includes the fixed forecasts and realized values
used for evaluation.
For AMD, BB, FNMA, and NVDA, respectively, the chronological validation/test
split contains $1600/6344$, $800/3545$, $1600/6345$, and $800/3554$ rounds.
All methods and window lengths share these splits and start afresh at
$x_1=(1,1)$. Validation finishes before parameters are frozen for testing;
no predictor is refitted.

\paragraph{Chronological evaluation and variability.}
Each stock is evaluated once in its original chronological order. With the
frozen forecasts, fixed splits, and selected parameters, all four methods
are deterministic, so repeating a run yields the same trajectory. We report
these trajectories without seed-based error bands. The four stocks and two
window lengths are separate evaluation settings, not independent repetitions
of a single time series.

\paragraph{Loss, budgets, and feedback.}
Each feedback coordinate is a selected subgradient of the pinball loss
\[
 \ell_t^{(\alpha)}(x)=(1-\alpha)[x^{(\alpha)}-e_t]_+
                         +\alpha[e_t-x^{(\alpha)}]_+.
\]
COCO uses the sum of these two convex losses. For $H\in\{10,20\}$, budgets
use linearly interpolated score quantiles within each nonoverlapping window:
$Q_{0.92}$ for the $90\%$ interval and $Q_{0.72}$ for the $70\%$ interval.
In complete windows, we add $10^{-6}\max\{1,Q_{0.92}\}$ when the
positive-direction comparator equals the unpadded $Q_{0.92}$.
We never pad $Q_{0.72}$ or final partial windows, which are retained.
Budgets are constructed offline from their own windows and revealed after
action selection. The coordinatewise positive parts prevent cancellation
between the two width violations.

\paragraph{Algorithm parameters.}
All four methods use $\mathcal X=\{(x^{(0.9)},x^{(0.7)}):0\le x^{(0.7)}\le x^{(0.9)}\le50\}$,
with coordinate cap $50$ and Euclidean diameter $D_{\mathcal X}=50\sqrt{2}$.
COCO uses $L=\sqrt{2}$, $\beta=1/(2LD_{\mathcal X})=1/200$, and $\lambda=1/(2\sqrt{T})$, where $T$ is the current split length. Its adaptive step is $$\eta_t=m D_{\mathcal X}/(2\sum_{s\le t}\|\zeta_s\|_2^2)^{1/2},$$ where $\zeta_s$ is its surrogate subgradient. 


\paragraph{Validation and selected parameters.}
For each stock and window length, we select \softblackwell's $\eta_H$ from $\{0.1,1\}$, both constant GD rates from $\{0.001,0.01,0.1,1,10,100\}$, and COCO's adaptive-step multiplier $m$ from $\{0.1,1\}$.
Selection minimizes final validation $\|S_T\|_2/T$, breaking ties by $V_T/T$, average summed physical width, and then the parameter value. We compare squared integer-count errors, equal to $100\|S_T\|_2^2$, to identify calibration ties exactly. Selection thus prioritizes calibration rather than mean pinball loss. COCO selects $m=0.1$ in every setting; Table~\ref{tab:stock-multilevel-rates} lists the other selected parameters.

\begin{table}[htbp]
  \centering
  \caption{Validation-selected stock parameters.}
  \label{tab:stock-multilevel-rates}
  \begin{tabular}{lrrrrr}
    \toprule
    & \multicolumn{2}{c}{\softblackwell ($\eta_H$)} & GD ($\eta$)
    & \multicolumn{2}{c}{Lazy GD ($\eta$)} \\
    \cmidrule(lr){2-3} \cmidrule(lr){5-6}
    Stock & $H=10$ & $H=20$ & Both $H$ & $H=10$ & $H=20$ \\
    \midrule
    AMD  & 1   & 1   & 0.1  & 1   & 1 \\
    BB   & 1   & 0.1 & 0.01 & 0.1 & 0.001 \\
    FNMA & 0.1 & 1   & 0.1  & 10  & 10 \\
    NVDA & 0.1 & 1   & 0.1  & 0.1 & 1 \\
    \bottomrule
  \end{tabular}
\end{table}

\paragraph{Results across stocks and windows.}
Figures~\ref{fig:stock-multilevel-h10} and~\ref{fig:stock-multilevel-h20} highlight \softblackwell's strong calibration--width balance, with ordinary GD also competitive. \softblackwell achieves smaller final calibration error than COCO in six of eight settings and than GD in five. On FNMA at $H=20$, it has the smallest calibration error, with violation close to GD and below both remaining baselines. Lazy GD exhibits substantially larger calibration errors throughout. COCO's low violation often comes with larger calibration error, most visibly on FNMA. 
GD benefits from the pinball subgradient being exactly the coverage residual. With constant step size and no projection, its updates telescope to $S_T=(x_1-x_{T+1})/\eta$; uniformly bounded iterates therefore give $\|S_T\|_2/T=O(1/T)$. The telescoping identity is general, while pinball loss makes it directly a coverage-calibration identity. GD's low soft violation is an empirical strength of this setting; its update does not explicitly control the width budgets, whereas \softblackwell incorporates both calibration and violation feedback.


\begin{figure}[H]
  \centering
  \includegraphics[width=0.9\textwidth, keepaspectratio]{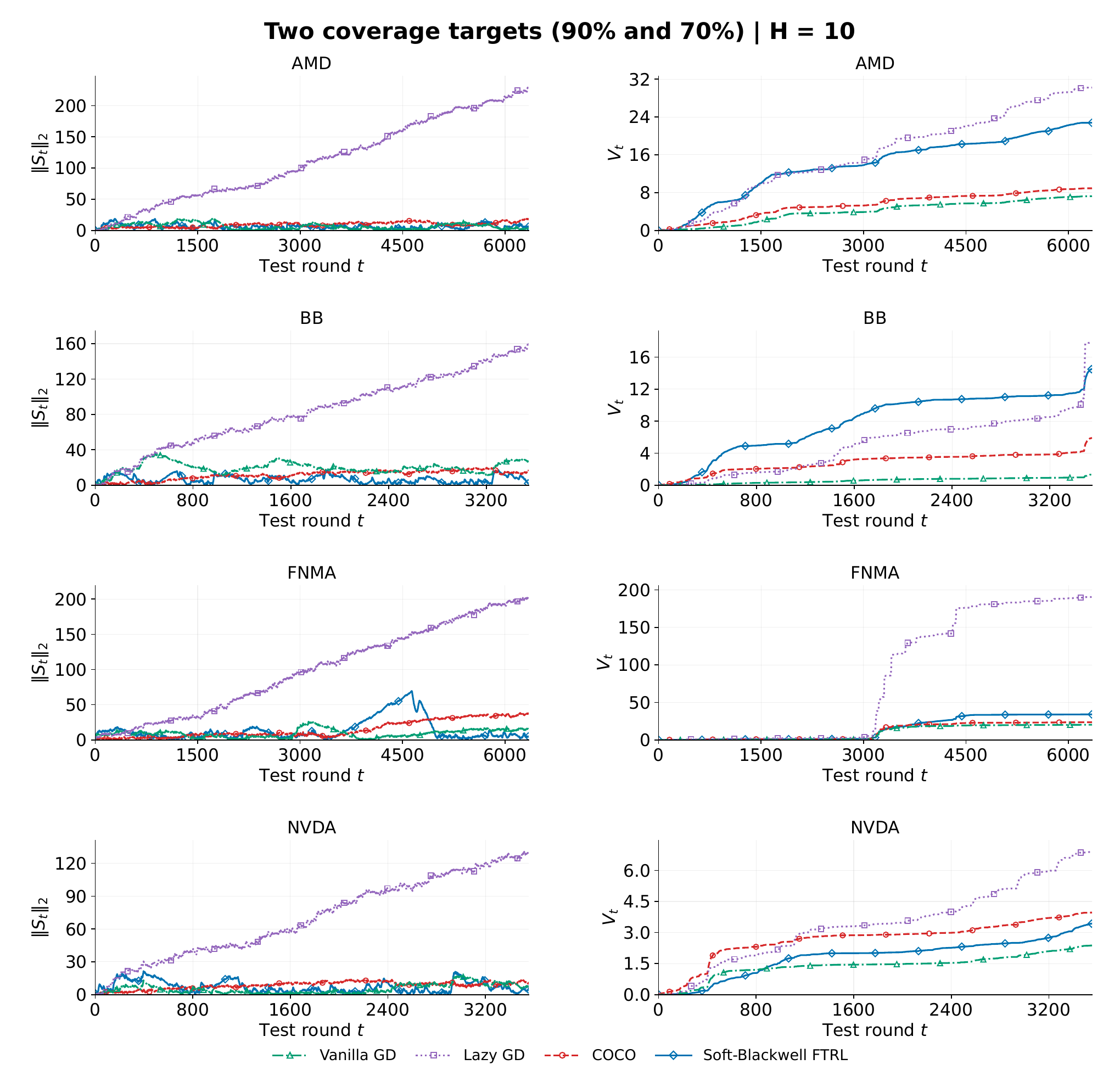}
  \caption{Simultaneous stock coverage with $H=10$. Rows show AMD, BB,
  FNMA, and NVDA. Left: cumulative gradient imbalance $\|S_t\|_2$;
  right: cumulative positive width violation $V_t$. All methods use the same
  observations, budgets, hard domain, and validation/test split.
  GD enforces subset consistency by Euclidean projection. All axes are linear.}
  \label{fig:stock-multilevel-h10}
\end{figure}
\begin{figure}[H]
  \centering
  \includegraphics[width=0.9\textwidth, keepaspectratio]{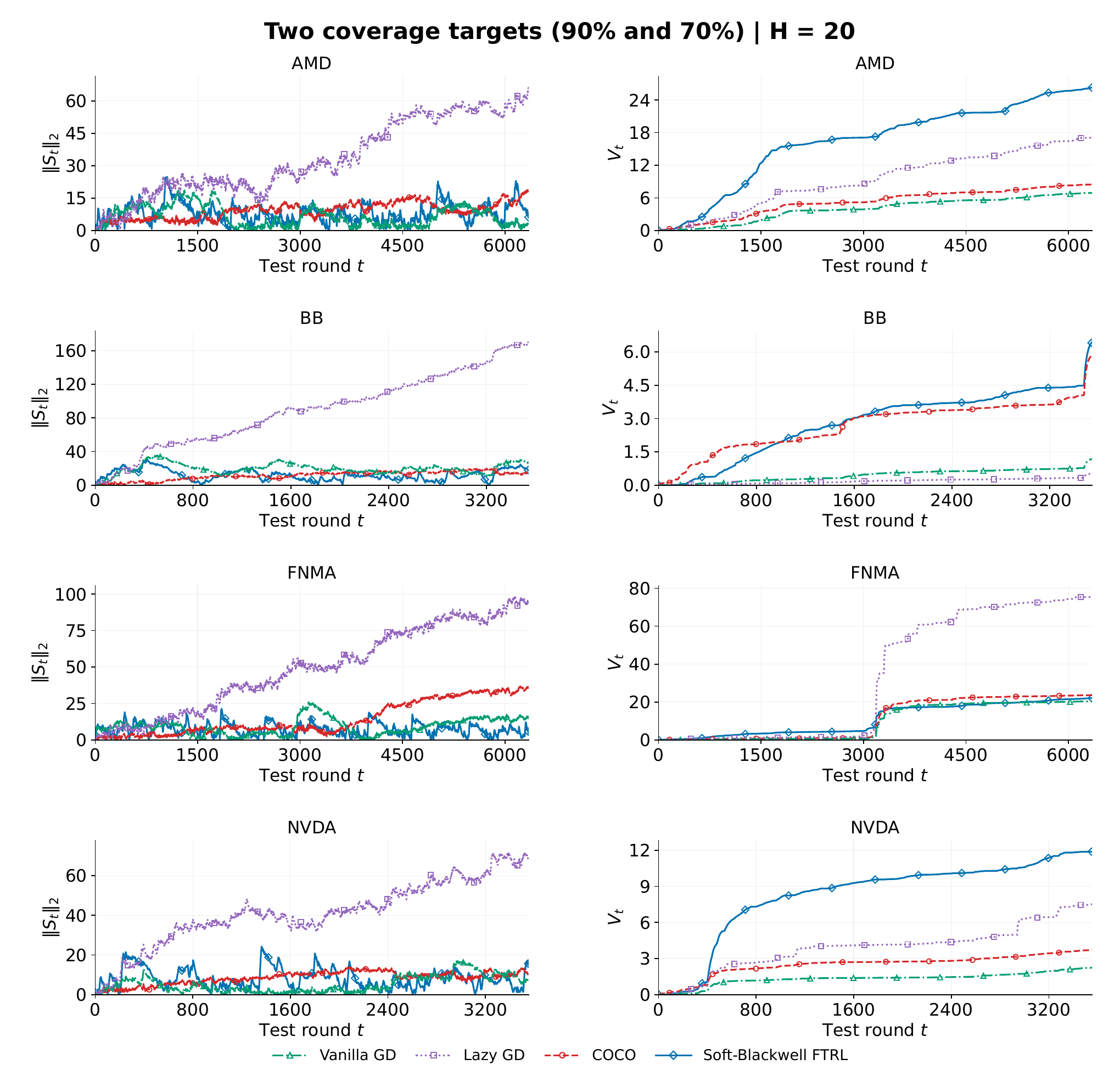}
  \caption{Simultaneous stock coverage with $H=20$, using the same stock
  order, cumulative metrics, and validation procedure as
  Figure~\ref{fig:stock-multilevel-h10}. The main-text FNMA comparison in
  Figure~\ref{fig:stock-multilevel-fnma} uses $H=10$.}
  \label{fig:stock-multilevel-h20}
\end{figure}

\end{document}